\documentclass{article}

\usepackage{amssymb}
\usepackage{amsfonts}

\usepackage{tgtermes}
\usepackage{amsmath}
\usepackage{amsthm}  
\usepackage{scalefnt,letltxmacro}
\LetLtxMacro{\oldtextsc}{\textsc}
\renewcommand{\textsc}[1]{\oldtextsc{\scalefont{1.10}#1}}
\usepackage[scaled=0.92]{PTSans}

\usepackage[usenames,dvipsnames]{xcolor}
\usepackage{soul}
\definecolor{shadecolor}{gray}{0.9}

\usepackage{afterpage}
\usepackage{framed}
\usepackage{nicefrac}
\usepackage{bm}

\usepackage{paralist}

\usepackage[colorinlistoftodos,
           prependcaption,
           textsize=small,
           backgroundcolor=yellow,
           linecolor=lightgray,
           bordercolor=lightgray]{todonotes}

\usepackage{lineno}

\usepackage{ragged2e}

\usepackage{graphicx}
\usepackage[labelfont=it,labelsep=period]{caption}
\usepackage[format=hang]{subcaption}
\usepackage{wrapfig}
\usepackage{placeins}

\usepackage{booktabs}
\usepackage{arydshln} 
\usepackage{multirow}

\usepackage{algorithm}
\usepackage{listings}
\usepackage{fancyvrb}
\fvset{fontsize=\normalsize}
\usepackage[noend]{algpseudocode}

\usepackage[colorlinks,linktoc=all]{hyperref}
\usepackage[all]{hypcap}
\hypersetup{citecolor=Violet}
\hypersetup{linkcolor=black}
\hypersetup{urlcolor=MidnightBlue}
\usepackage{url}

\usepackage[nameinlink]{cleveref}
\creflabelformat{equation}{#1#2#3}
\crefname{equation}{eq.}{eqs.}  
\Crefname{equation}{Eq.}{Eqs.}

\usepackage[acronym,nowarn]{glossaries}

\usepackage{listings}
\lstdefinestyle{alp_style}{
    commentstyle=\color{OliveGreen},
    numberstyle=\tiny\color{black!60},
    stringstyle=\color{BrickRed},
    basicstyle=\ttfamily\scriptsize,
    breakatwhitespace=false,
    breaklines=true,
    captionpos=b,
    keepspaces=true,
    numbers=none,
    numbersep=5pt,
    showspaces=false,
    showstringspaces=false,
    showtabs=false,
    tabsize=2
}
\usepackage{lipsum}

\newcommand{\argmin}{\operatornamewithlimits{argmin}}

\def\f0{{\mathbf 0}}

\newtheorem{thm}{Theorem}

\newtheorem{defn}{Definition}
\newtheorem{problem}{Problem}
\newtheorem{assumption}{Assumption}

\newtheorem{lemma}{Lemma}
\newtheorem{proposition}{Proposition}

\newtheorem*{result}{Informal result}

\theoremstyle{definition}

\newtheorem{remark}{Remark}

\newacronym{ADVI}{advi}{automatic differentiation variational inference}

\newacronym{BBVI}{bbvi}{black-box variational inference}

\newacronym{CDF}{cdf}{cumulative distribution function}

\newacronym{DVAE}{dvae}{discrete variational autoencoder}

\newacronym{ELBO}{elbo}{evidence lower bound}
\newacronym{EM}{em}{expectation maximization}

\newacronym{HMC}{hmc}{{H}amiltonian {M}onte {C}arlo}

\newacronym{ISPD}{$\int$spd}{integrally strictly positive definite}

\newacronym{KL}{kl}{{K}ullback-{L}eibler}

\newacronym{KSD}{ksd}{kernelised {S}tein discrepancy}

\newacronym{LDA}{lda}{latent {D}irichlet allocation}
\newacronym{LSTM}{lstm}{long short-term memory}

\newacronym{MAP}{map}{\emph{maximum a posteriori}}
\newacronym{MCMC}{mcmc}{{M}arkov chain {M}onte {C}arlo}
\newacronym{MMD}{mmd}{maximum mean discrepancy}

\newacronym{ODE}{ode}{ordinary differential equation}

\newacronym{RKHS}{rkhs}{reproducing kernel {H}ilbert space}

\newacronym{SVI}{svi}{stochastic variational inference}

\newacronym{Stein-OT}{Stein-ot}{Stein optimal transport}
\newacronym{SVGD}{svgd}{Stein variational gradient descent}

\newacronym{SMC}{smc}{sequential  {M}onte {C}arlo}

\newacronym{VAE}{vae}{variational autoencoder}
\newacronym{VEM}{vem}{variational expectation maximization}

\newacronym{VI}{vi}{variational inference}
\usepackage{amssymb}
\usepackage{amsmath}
\usepackage{arxiv}
\usepackage{algorithm}
\usepackage{algpseudocode}
\usepackage{placeins}
\usepackage{afterpage}

\usepackage[utf8]{inputenc} 
\usepackage[T1]{fontenc}    
\usepackage{hyperref}       
\usepackage{url}            
\usepackage{booktabs}       
\usepackage{amsfonts}       
\usepackage{nicefrac}       
\usepackage{microtype}      
\usepackage{lipsum}		
\usepackage{graphicx}
\usepackage{natbib}
\usepackage{doi}
\usepackage{tikz}

\usetikzlibrary{decorations.pathreplacing,calc}
\newcommand{\tikzmark}[1]{\tikz[overlay,remember picture] \node (#1) {};}

\newcommand*{\AddNote}[4]{%
    \begin{tikzpicture}[overlay, remember picture]
        \draw [decoration={brace,amplitude=0.5em},decorate,ultra thick,blue]
            ($(#3)!(#1.north)!($(#3)-(0,1)$)$) --  
            ($(#3)!(#2.south)!($(#3)-(0,1)$)$)
                node [align=center, text width=2.5cm, pos=0.5, anchor=west] {#4};
    \end{tikzpicture}
}%

\algtext*{EndFor}
\newcommand{\algcustomendfor}{\textbf{end for}}

\title{Stein transport for Bayesian inference}

\author{%
  Nikolas N{\"u}sken
  \\
  King's College London\\
  \texttt{nikolas.nusken@kcl.ac.uk} \\
}

\renewcommand{\shorttitle}{Stein transport}

\hypersetup{
pdftitle={A template for the arxiv style},
pdfsubject={q-bio.NC, q-bio.QM},
pdfauthor={David S.~Hippocampus, Elias D.~Striatum},
pdfkeywords={First keyword, Second keyword, More},
}

\begin{document}
\maketitle

\begin{abstract}
We introduce \emph{Stein transport}, a novel methodology for Bayesian inference designed to efficiently push an ensemble of particles along a predefined curve of tempered probability distributions. The driving vector field is chosen from a reproducing kernel Hilbert space and can be derived either through a suitable kernel ridge regression formulation or as an infinitesimal optimal transport map in the Stein geometry. 
The update equations of Stein transport resemble those of 
Stein variational gradient descent (SVGD), but introduce a time-varying score function as well as specific weights attached to the particles.
While SVGD relies on convergence in the long-time limit, Stein transport reaches its posterior approximation at finite time $t=1$. Studying the mean-field limit, we discuss the errors incurred by regularisation and finite-particle effects, and we connect Stein transport to birth-death dynamics and Fisher-Rao gradient flows. In a series of experiments, we show that in comparison to SVGD, Stein transport not only often reaches more accurate posterior approximations with a significantly reduced computational budget, but  that it also effectively mitigates the variance collapse phenomenon commonly observed in SVGD.  
\end{abstract}

\keywords{Stein variational gradient descent \and Kernel ridge regression \and Optimal transport}

\section{Introduction}

Approximating  high-dimensional probability distributions poses a key computational challenge in Bayesian inference, with Markov Chain Monte Carlo (MCMC) \citep{brooks2011handbook,robert2013monte} and Variational Inference (VI) \citep{jordan1999introduction,blei2017variational} representing the two most common paradigms used in practice. Combining features of both approaches, particle-based algorithms have attracted considerable interest in recent years, allowing for highly flexible as well as tractable and theoretically grounded methodologies \citep{liu2019understanding,trillos2023optimization,chen2023sampling}.

A common thread is the idea of devising dynamical movement schemes that aim at reducing a suitable discrepancy towards the target, such as the Kullback-Leibler (KL) divergence \citep{liu2016stein}, the maximum mean discrepancy \citep{arbel2019maximum}, or the kernelised Stein discrepancy \citep{fisher2021measure,korba2021kernel}. Interacting particle systems obtained in this general framework are intended to provide reasonable approximations to the posterior as algorithmic time approaches infinity ($t \rightarrow \infty$) and their convergence can be analysed using the theory of gradient flows on the space of probability distributions \citep{ambrosio2008gradient,villani2008optimal}, see, for instance, \citet{duncan2019geometry,korba2020non, chewi2020svgd, nusken2021stein}.

In this paper, we follow an alternative construction principle, sometimes referred to as the \emph{homotopy method} \citep{daum2008particle,reich2011dynamical}.\footnote{The term `homotopy' is borrowed from the field of topology, where it describes deformations of one continuous function into another.}
Rather than minimising an objective functional and relying on convergence in the long-time limit, from the outset we fix a guiding curve of probability distributions that interpolates between the prior (at $t=0$) and the posterior (at $t=1$). Such interpolations (or homotopies) are routinely used in tempering approaches to Bayesian inference, often in tandem with sequential Monte Carlo \citep{chopin2020introduction,smith2013sequential}, and have recently been explored in connection with diffusion models \citep{vargas2024transport}.  

Similarly to Stein variational gradient descent (SVGD) \citep{liu2016stein}, we seek a driving vector field for the particles from a reproducing kernel Hilbert space (RKHS). Based on a kernel ridge regression type formulation and the prescribed interpolation, we derive update equations that are reminiscent of those of SVGD, and that form the backbone of the proposed scheme, named \emph{Stein transport}.

To clarify the connections between SVGD and Stein transport, we explore their shared geometric foundations rooted in a formal Riemannian structure, initially identified by \citet{liu2017stein} and further elaborated by \citet{duncan2019geometry,nusken2021stein}. Specifically, while SVGD executes a gradient flow of the Kullback-Leibler (KL) divergence on the space of probability distributions -- where this space is endowed with an optimal transport distance induced by the kernel~--  Stein transport can be viewed as implementing infinitesimally optimal transport maps within the same geometrical framework. This perspective not only justifies the name `Stein transport' but also provides an alternative derivation that aligns with the broader ‘sampling via transport’ framework introduced by \citet{el2012bayesian}. The fact that Stein transport can be derived from either a regression or an optimal transport perspective underscores the intrinsic connection between these two approaches: Both are centered around least-squares type objectives (see also \citet{zhu2024approximation}), a structural similarity that might become one of the key aspects in the recently emerging field of statistical optimal transport \citep{chewi2024statistical}.

Implementations of Stein transport are subject to numerical errors arising from both the (Tikhonov) regularisation of the regression problem and finite-particle effects. By analysing the mean-field limit of Stein transport, we examine these sources of inaccuracy and establish connections to birth-death dynamics \citep{lu2019accelerating,lu2023birth} and Fisher-Rao gradient flows \citep{chizat2018unbalanced,liero2018optimal}. Based on this analysis, we propose practical guidelines to minimise these errors. Specifically, we introduce \emph{Adjusted Stein transport}, which alternates between Stein transport and SVGD steps (where the latter are supposed to `adjust' the particles, thereby increasing the accuracy of Stein transport). Our numerical experiments demonstrate that often this approach not only significantly improves on SVGD in terms of accuracy  and computational cost, but also effectively addresses the variance collapse issue that is frequently encountered in SVGD \citep{ba2021understanding}.

\paragraph{Outline.} The paper is organised as follows. In Section \ref{sec:homotopies}, we review the homotopy method and its relation to Stein operators. Section \ref{sec:KRR} formulates an appropriate kernel ridge regression problem, demonstrating that its unique solution can be obtained in closed form (see Proposition \ref{prop:KRR}). The resulting system of equations defining Stein transport is presented in \eqref{eq:Stein ot}. Section \ref{sec:ot} analyses the mean-field limit of Stein transport, covering both the (static) kernel ridge regression problem (Section \ref{sec:inf regression}) and the dynamics of the particle system (Section \ref{sec:mean field dynamics}). We study in particular the impact of regularisation and finite-particle effects on the accuracy of the posterior approximation.
In Section \ref{sec:geometry}, we explore the geometrical aspects of Stein transport and its conceptual connections to SVGD, while Section \ref{sec:works} provides further connections to related work. Section \ref{sec:numerics} details the implementation of Stein transport, and -- based on insights from Section \ref{sec:mean field dynamics} -- develops the adjusted variant of Stein transport (Section \ref{sec:adjusted}). The paper concludes with numerical experiments (Section \ref{sec:num experiments}) and a discussion of conclusions and outlook (Section \ref{sec:outlook}).

\section{Interpolations in Bayesian inference and the Stein operator}
\label{sec:homotopies}

For a probabilistic model $p(x,y)$, Bayesian inference relies on the posterior $p(x|y) = \frac{p(y|x)p(x)}{p(y)}$, where $x \in \mathbb{R}^d$ is a (hidden) parameter of interest, and $y$ represents the available data. Moreover, $p(x)$ is the prior, $p(y|x)$ is the likelihood, and $p(y)$ is an intractable normalisation constant. We suppress the dependence on $y$ and introduce the notation $\pi_0(x) = p(x)$ for the prior, and  $\pi_1(x) = p(x|y)$ for the posterior. The likelihood will be assumed of the form $p(y|x) = \exp(-h(x))$, for a continuous function $h: \mathbb{R}^d \rightarrow \mathbb{R}$, bounded from below, and such that $\int_{\mathbb{R}^d} \exp(-h(x)) \pi_0(\mathrm{d}x) < \infty$. 

The \emph{homotopy method} \citep{daum2008particle,daum2009nonlinear,daum2010exact,heng2021gibbs,reich2011dynamical,reich2012gaussian,reich2022data} posits a continuous deformation in algorithmic time $t \in [0,1]$ that bridges the prior $\pi_0$ and the posterior $\pi_1$ through intermediate distributions $\pi_t$, defined as
\begin{equation}
\label{eq:homotopy}
\pi_t = \frac{e^{-th}\pi_0}{Z_t}, \qquad \qquad t \in [0,1],     
\end{equation}
where 
$Z_t = \int_{\mathbb{R}^d} e^{-th} \, \mathrm{d}\pi_0$
denote the respective normalisation constants. In the context of tempering and simulated annealing \citep[Chapter 17]{chopin2020introduction}, $t$ plays the role of an inverse temperature parameter.  
Presented with samples $X_0^1,\ldots,X_0^N \in \mathbb{R}^d$ 
from the prior $\pi_0$ (to be thought of as particles), our objective is to devise a dynamical scheme that moves those particles in such a way that the corresponding empirical distribution approximately follows the interpolation \eqref{eq:homotopy}, that is, $\frac{1}{N}\sum_{i=1}^N \delta_{X_t^i} \approx \pi_t$, for all $t \in [0,1]$.
Clearly, a successful implementation of this strategy would yield samples $X_1^1,\ldots,X_1^N$ approximately distributed according to the posterior $\pi_1$.

\begin{remark}
\label{rem:other interpolations}
The specific form of the interpolation \eqref{eq:homotopy} is not essential for the developments in this paper, as long as  $\pi_0$ and $\pi_1$ correspond to prior and posterior, respectively. It is the case, however, that the specific form of \eqref{eq:homotopy} can be linked to Fisher-Rao and Newton gradient flows \citep{chen2023sampling,chopin2023connection,domingo2023explicit,lu2023birth}, see Section \ref{sec:gradient flows} for a discussion. On the other hand, \citet{syed2021parallel} demonstrate benefits of alternative interpolations in the context of parallel tempering, and exploring these possibilities is an interesting avenue for future work.   
\end{remark}

In order to construct such schemes, we will make essential use of the \emph{Stein operators} $S_\pi: C^1(\mathbb{R}^d;\mathbb{R}^d) \rightarrow C(\mathbb{R}^d;\mathbb{R})$, defined via
\begin{equation}
\label{eq:Stein op}
S_\pi v := \frac{1}{\pi} \nabla \cdot (\pi v) = \nabla \log \pi \cdot v + \nabla \cdot v,
\end{equation}
which can be associated to any strictly positive and continuously differentiable probability density $\pi$ (see \cite{anastasiou2023stein} for an overview). Here and in what follows, $\nabla \cdot v := \sum_{i=1}^d \partial_{x_i}v_i$ denotes the divergence of $v$.
Importantly for applications in Bayesian inference, \eqref{eq:Stein op} can be computed given a vector field $v$ without access to the normalisation constant of $\pi$: Indeed, if $\widetilde{\pi} = Z \pi$ is an unnormalised version of $\pi$, then $\nabla \log \widetilde{\pi} = \nabla \log \pi$. 
The relevance of $S_\pi$ for the present paper derives from the following result:
\begin{proposition}[Stein equation]
\label{prop:Stein eq}
Assume that the time-dependent vector field $v_t \in C^1(\mathbb{R}^d;\mathbb{R}^d)$ satisfies the \emph{Stein equation}
\begin{equation}
\label{eq:Stein eq}
S_{\pi_t} v_t = h - \int_{\mathbb{R}^d} h \, \mathrm{d}\pi_t,
\end{equation}
for all $t \in [0,1]$. Then the ordinary differential equation (ODE) with random initial condition $\pi_0$,
\begin{equation}
\label{eq:ODE}
\frac{\mathrm{d}X_t}{\mathrm{d}t} = v_t(X_t), \qquad \qquad X_0 \sim \pi_0,
\end{equation}
reproduces the interpolation \eqref{eq:homotopy}, in the sense that $\mathrm{Law}(X_t) = \pi_t$, for all $t \in [0,1]$, whenever \eqref{eq:ODE} is well posed.
\end{proposition}
\begin{proof}
See, for instance, \citet{daum2010exact}, \citet{reich2011dynamical} or \citet{heng2021gibbs}. For the reader's convenience we provide a proof in Appendix \ref{sec:misc}.
\end{proof}

\begin{remark}[Nonuniqueness]
Solutions to \eqref{eq:Stein eq} are not unique; indeed from ${S_\pi v = \frac{1}{\pi}\nabla\cdot(\pi v)}$ we see that any $\pi$-weighted divergence-free vector field is in the kernel of $S_\pi$, and may hence be added to a solution of \eqref{eq:Stein eq} without affecting its validity.
Any such solution is permissible in principle, reproducing the interpolation \eqref{eq:homotopy} at the level of time-marginals, but leading to different trajectories (or measures on path space). In terms of (optimal) transport, the Stein equation \eqref{eq:Stein eq} can therefore be thought of as encoding marginal constraints, without selecting any optimality principle.
\end{remark}

Although \eqref{eq:ODE} is formulated for random initial conditions $\pi_0$, corresponding particle movement schemes can often be obtained by approximating $\pi_0$ with an empirical measure and applying the flow \eqref{eq:ODE} to each particle separately (see \eqref{eq:Stein ot} below).
We mention in passing that the homotopy method allows us to compute the normalising constant $Z_1 = p(y)$, commonly used for Bayesian model selection, via the identity
\begin{equation*}
Z_1 = \exp \left(- \int_0^1  \int_{\mathbb{R}^d} h \, \mathrm{d}\pi_t \, \mathrm{d}t\right),   
\end{equation*}
see \citet{gelman1998simulating,oates2016controlled}, or the proof of Proposition \ref{prop:Stein eq} in Appendix \ref{appsec:KRR_proof}.

\section{Solving the Stein equation \eqref{eq:Stein eq} using kernel ridge regression}
\label{sec:KRR}

Proposition \ref{prop:Stein eq} shifts the challenge of approximating the posterior $\pi_1$ to the problem of obtaining suitable (approximate) solutions to the Stein equation \eqref{eq:Stein eq}, given samples $X_t^1,\ldots,X_t^N \in \mathbb{R}^d$, distributed (approximately) according to $\pi_t$. In this section, we present an approach based on a formulation in the spirit of kernel ridge regression \citep{kanagawa2018gaussian}, seeking vector fields $v_t$ in a reproducing kernel Hilbert space (RKHS) $\mathcal{H}_k^d$. To set the stage, we recall the relevant background (see, for example, \citet{smola1998learning}, \citet{kanagawa2018gaussian} or \citet[Chapter 4]{steinwart2008support} for more details):
\paragraph{Preliminaries on positive definite kernels and RKHSs.} Throughout, we consider positive definite kernels; those are bivariate real-valued functions ${k: \mathbb{R}^d \times \mathbb{R}^d \rightarrow \mathbb{R}}$ that are symmetric, $k(x,y) = k(y,x)$, for all $x,y \in \mathbb{R}^d$, and that satisfy the positive (semi-)definiteness condition $\sum_{i,j=1}^N \alpha_i \alpha_j k(x_i,x_j) \ge 0$, for all $\alpha_1,\ldots,\alpha_N \in \mathbb{R}$, $x_1,\ldots,x_N \in \mathbb{R}^d$, and $N\in \mathbb{N}$.
We will assume that $k \in C^{1,1}(\mathbb{R}^d \times \mathbb{R}^d;\mathbb{R})$, that is, for all $i,j=1,\ldots,d$, the mixed partial derivatives $\partial_{x_i} \partial_{y_j}k(x,y)$ exist and are continuous.
The reproducing kernel Hilbert space (RKHS) corresponding to $k$ is a Hilbert space of real-valued functions on $\mathbb{R}^d$ and will be denoted by $(\mathcal{H}_k, \langle \cdot,\rangle_{\mathcal{H}_k})$. 
It is characterised by 
the conditions that $k(x,\cdot) \in \mathcal{H}_k$, for all $x \in \mathbb{R}^d$, as well as $\langle f, k(x,\cdot) \rangle_{\mathcal{H}_k} = f(x)$, for all $x \in \mathbb{R}^d$ and $f \in \mathcal{H}_k$ \citep[Section 4.2]{steinwart2008support}.
The $d$-fold Cartesian product $$\mathcal{H}_k^d = \underbrace{\mathcal{H}_k \times \ldots \times \mathcal{H}_k}_{d \, \mathrm{times}}$$ consists of vector fields $v = (v_1,\ldots,v_d)$ with component functions $v_i \in \mathcal{H}_k$ and is equipped with the inner product $\langle u, v \rangle_{\mathcal{H}_k^d} := \sum_{i=1}^d \langle u_i,v_i \rangle_{\mathcal{H}_k}$.

We are now in a position to present our kernel-based approach towards approximating solutions to the Stein equation \eqref{eq:Stein eq}. The following formulation seeks to minimise the difference between the left- and right-hand sides of \eqref{eq:Stein eq}, based on available samples:
\begin{problem}[Kernel ridge regression for the Stein equation \eqref{eq:Stein eq}]
\label{prob:KRR}
Given samples $X^1,\ldots, X^N \in \mathbb{R}^d$, a strictly positive probability density $\pi \in C^1(\mathbb{R}^d;\mathbb{R}_{>0})$, and a regularisation parameter $\lambda > 0$, find a solution to
\begin{equation}
\label{eq:KRR}
    v^* \in \argmin_{v \in \mathcal{H}^d_k} \left( \frac{1}{N} \sum_{j=1}^N \left( (S_\pi v)(X^j) - h_0(X^j) \right)^2  + \lambda \Vert v \Vert_{\mathcal{H}^d_k}^2 \right),
\end{equation}
where $h_0$ refers to the 
data-centred version of $h$, that is, $h_0(x) := h(x) - \frac{1}{N}\sum_{j=1}^N h(X^j)$. 
\end{problem}
The regularising term $\lambda \Vert v \Vert_{\mathcal{H}_k^d}^2$ guarantees strict  convexity of the objective in \eqref{eq:KRR} and will turn out to stabilise the associated numerical procedure (in particular, the inversion of a linear system). In our experiments, we typically choose $\lambda$ to be small, say
$\lambda
\approx 10^{-3}$. Larger values of $\lambda$ may be appropriate when the likelihood itself is noisy or a stronger form of regularisation appears to be suitable for other reasons (see, for instance, \citet{dunbar2022ensemble}). The formulation in Problem \ref{prob:KRR} is classical if $S_\pi$ is replaced by the identity operator, the task in this case being to recover an underlying target function from noisy measurements \citep[Section 3.2]{kanagawa2018gaussian}. We would also like to point the reader to the work of  \citet{zhu2024approximation}, where regression type formulations are used to construct dynamical schemes for inference (but the specifics of our formulation are quite different).
In the context of PDEs, this approach is often linked to collocation methods, see
\cite[Chapter 16]{wendland2004scattered}, \citet[Chapter 38]{fasshauer2007meshfree}; for a recent nonlinear extension, see \citet{chen2021solving}.

\textbf{The KSD-kernel.} In order order to address Problem \ref{prob:KRR}, recall the following from the theory of \emph{kernelised Stein discrepancies} (KSD) \citep{chwialkowski2016kernel,liu2016kernelized,gorham2017measuring}: 

Given a positive definite kernel $k$ and a score function $\nabla \log \pi$, we can construct a new  positive definite kernel
$\xi^{k,\nabla \log \pi}: \mathbb{R}^d \times \mathbb{R}^d \rightarrow \mathbb{R}$, defined as 
\begin{subequations}
\label{eq:xi}
\begin{align}
     \xi^{k,\nabla \log \pi}(x,y)  := S_\pi^x S_\pi^y k(x,y) = & \, \,\,\nabla \log \pi(x) \cdot \nabla_y k(x,y) + 
    \nabla \log \pi(y) \cdot \nabla_x k(x,y)
    \nonumber
\\    
  & + \nabla_x \cdot \nabla_y k(x,y) + \nabla \log \pi(x) \cdot k(x,y) \nabla \log \pi(y),
  \nonumber \tag*{(\ref{eq:xi})}
\end{align}
\end{subequations}
in the following referred to as the \emph{KSD-kernel} associated to $k$ and $\nabla \log \pi$. In \eqref{eq:xi} we have used the notation
${\nabla_x \cdot \nabla_y k(x,y) := \sum_{i=1}^d \partial_{x_i} \partial_{y_i}k(x_i,y_i)}$, and $S_\pi^x$  (respectively $S_\pi^y$) is understood to act on the variable $x$ (respectively $y$) only. The salient feature of $\xi^{k, \nabla \log \pi}$ is that it integrates to zero against $\pi$,
\begin{equation*}
\int_{\mathbb{R}^d} \xi^{k,\nabla \log \pi}(\cdot,y) \pi(\mathrm{d}y) = 0,
\end{equation*}
and that, as a consequence, the kernelised Stein discrepancy 
\begin{equation}
\label{eq:KSD}
\mathrm{KSD}(\mu | \pi) = \int_{\mathbb{R}^d} \int_{\mathbb{R}^d} \xi^{k,\nabla \log \pi}(x,y) \mu(\mathrm{d}x)\mu(\mathrm{d}y)
\end{equation}
can be used to measure similarity between the two probability measures $\mu$ and $\pi$;  see \citet[Theorem 2.2]{chwialkowski2016kernel}  and \citet[Theorem 3.6]{liu2016kernelized}.

The minimisation in \eqref{eq:KRR}
admits a unique solution that can now be represented in closed form: We define the Gram matrix $\bm{\xi}^{k,\nabla \log \pi} \in \mathbb{R}^{N \times N}$ and the vector $\bm{h}_0 \in \mathbb{R}^N$ obtained from evaluating the KSD-kernel $\xi^{k,\nabla \log \pi}$  
and the centred negative log-likelihood $h_0$ (defined in Problem \ref{prob:KRR}) at the locations of the particles,
\begin{equation}
\label{eq:matrices}
(\bm{\xi}^{k,\nabla \log \pi})_{ij} = \xi^{k,\nabla \log \pi}(X^i,X^j) \in \mathbb{R}^{N \times N}, \qquad (\bm{h}_0)_i = h(X^i) \in \mathbb{R}^N.
\end{equation}
Furthermore, we denote by $I_{N \times N} \in \mathbb{R}^{N \times N}$ the identity matrix of dimension $N$. Using this notation, we have the following result.
\begin{proposition}
\label{prop:KRR}
For $\lambda >0$, the kernel ridge regression problem \eqref{eq:KRR} admits a unique solution, given by
\begin{equation}
\label{eq:v star}
v^* = \frac{1}{N} \sum_{j=1}^N \phi^j\left( k(\cdot,X^j)\nabla \log \pi(X^j) + \nabla_{X^j} k(\cdot,X^j)\right),
\end{equation}
where $(\phi^j)^N_{j=1} = \phi \in \mathbb{R}^N$ is the unique solution to the linear system
\begin{equation}
\label{eq:linear system}
(\tfrac{1}{N}\bm{\xi}^{k,\nabla \log \pi} + \lambda I_{N \times N}) \phi = \bm{h}_0. 
\end{equation}
\end{proposition}
\begin{remark}
Since $\xi^{k,\nabla \log \pi}$ is a positive definite kernel, the matrix $\bm{\xi}^{k,\nabla \log \pi} + \lambda I_{N \times N}$ is invertible, and hence \eqref{eq:linear system} determines $\phi$ uniquely.
\end{remark}
\begin{proof}
There are (at least) two ways of arriving at the representation of $v^*$ given by \eqref{eq:v star} and \eqref{eq:linear system}: Firstly, we may rely on (a version of) the representer theorem \citep{wahba}, leveraging the fact that for each $j=1,\ldots,N$, the mapping $v \mapsto (S_\pi v)(X^j)$ is a continuous linear functional on $\mathcal{H}_k^d$, and that the corresponding Riesz representers are given by $k(\cdot,X^j) \nabla \log \pi(X^j) + \nabla_{X^j}k(\cdot,X^j)$. Alternatively, we can think of \eqref{eq:KRR} as a Tikhonov-regularised least squares problem \citep[Section 2.2]{kirsch2021introduction}, whose solution $S_{\pi,N}^*(\lambda I_{N \times N} + S_{\pi,N} S_{\pi,N}^*)^{-1}\bm{h}_0$ can be identified with $v^*$.\footnote{The sample versions $S_{\pi,N}$ of the Stein operator $S_\pi$ will be introduced in Appendix \ref{appsec:KRR_proof}.} We detail both approaches in  Appendix \ref{appsec:KRR_proof}.
\end{proof}
Combining Propositions \ref{prop:Stein eq} and \ref{prop:KRR}, we obtain the following interacting particle system, defining Stein transport: 
\begin{equation}
\label{eq:Stein ot}
\left\{
\begin{array}{rl}
    & \frac{\mathrm{d}X_t^i}{\mathrm{d}t} =  \frac{1}{N} \sum_{j=1}^N \phi_t^j \left(k(X_t^i,X_t^j) \nabla \log \pi_t (X_t^j) + \nabla_{X_t^j} k(X_t^i,X_t^j)\right), \qquad t \in [0,1],
    \vspace{0.1cm}
    \\
    & (\tfrac{1}{N}\bm{\xi}^{k,\nabla \log \pi_t} + \lambda I_{N \times N}) \phi_t  = \bm{h}_{0,t}. 
    \end{array}
    \right.
\end{equation}
Note that for $\bm{\xi}^{k,\nabla \log \pi_t}$ and the update $\frac{\mathrm{d}X_t^i}{\mathrm{d}t}$ in \eqref{eq:Stein ot}, the score function $\nabla \log \pi_t$ is assumed to be calculated according to the interpolation \eqref{eq:homotopy}, that is,
$\nabla \log \pi_t = -t \nabla h + \nabla \log \pi_0$, where $\nabla \log \pi_0$ is typically tractable as the score function associated to the prior (but see Section \ref{sec:KMED} for remarks on the score-free setting). The vector $(\bm{h}_{0,t})_i = h(X^i_t) - \frac{1}{N}\sum_{i=1}^N h(X_t^i)$ is time-dependent through evaluations at the particles' locations. A standard Euler discretisation of \eqref{eq:Stein ot} now yields an implementable procedure for Bayesian inference that we summarise in Algorithm \ref{alg:stein}.  

\begin{algorithm}[th]
  \renewcommand{\algorithmicrequire}{\textbf{Input:}}
  \renewcommand{\algorithmicensure}{\textbf{Output:}}
  \caption{Stein transport\label{alg:stein}}
\begin{algorithmic}
\Require{Prior samples $X_0^1,\ldots,X_0^N \in \mathbb{R}^d$, prior score function $\nabla \log \pi_0: \mathbb{R}^d \rightarrow \mathbb{R}^d$, negative log-likelihood ${h:\mathbb{R}^d \rightarrow \mathbb{R}}$, with gradient $\nabla h: \mathbb{R}^d \rightarrow \mathbb{R}^d$, positive definite kernel $k$, regularisation parameter $\lambda > 0$, time discretisation $0 = t_0 < t_1 < \ldots < t_{N_\mathrm{steps}} = 1$ with time step $\Delta t$.}
    \vspace{0.2cm}
    \For{$n= 0, \ldots,  N_{\mathrm{steps}} - 1$}
    \vspace{0.1cm}
      \State Compute the scores  $P_n^j = -t_n \nabla h (X_n^j) + \nabla \log \pi_0(X_n^j).$
      \vspace{0.1cm}
      \State Compute the KSD Gram matrix
      \vspace{-0.2cm}
      $$
      \begin{array}{rl}
      \bm{\xi}_{ij}  = & P_n^i \cdot \nabla_{X_n^j} k(X_n^i,X_n^j) + P_n^j \cdot \nabla_{X_n^i} k(X_n^i,X_n^j)
      \\
      & + \nabla_{X_n^i} \cdot \nabla_{X_n^j} k(X_n^i, X_n^j) + P_n^i \cdot k(X_n^i,X_n^j) P_n^j.
      \end{array}$$
      \State Compute the centred negative log-likelihoods $ \bm{h}_i = h(X^i_n) - \frac{1}{N} \sum_{i=1}^N h(X_n^i)$.
      \vspace{0.1cm}
      \State Solve $(\tfrac{1}{N}\bm{\xi} + \lambda I_{N \times N}) \phi = \bm{h}$ for $\phi \in\mathbb{R}^N$.
      \vspace{0.1cm}
      \State $X^i_{n+1} \gets X^i_{n} +  \frac{\Delta t }{N} \sum_{j=1}^N \phi^j \left(k(X_n^i,X_n^j) P_n^j (X_n^j) + \nabla_{X_n^j} k(X_n^i,X_n^j)\right).$
    \EndFor
    \algcustomendfor
    \vspace{0.1cm}
    \State 
    \Return{approximate posterior samples $X_{N_{\mathrm{steps}}}^1,\ldots X_{N_\mathrm{steps}}^N$.}
  \end{algorithmic}
\end{algorithm}

\paragraph{Comparison to SVGD.} The system \eqref{eq:Stein ot} bears a remarkable similarity to 
\begin{equation}
\label{eq:svgd}
\frac{\mathrm{d}X_t^i}{\mathrm{d}t} =  \frac{1}{N} \sum_{j=1}^N \left(k(X_t^i,X_t^j) \nabla \log \pi (X_t^j) + \nabla_{X_t^j} k(X_t^i,X_t^j)\right), \qquad t \in [0,\infty),
\end{equation}
the governing equations of Stein variational gradient descent (SVGD) introduced by \citet{liu2016stein}. Postponing a more conceptual (geometric) comparison to Section \ref{sec:geometry}, we note that \eqref{eq:svgd} involves the score function $\nabla \log \pi$ associated to the target (that is, $\nabla \log \pi_1$ in the notation of this paper), whereas \eqref{eq:Stein ot} is driven by the time-dependent score function $\nabla \log \pi_t$ induced by the interpolation \eqref{eq:homotopy}. Stein transport furthermore relies on the vector $\phi_t \in \mathbb{R}^N$, the components of which can be interpreted as weights attached to the particles (we will provide a more in-depth discussion of the the role of $\phi$ in Section \ref{sec:ot}). In terms of computational cost, Stein transport requires the assembly and inversion of the $N$-dimensional linear system $(\tfrac{1}{N}\bm{\xi} + \lambda I_{N \times N}) \phi  = \bm{h}$, where $N$ is the number of particles. We would like to stress, however, that Stein transport and SVGD demand the same number of gradient evaluations of the log-likelihood per time step, and that for moderately sized particle systems (say $N \approx 10^3$), the inversion of a linear system only incurs a very minor computational overhead. For large-scale applications, speed-ups may be achieved using random feature expansions of $k$ \citep{rahimi2007random} or projections combined with preconditioning \citep{rudi2017falkon}.

Despite these similarities, Stein transport and SVGD exhibit important differences that affect their numerical performance. First, Stein transport achieves its posterior approximation at $t=1$, whereas SVGD relies on long-time convergence towards a fixed point in \eqref{eq:svgd}. As demonstrated experimentally in Section \ref{sec:numerics}, the finite-time property of Stein transport often significantly reduces the number of necessary time steps to achieve a prescribed accuracy. Perhaps more importantly, Stein transport does not suffer from particle collapse in high dimensions as does SVGD \citep{ba2021understanding}. This is further elaborated in Proposition \ref{prop:projection}, Remark \ref{rem:SVGD collapse} and Section \ref{sec:collapse} below. 

\paragraph{Stein transport without gradients.} We mention in passing that Stein transport can be implemented without gradient evaluations of the log-likelihood. Indeed, the time-dependent score may be evolved along the particle flow, starting from $\nabla \log \pi_0$: 
\begin{lemma}[Evolution of the score function]
\label{lem:score}
Let $v_t \in C^2(\mathbb{R}^d;\mathbb{R}^d)$ be a family of vector fields such that 
\begin{equation*}
\frac{\mathrm{d}X_t}{\mathrm{d}t} = v_t(X_t), \qquad  \qquad X_0 \sim \pi_0,
\end{equation*}
is well posed, and denote the corresponding laws by $\pi_t = \mathrm{Law} (X_t)$. Assume for all $t \ge 0$ that $\pi_t \in C^1(\mathbb{R}^d;\mathbb{R}) $ with $\pi_t > 0$ and define $P_t := \nabla \log \pi_t(X_t)$. Then it holds that 
\begin{equation*}
\frac{\mathrm{d}P_t}{\mathrm{d}t} = - \nabla (\nabla \cdot v_t)(X_t) - (\nabla v_t)(X_t) P_t.
\end{equation*}
\end{lemma}
\begin{proof}
This follows by direct calculation, see Appendix \ref{sec:misc}.
\end{proof}
Based on \eqref{eq:v star}, the terms $\nabla (\nabla \cdot v^*)$ and $\nabla v^*$ can be computed in closed form, applying the differential operators to the kernel $k$. The corresponding gradient-free interacting particle system is thus given by
\begin{equation}
\label{eq:Stein ot wo gradients}
\left\{
\begin{array}{rl}
    & \frac{\mathrm{d}X_t^i}{\mathrm{d}t} =  \frac{1}{N} \sum_{j=1}^N \phi_t^j \left(k(X_t^i,X_t^j) P_t^j + \nabla_{X_t^j} k(X_t^i,X_t^j)\right), \qquad t \in [0,1],
    \vspace{0.1cm}
    \\
    & \frac{\mathrm{d}P_t^i}{\mathrm{d}t}  =
    - \frac{1}{N} \sum_{j=1}^N \phi_t^j \left(\nabla^2_{X_t^i}k(X_t^i,X_t^j) P_t^j + \nabla_{X_t^i} \left( \nabla_{X_t^i} \cdot \nabla_{X_t^j} k(X_t^i,X_t^j)\right) \right),

    \vspace{0.1cm}
    \\
    & (\tfrac{1}{N}\bm{\xi}^{k,\nabla \log \pi_t} + \lambda I_{N \times N}) \phi_t  = \bm{h}_{0,t}, 
    \end{array}
    \right.
\end{equation}
where the initial samples $(X_0^i)_{i=1}^N$ are taken from the prior $\pi_0$, and the scores are initialised as $P_0^i = \nabla \log \pi_0(X_0^i)$. Using \eqref{eq:Stein ot wo gradients}, Algorithm \ref{alg:stein} can straightforwardly be extended so that it does not require gradient evaluations of $h$.

\section{Convergence and mean-field description}
\label{sec:ot}

In this section, we investigate the infinite-particle limit of Stein transport, deriving formulations based on the limiting measures $\pi_t = \lim_{N \rightarrow \infty} \tfrac{1}{N}\sum_{i=1}^N \delta_{X_t^i}$. In Section \ref{sec:inf regression}, we first examine the infinite-particle version of the  kernel ridge regression Problem \ref{prob:KRR}. Building on this, we draw qualitative and quantitative conclusions about the dynamics \eqref{eq:Stein ot} and its mean-field version in Section \ref{sec:mean field dynamics}, particularly discussing the errors incurred by regularisation ($\lambda > 0$) and finite-particle approximation ($N < \infty$).   

\subsection{Infinite-particle kernel ridge regression}

\label{sec:inf regression}

Formally taking the limit $\tfrac{1}{N}\sum_{i=1}^N \delta_{X^i} \xrightarrow{N \rightarrow \infty} \pi$ in \eqref{eq:KRR}, we arrive at the following mean-field version of the kernel ridge regression task stated as Problem \ref{prob:KRR}: 

\begin{problem}[Kernel ridge regression for the Stein equation \eqref{eq:Stein eq}, mean-field version]
\label{prob:KRR mean field}
Given a strictly positive probability density $\pi \in C^1(\mathbb{R}^d;\mathbb{R})$ and a regularisation parameter $\lambda > 0$, find a solution to
\begin{equation}
\label{eq:KRR mean field}
    v^* \in \argmin_{v \in \mathcal{H}^d_k} \left( \int_{\mathbb{R}^d} \left( (S_\pi v) - h_{0,\infty} \right)^2 \mathrm{d}\pi + \lambda \Vert v \Vert_{\mathcal{H}^d_k}^2 \right),
\end{equation}
where $h_{0,\infty}$ refers to the 
$\pi$-centred version of $h$, that is, $h_{0,\infty}(x) := h(x) - \int_{\mathbb{R}^d} h\,\mathrm{d}\pi$. 
\end{problem}
Notice that Problem \ref{prob:KRR mean field} is in fact a generalisation (and not only a mean-field version) of Problem \ref{prob:KRR}: Choosing ${\pi = \tfrac{1}{N} \sum_{i=1}^N} \delta_{X^i}$ in \eqref{eq:KRR mean field} leads back to \eqref{eq:KRR}, and in fact both formulations could be treated on an equal footing, as done by \citet[Chapter 5]{steinwart2008support}  for conventional kernel ridge regression.
To address Problem \ref{prob:KRR mean field}, in the remainder of this section we will assume the following:
\begin{assumption}
\label{ass:basic}
The probability density $\pi \in C^1(\mathbb{R}^d;\mathbb{R})$ is strictly positive, and such that $\Vert h \Vert_{L^2(\pi)} < \infty$ as well as ${\Vert \nabla \log \pi \Vert_{(L^2(\pi))^d} < \infty}$. The kernel $k$ is bounded, with bounded first-order partial derivatives. 
\end{assumption}
The solution to Problem \ref{prob:KRR mean field} will be given in terms of the integral operator
\begin{equation}
\label{eq:def Tk}
\mathcal{T}_{k,\pi}\phi = \int_{\mathbb{R}^d} k(\cdot,y) \phi(y) \pi(\mathrm{d}y), \qquad \phi \in L^2(\pi),    
\end{equation}
and we refer to \citet[Chapter 4.3]{steinwart2008support} for some of its properties. Interpreting  \eqref{eq:def Tk} componentwise, we can also apply $\mathcal{T}_{k,\pi}$ to vector-valued functions, that is, when $\phi \in (L^2(\pi))^d$. In parallel to Proposition \ref{prop:KRR}, we can now characterise the unique solution to Problem \ref{prob:KRR mean field}:

\begin{proposition}
\label{prop:KRR mean field}
Assume that $\pi$, $h$ and $k$ satisfy Assumption \ref{ass:basic}.  Then there exists a unique solution to \eqref{eq:KRR mean field}, given by $v^* = -\mathcal{T}_{k,\pi} \nabla \phi$, where $\phi \in L^2(\pi)$ is the unique solution to the \emph{Stein-Poisson equation} 
\begin{equation}
\label{eq:Stein Poisson}
-\nabla \cdot (\pi \mathcal{T}_{k,\pi} \nabla \phi) + \lambda \pi \phi =  \pi \left( h - \int_{\mathbb{R}^d} h \, \mathrm{d}\pi\right).
\end{equation}
\end{proposition}
\begin{proof}
As in the proof of Proposition \ref{prop:KRR}, we can interpret \eqref{eq:KRR mean field} as a Tikhonov-regularised least-squares problem. Details can be found in Appendix \ref{app:sec4 proofs}.  
\end{proof}

The Stein-Poisson equation \eqref{eq:Stein Poisson} is a kernelised and regularised version of the (conventional) Poisson equation 
\begin{equation}
\label{eq:Poisson}
-\nabla \cdot (\pi \nabla \phi) = \pi\left(h - \int _{\mathbb{R}^d} h \, \mathrm{d}\pi\right),
\end{equation}
which governs central limit theorems and fluctuations in Markov processes \citep{komorowski2012fluctuations,pardoux2001poisson} and also features prominently in McKean-Vlasov approaches to nonlinear filtering \citep{laugesen2015poisson,pathiraja2021mckean,coghi2023rough} as well as in the development of particle-flow algorithms on the basis of Proposition \ref{prop:Stein eq} \citep{heng2021gibbs,reich2011dynamical,taghvaei2016gain,radhakrishnan2019gain,taghvaei2020diffusion,maurais2023adaptive,maurais2024sampling,wang2024measure,taghvaei2020diffusion,tian2024liouville}. We mention in passing that the Stein-Poisson equation \eqref{eq:Stein Poisson} has implicitly been used by \citet{oates2014control} to construct control variates for Monte Carlo estimators.

The following two remarks compare \eqref{eq:Stein Poisson} and \eqref{eq:Poisson} in terms of well-posedness and finite-particle approximations: In a nutshell, the Stein-Poisson equation \eqref{eq:Stein Poisson}  is inferior in terms of well-posedness, but superior in terms of finite-particle approximations. We will comment on the geometrical underpinnings of \eqref{eq:Stein Poisson} and \eqref{eq:Poisson} in Section \ref{sec:geometry}.

\begin{remark}[Well-posedness]
\label{rem:well posedness}
Although  at first glance the expression on the left-hand side of \eqref{eq:Stein Poisson} is only well defined if $\phi$ is differentiable, the operator $\phi \mapsto \nabla \cdot (\pi \mathcal{T}_{k,\pi} \nabla \phi)$ can be extended to a bounded linear operator on $L^2(\pi)$, see Lemma \ref{lem:compact} in Appendix \ref{app:sec4 proofs}. Also note that $\lambda > 0$ is necessary (and sufficient) for \eqref{eq:Stein Poisson} to be well posed. Indeed, since the operator $\phi \mapsto -\nabla \cdot (\pi \mathcal{T}_{k,\pi} \nabla \phi)$ is compact in $L^2(\pi)$ according to \citet[proof of Lemma 23]{duncan2019geometry}, the Stein-Poisson equation \eqref{eq:Stein Poisson} is ill posed for $\lambda = 0$, in the sense of \citet[Section 1.2]{kirsch2021introduction}: A solution $\phi$ will in general not exist, and, even if it does, it will necessarily be unstable with respect to small perturbations of $h$. In contrast, the Poisson equation \eqref{eq:Poisson} is well posed without regularisation, under mild conditions on the tails of $\pi$ (see, for instance, \citet[Corollary 2.4]{lelievre2016partial}).
\end{remark}

\begin{remark}[Approximation by empirical measures]
\label{rem:SP emp measures}
In the same vein as \eqref{eq:KRR mean field}, the Stein-Poisson equation \eqref{eq:Stein Poisson} is in fact meaningful for $\pi = \tfrac{1}{N}\sum_{i=1}^N \delta_{X^i}$. Indeed, dividing \eqref{eq:Stein Poisson} by $\pi$ and manipulating the left-hand side, we arrive at 
\begin{equation*}
\int_{\mathbb{R}^d} \xi^{k,\nabla \log \pi_t} (\cdot,y) \phi(y) \pi(\mathrm{d}y) + \lambda \phi = h - \int_{\mathbb{R}^d} h \, \mathrm{d}\pi,  \end{equation*}
which makes sense for $\pi = \tfrac{1}{N}\sum_{i=1}^N \delta_{X^i}$, and reduces to the linear system \eqref{eq:linear system} under this substitution (note that we need to keep $\nabla \log \pi_t$ as is, as an `exterior input'). In contrast, there is no direct interpretation for $\pi = \tfrac{1}{N} \sum_{i=1}^N \delta_{X^i}$ of the conventional Poisson equation \eqref{eq:Poisson}, and additional approximations (or finite dimensional projections) are necessary to obtain a particle-based formulation. This difficulty can be traced back to the following observation: The Poisson equation \eqref{eq:Poisson} formally describes minimisers in 
\begin{equation*}
    v^* \in \argmin_{v \in L^2(\pi)} \left( \int_{\mathbb{R}^d} \left( (S_\pi v) - h_0 \right)^2 \mathrm{d}\pi + \lambda \Vert v \Vert_{L^2(\pi)}^2 \right),
\end{equation*}
for $\lambda \rightarrow 0$, replacing $\Vert \cdot \Vert_{\mathcal{H}_k^d}$ in \eqref{eq:KRR mean field} by $\Vert \cdot \Vert_{(L^2(\pi))^d}$. Unfortunately, this replacement cannot be done in \eqref{eq:KRR}, since $S_\pi v (X^i)$ would be ill defined (point evaluations are not possible for $v \in (L^2(\pi))^d$).
\end{remark}

We end this subsection by establishing rigorous correspondences between the finite-sample kernel ridge regression in Problem \ref{prob:KRR} and its mean-field version in Problem \ref{prob:KRR mean field}. Under appropriate growth conditions on $\nabla \log \pi$ and $h$, the corresponding optimal vector fields converge in the $\mathcal{H}_k^d$-norm: 
\begin{thm}[Connections between Problems \ref{prob:KRR} and \ref{prob:KRR mean field}]
\label{thm:connections}
Let $\pi$, $h$ and $k$ satisfy Assumption \ref{ass:basic}. Furthermore, assume that there exist constants $C>0$ and $p >1$ such that $\nabla \log \pi$ and $h$ satisfy the growth conditions
\begin{equation*}
|\nabla \log \pi(x) | \le C(1 + |x|^{p/2}) \quad \text{and} \quad |h(x)| \le C(1+ |x|^{p/2}), \qquad \text{for all} \, \, x \in \mathbb{R}^d.
\end{equation*}
Let $(X^i)_{i=1}^\infty \subset \mathbb{R}^d$ be a sequence of points such that the empirical measures $\frac{1}{N}\sum_{i=1}^N \delta_{X^i}$ converge to $\pi$ in the $p$-Wasserstein distance, as $N \rightarrow \infty$.\footnote{Note that this implicitly implies that the $p^{\mathrm{th}}$ moment of $\pi$ is finite.} Denote by $v^*_{N,\lambda}$ and $v^*_{\infty,\lambda}$ the minimisers of \eqref{eq:KRR} and \eqref{eq:KRR mean field}, respectively, and fix $\lambda > 0$. Then $v^*_{N,\lambda}$ converges to   $v^*_{\infty,\lambda}$ in $\mathcal{H}_k^d$ as $N \rightarrow \infty$, that is, \begin{equation*}\Vert v^*_{N,\lambda}- v^*_{\infty,\lambda}\Vert_{\mathcal{H}_k^d} \xrightarrow{N \rightarrow \infty} 0. \end{equation*} 
\end{thm} 
\begin{proof}
See Appendix \ref{app:sec4 proofs}. The proof adapts techniques and results from statistical learning theory \citep{smale2007learning,blanchard2018optimal} to the setting of Problems \ref{prob:KRR} and \ref{prob:KRR mean field}.  
\end{proof}
\begin{remark}
Under Assumption \ref{ass:basic}, convergence in $\mathcal{H}_k^d$ implies uniform convergence of $v^*_{N,\lambda}$ and its first derivatives, see \citet[Lemmas 4.23 and 4.34]{steinwart2008support}.
\end{remark}
Our next result establishes the sense in which the solution to Problem \ref{prob:KRR} (given in Proposition \ref{prop:KRR}) provides an approximate solution to the Stein equation \eqref{eq:Stein eq}. Clearly, in order to obtain satisfactory guarantees, the `search space' $\mathcal{H}_k^d$ has to be sufficiently expressive. We formalise this intuition as follows:  

\begin{assumption}[Universality]
\label{ass:universality}
The inclusion $\mathcal{H}_k^d \subset (L^2(\pi))^d$ is dense.
\end{assumption}
The Gaussian, Laplacian, inverse multiquadratic and Mat{\'e}rn kernels are universal in the sense of Assumption \ref{ass:universality}, under mild conditions on $\pi$, and we refer the reader to \citet{sriperumbudur2011universality} for comprehensive results and a detailed overview. We can now state the following result, showing that indeed Proposition \ref{prop:KRR} yields a good approximate solution to the Stein equation \eqref{eq:Stein eq}, provided that $\lambda$ is small and $N$ is large enough.

\begin{thm}
\label{thm:consistency}
Assume the setting from Theorem \ref{thm:connections}, and furthermore that Assumption \ref{ass:universality} is satisfied. Then, for all $\varepsilon > 0$ there exists $\lambda > 0$ and $N_0 \in \mathbb{N}$ such that
\begin{equation}
\label{eq:Stein accuracy}
\left\Vert S_\pi v^*_{N,\lambda} - \left(h - \int_{\mathbb{R}^d} h \, \mathrm{d}\pi \right) \right\Vert_{L^2(\pi)} < \varepsilon,     
\end{equation}
for all $N > N_0$.
\end{thm}
\begin{proof}
See Appendix \ref{app:sec4 proofs}.
\end{proof}

\begin{remark} The statement of Theorem \ref{thm:consistency} requires $\lambda$ to tend to zero `more slowly than $N$ tends to infinity'. Indeed, this is a necessary requirement as the Stein-Poisson equation \eqref{eq:Stein Poisson} describing the $N\rightarrow \infty$ limit becomes ill posed for $\lambda = 0$, and is typical of regularisation strategies for inverse problems \citep{kirsch2021introduction}.
\end{remark}

\begin{remark}[Source conditions for $h$] 
\label{rem:source condition}
In order to estimate the posterior approximation error incurred by Stein transport, it would be highly desirable to obtain versions of \eqref{eq:Stein accuracy} controlling stronger norms than $\Vert \cdot \Vert_{L^2(\pi)}$, with more quantitative convergence rates in terms of $N$, $\varepsilon$ and $\lambda$. For this, additional assumptions on $h$ would be required, for instance source conditions \citep[Section 3.2]{engl1996regularization} of the form 
\begin{equation}
\label{eq:h alpha power}
h - \int_{\mathbb{R}^d} h \, \mathrm{d}\pi \in \mathcal{H}^\alpha_{\xi^{k,\nabla \log \pi}}, \qquad 0 < \alpha \le 1,
\end{equation}
requiring the centred negative log-likelihood to belong to a fractional power of the RKHS \citep[Definition 4.11]{muandet2017kernel} associated to the KSD-kernel $\xi^{k,\nabla \log \pi}$. The spaces $\mathcal{H}^\alpha_{\xi^{k,\nabla \log \pi}}$ interpolate between $L^2(\pi)$ and $\mathcal{H}_{\xi^{k, \nabla \log \pi}}$, and thus \eqref{eq:h alpha power} should be interpreted as an additional (abstract) smoothness assumption: The negative log-likelihood $h$ and $\xi^{k,\nabla \log \pi}$ are more aligned (or $h$ is more regular when regularity is measured in terms of $\xi_{k,\nabla \log \pi}$) if \eqref{eq:h alpha power} holds with greater powers of $\alpha$.  Fine analysis of \eqref{eq:h alpha power} could potentially inform the choice of $k$ and is left for future work; an application of \eqref{eq:h alpha power} will be presented below in Proposition \ref{prop:reg -> 0}.
\end{remark}

\subsection{Mean field dynamics}
\label{sec:mean field dynamics}

According to Theorem \ref{thm:consistency}, the driving vector field constructed from the formulation in Problem \ref{prob:KRR} satisfies the Stein equation \eqref{eq:Stein eq} in an approximate sense. In this section, we discuss the impact of regularisation ($\lambda > 0$) and finite-particle effects ($N < \infty$) on the dynamics of the particle system, and in particular on the posterior approximation at final time $t = 1$.

\subsubsection{Impact of the regularisation ($\lambda > 0$): weighted/birth-death dynamics}
\label{sec:weighted dynamics}
If $\phi$ solves the Stein-Poisson equation \eqref{eq:Stein Poisson}, it follows by integrating both sides that $\int \phi \, \mathrm{d}\pi = 0$. Therefore, the regularisation has the same effect as a modification of the negative log-likelihood: We can bring $\lambda \pi \phi$ to the right-hand side and redefine $\widetilde{h}_t := h - \lambda \phi_t$ to absorb the regularising term.
As a consequence, we expect that, asymptotically as $N \rightarrow \infty$, Stein transport solves a slightly different Bayesian inference problem, associated to $\widetilde{h}_t$. The following corollary to the proof of Theorem \ref{thm:consistency} shows that if $\lambda$ is small, then $\widetilde{h}_t$ is close to $h$ (and thus we expect to recover the original inference problem as $\lambda \rightarrow 0)$:
\begin{proposition}
\label{prop:reg -> 0}
Let Assumptions \ref{ass:basic} and \ref{ass:universality} be satisfied, and denote by $\phi^{(\lambda)} \in L^2(\pi)$ the solution to the Stein-Poisson equation \eqref{eq:Stein Poisson}, with regularisation parameter $\lambda > 0$. Then $\lambda \phi^{(\lambda)}$ converges to zero in $L^2(\pi)$ as $\lambda \rightarrow 0$. If, moreover, $h$ satisfies the source condition \eqref{eq:h alpha power} for some $\alpha \in (0,1]$, then there exists a constant $C_{\alpha} >0$ such that 
\begin{equation}
\label{eq:reg decay}
\Vert \lambda \phi^{(\lambda)} \Vert_{L^2(\pi)}  \le C_\alpha \lambda^\alpha, \qquad \lambda >0. 
\end{equation}
\end{proposition}
\begin{proof}
See Appendix \ref{app:sec4 proofs}. 
\end{proof}
The estimate \eqref{eq:reg decay} shows that it is desirable to align $h$ with $\xi^{k,\nabla \log \pi}$ in the sense of \eqref{eq:h alpha power}, as the correction $\lambda \phi^{(\lambda)}$ vanishes more quickly for larger values of $\alpha$ (cf. Remark \ref{rem:source condition}).
Proposition \ref{prop:reg -> 0} gives a theoretical handle on the impact of regularisation, but the  \emph{`modified likelihood interpretation'} of the regularising term $\lambda \pi \phi$ in \eqref{eq:Stein Poisson} also suggests a correction scheme to debias Stein transport: we can attach weights $(w_t^i)_{i=1}^N$ to the particles $(X_t^i)_{i=1}^N$ that absorb the error term $-\lambda \phi^i_t$. More specifically, the Stein-Poisson equation \eqref{eq:Stein Poisson} allows us to write 
\begin{equation}
\label{eq:mean field}
\underbrace{\nabla \cdot(\mathcal{T}_{k,\pi_t} \nabla \phi_t)}_{\mathrm{transport}} \underbrace{- \lambda \phi_t \pi_t}_{\mathrm{reweighting}} = - \pi_t \left( h - \int_{\mathbb{R}^d} h \, \mathrm{d}\pi_t \right) = \partial_t \pi_t,
\end{equation}
where the second equality follows if we assume that $(\pi_t)_{t \in [0,1]}$ satisfies the interpolation \eqref{eq:homotopy}, see \eqref{eq:tempering dynamics} in Appendix \ref{sec:misc}. It is instructive to perform integration by parts in $\nabla \cdot(\mathcal{T}_{k,\pi_t} \nabla \phi_t)$, rewriting \eqref{eq:mean field} as
\begin{equation}
\label{eq:Stein transport mean field}
\partial_t 
\pi_t = - \nabla \cdot \left( \pi_t \int_{\mathbb{R}^d} (k(\cdot,y) \nabla \log \pi_t(y) + \nabla_y k(x,y)) \, \phi_t(y)\pi_t(\mathrm{d}y)\right) - \lambda \phi_t \pi_t.
\end{equation}
The coupled system comprised of \eqref{eq:Stein Poisson} and \eqref{eq:Stein transport mean field} are the mean-field equations of (weighted) Stein transport, and, as expected, they are similar to the SVGD mean field limit \citep{lu2019scaling,duncan2019geometry}. Equations \eqref{eq:mean field} and \eqref{eq:Stein transport mean field} motivate the \emph{weighted interacting particle system}
\begin{subequations}
\label{eq:weighted dynamics}
\begin{align}
\frac{\mathrm{d}X^i_t}{\mathrm{d}t} & = v_t(X^i_t)   
\\
\label{eq:weighting}
\frac{\mathrm{d}w_t^i}{\mathrm{d}t} & = -\lambda w_t^i \phi_t^i, \qquad\qquad w_0^i = \tfrac{1}{N}, \qquad i = 1,\ldots, N, 
\end{align}
\end{subequations}
where $(\phi^i_t)_{i=1}^N$ and $v_t$ are determined from $(X_t^i)_{i=1}^N$ as in Proposition \ref{prop:KRR} (or, from the weighted generalisation of Problem \ref{prob:KRR} discussed in Remark \ref{rem:weighted KRR} in Appendix \ref{appsec:KRR_proof}). The combination of transporting and reweighting the particles is similar in spirit to the dynamical schemes proposed by \cite{lu2019accelerating,lu2023birth,yan2023learning,gladin2024interaction}. 

The (weighted) empirical measures corresponding to \eqref{eq:weighted dynamics},
\begin{equation}
\label{eq:weighted empirical measure}
\rho_t^{(N)}(f) := \sum_{i=1}^N w_t^i f(X_t^i), \qquad \qquad f \in C_b(\mathbb{R}^d),
\end{equation}
should recover the interpolation \eqref{eq:homotopy} as $N \rightarrow \infty$, for any value of $\lambda > 0$, as the mean field limit is formally given by \eqref{eq:mean field}, which is (by construction) satisfied by the interpolation \eqref{eq:homotopy}. 
However, a rigorous proof under realistic assumptions appears to be difficult, as the weights in \eqref{eq:weighting} might degenerate as $N \rightarrow \infty$ and therefore the weighted empirical measure \eqref{eq:weighted empirical measure} might not have a limit. Indeed from a practical perspective, resampling schemes are typically required to prevent weight degeneracy for dynamical schemes involving weights \citep{del2006sequential,chopin2020introduction}. It is an interesting direction to amend Algorithm \ref{alg:stein} on the basis of \eqref{eq:weighted dynamics} and to include appropriate reweighting; in this paper, we prefer to choose $\lambda$ small enough so that the weights in \eqref{eq:weighting} remain almost constant (as suggested by Proposition \ref{prop:reg -> 0}) retaining a purely transport-based scheme. Ignoring the weight update \eqref{eq:weighting} incurs a small asymptotic bias in the posterior approximation, but forcing equal weights is preferable in terms of effective sample size \citep[Section 8.6]{chopin2020introduction}. We also stress (again) that aligning $h$ and $\xi^{k,\nabla \log \pi}$ in the sense of \eqref{eq:h alpha power} is expected to be beneficial according to Proposition \ref{prop:reg -> 0}: Smaller values of $\lambda \phi_t^i$ in \eqref{eq:weighting} lead to smaller weight updates (incurring a smaller asymptotic bias or, for the weighted scheme, a larger effective sample size).

\subsubsection{Impact of the finite particle approximation ($N <\infty$): projection from the mean-field limit}

The following proposition gives some insight into the nature of the error induced by the finite-particle approximation. As in Theorem \ref{thm:connections}, we denote by $v^*_{N,\lambda}$ the solution to Problem \ref{prob:KRR}.
\begin{proposition}[Projection]
\label{prop:projection}
Assume that $h$ and $\pi$ are such that there exists a solution $v = -\mathcal{T}_{k,\pi} \nabla \phi \in \mathcal{H}_k^d$ to the Stein equation \eqref{eq:Stein eq}; that is, $\phi$ solves the Stein-Poisson equation \eqref{eq:Stein Poisson} for $\lambda = 0$. Then, for any selection of points ${\bm{X} = (X^1, \ldots, X^N) \in (\mathbb{R}^d)^N}$ such that the Gram matrix $\bm{\xi}^{k,\nabla \log \pi} \in \mathbb{R}^{N \times N}$ is invertible and ${\tfrac{1}{N}\sum_{i=1}^N h(X^i) = \int_{\mathbb{R}^d}h \, \mathrm{d}\pi}$, there exist an orthogonal projection $P_{\bm{X}}$ on $\mathcal{H}_k^d$ such that $v^*_{N,0} = P_{\bm{X}}v$. In particular, $\Vert v^*_{N,0}\Vert_{\mathcal{H}_k^d} \le \Vert v \Vert_{\mathcal{H}_k^d}$.  
\end{proposition}
\begin{proof}
See Appendix \ref{app:sec4 proofs}. The orthogonal projection $P_{\bm{X}}$ projects onto the subspace
\begin{equation}
\label{eq:proj subspace}
\mathrm{span} \left\{ \nabla_{X^i} k(\cdot,X^i) + k(\cdot,X^i) \nabla \log \pi(X^i): i=1,\ldots, N\right\} \subset \mathcal{H}_k^d,
\end{equation}
spanned by the `SVGD vector fields' at the particle positions.
\end{proof}

The conditions in Proposition \ref{prop:projection} are very restrictive (existence of a solution in the relatively small space $\mathcal{H}_k^d$, see Remark \ref{rem:well posedness}, and $\frac{1}{N}\sum_{i=1}^N h(X^i) = \int_{\mathbb{R}^d}h \, \mathrm{d}\pi$), but a more comprehensive and quantitative version could be obtained with more technical effort. The main message from Proposition \ref{prop:projection} is that the vector field \eqref{eq:v star} based on $N$ particles tends to underestimate $v^*_{\infty} = -\mathcal{T}_{k,\pi} \nabla \phi$ obtained from the mean-field formulation in Proposition \ref{prop:KRR mean field} (if this comparison is made in the $\mathcal{H}_k^d$-norm). Consequently, it is reasonable to expect that Algorithm \ref{alg:stein} tends to move the particles too slowly, and that the posterior spread is hence overestimated (as the prior is typically more spread out than the posterior). Moreover, Stein transport is expected to be more accurate in cases where the projection $P_{\bm{X}}$ is close to the identity operator, that is, when the linear span in \eqref{eq:proj subspace} nearly exhausts $\mathcal{H}_k^d$. Intuitively, this reasoning suggests that the particles $(X^i)_{i=1}^N$ should be spread out (`adjusted') as much as possible, and accordingly we propose a modification (`adjusted Stein transport') in Section \ref{sec:adjusted}. 

\begin{remark}[SVGD underestimates the posterior spread]
\label{rem:SVGD collapse}
In contrast to Stein transport, SVGD tends to return posterior approximations that are too peaked in comparison to the true posterior \citep{ba2021understanding}. This observation can be understood from the fact that SVGD by construction minimises the reverse $\mathrm{KL}$-divergence and therefore intrinsically suffers from mode collapse \citep{blei2017variational}. As an alternative explanation, note that SVGD is a `noiseless' version of the interacting diffusion \citep{gallego1812stochastic}
\begin{equation*}
\mathrm{d}X_t^i = \tfrac{1}{N} \sum_{j=1}^N \left( -k(X_t^i,X_t^j ) \nabla V(X_t^j) + \nabla_{X_t^j} k(X_t^i,X_t^j)\right) \mathrm{d}t + \sum_{j=1}^N \sqrt{\tfrac{2}{N} \mathcal{K}(X_t^1,\ldots,X_t^N)} \, \mathrm{d}W_t^j, 
\end{equation*}
which is ergodic with respect to the product measure $\pi^{\otimes N}$, hence in principle exact for a finite number of particles. While the noise term disappears in the limit as $N \rightarrow \infty$ \citep[Proposition 2]{duncan2019geometry}, \citep{nusken2021stein}, neglecting it clearly leads to an underspread finite-particle posterior approximation. A numerical investigation of those aspects can be found in Section \ref{sec:numerics}.
\end{remark}

\section{On the geometry of Stein transport}
\label{sec:geometry}

In this section we discuss the relationship of Stein transport and SVGD to optimal transport, explaining some of the similarities between \eqref{eq:Stein ot} and \eqref{eq:svgd} from a conceptual angle. Investigations into the Stein geometry were started by \citet{liu2017stein} and have further been developed by \citet{duncan2019geometry,nusken2021stein}.

The key object in the Stein geometry is the extended\footnote{An extended metric satisfies the usual axioms of a metric, but may take the value infinity.} metric $d_k$ between probability measures $\mu,\nu \in \mathcal{P}(\mathbb{R}^d)$, 
\begin{equation}
\label{eq:dk}
d^2_k(\mu,\nu) = \inf_{(\pi_t,v_t)_{t \in [0,1]}} \left\{ \int_0^1 \Vert v_t \Vert^2_{\mathcal{H}_k^d} \, \mathrm{d}t: \quad \partial_t \pi_t + \nabla \cdot (\pi_t v_t) = 0, \qquad \pi_0 = \mu, \,\, \pi_1 = \nu \right\}.
\end{equation}

Intuitively speaking,
the continuity equation $\partial_t \pi_t + \nabla \cdot (\pi_t v_t) = 0$ asserts that the probability measure $\pi_t$ is transported by the vector fields $v_t$, in the sense of the ordinary differential equation (ODE) $\mathrm{d}X_t = v_t(X_t) \, \mathrm{d}t$, with random initial condition $X_0 \sim \pi_0$, see \citet[Section 4.1.2]{santambrogio2015optimal}. Therefore, $d_k$ measures the distance between $\mu$ and $\nu$ in terms of the length of the shortest connecting curve (geodesic) of probability measures (hence $d_k$ may be thought of as a geodesic distance in a Riemannian manifold), when the (infinitesimal) length of curves is measured in terms of the RKHS-norms of their driving vector fields. A key motivation for considering \eqref{eq:dk} is that replacing $\Vert \cdot \Vert_{\mathcal{H}_k^d}$ by $\Vert \cdot \Vert_{L^2(\pi_t)}$ in \eqref{eq:dk} recovers the Benamou-Brenier representation of the quadratic Wasserstein distance $d_{W_2}$ \citep{benamou2000computational}; a detailed comparison between $d_k$ and $d_{W_2}$ can be found in \citet[Appendix A]{duncan2019geometry}.

According to \citet[Theorem 3.5]{liu2017stein}, SVGD performs gradient flow dynamics of the Kullback-Leibler divergence (KL) with respect to $d_k$,
\begin{equation}
\label{eq:gradient flow}
\partial_t \pi_t = - \mathrm{grad}_k \mathrm{KL}(\pi_t|\pi_1),
\end{equation}
recalling that $\pi_1$ represents the target posterior in our notation. The gradient operation $\mathrm{grad}_k$ is induced by the geometry encoded by \eqref{eq:dk}, the main point being that the abstract  evolution equation \eqref{eq:gradient flow} governs the mean-field limit of \eqref{eq:svgd} under this interpretation \citep{liu2017stein,duncan2019geometry}.
In order to obtain a similar geometric  characterisation of Stein transport (that is, of the coupled system composed of \eqref{eq:Stein Poisson} and \eqref{eq:Stein transport mean field}), we notice that \eqref{eq:dk} naturally induces a notion of \emph{Stein optimal transport maps} (namely those that arise from the shortest connecting curve): 
\begin{defn}[Stein optimal transport maps]
\label{def:Stein optimal transport}
Let $\mu,\nu \in \mathcal{P}(\mathbb{R}^d)$ be two probability measures such that $d_k(\mu,\nu) < \infty$. Assume that $(\pi_t,v_t)_{t \in [0,1]} \in C^1([0,1]\times \mathbb{R}^d;\mathbb{R}) \times C_b^1([0,1]\times \mathbb{R}^d; \mathbb{R}^d)$ is a geodesic, that is, a connecting curve that realises the infimum in \eqref{eq:dk}: 
\begin{equation}
\label{eq:Stein optimality}
\partial_t \pi_t + \nabla \cdot(\pi_t v_t) = 0, \quad \pi_0 = \mu, \,\,\pi_1 = \nu, \quad \text{and} \quad \int_0^1\Vert v_t \Vert^2_{\mathcal{H}_k^d} \, \mathrm{d}t  = d^2_k(\mu,\nu).
\end{equation}
The corresponding ODE $\frac{\mathrm{d}X_t}{\mathrm{d}t} = v_t(X_t)$ induces a time-one flow map\footnote{The time-one flow map associated to an ODE maps initial conditions to the corresponding solution at time $t=1$, that is, $F(x_0) = X_1$ if $X_t$ solves the ODE $\frac{\mathrm{d}X_t}{\mathrm{d}t} = v_t(X_t)$ with initial condition $X_0 = x_0$.} $F:\mathbb{R}^d \rightarrow \mathbb{R}^d$ which we then call a \emph{Stein optimal transport map} between $\mu$ and $\nu$.
\end{defn}

\begin{remark}[Wasserstein geodesics]
As alluded to above, replacing $\Vert \cdot \Vert_{\mathcal{H}_k^d}$ by $\Vert \cdot \Vert_{L^2(\pi_t)}$ in \eqref{eq:dk} recovers the Wasserstein\mbox{-2} distance via the Benamou-Brenier formula \citep{benamou2000computational}. If Definition \ref{def:Stein optimal transport} is modified accordingly (replacing Stein geodesics by Wasserstein geodesics), then the induced optimal transport maps coincide with the convential ones for the quadratic cost \citep[Section 8.1]{villani2003topics}.     
\end{remark}
By definition, Stein optimal transport maps in the sense of Definition \ref{def:Stein optimal transport} satisfy $F_{\#} \mu = \nu$, where $F_{\#}$ denotes the pushforward of measures. As pointed out by \citet{el2012bayesian}, constructing a transport map $F$ connecting prior and posterior (that is, $F_{\#} \pi_0 = \pi_1$) would solve the Bayesian inference problem, as then samples from the prior $\pi_0$ can be transformed into samples from the posterior $\pi_1$ (in formulas: $X_0 \sim \pi_0 \implies F(X_0) \sim \pi_1$). However, obtaining geodesics for \eqref{eq:dk} is computationally demanding; indeed the first order minimality conditions are given by the following (numerically challenging) system of coupled partial differential equations \citep[Proposition  18]{duncan2019geometry},
\begin{subequations}
\label{eq:geodesics}
\begin{align}
\label{eq:geodesic1}
    \partial_t \pi_t + \nabla \cdot (\pi_t \mathcal{T}_{k,\pi_t}\nabla \phi_t) & = 0,
    \\
\label{eq:geodesic2}
    \partial_t \phi_t + \nabla \phi_t \cdot \mathcal{T}_{k,\pi_t} \nabla \phi_t & = 0,
    \end{align}
\end{subequations}
recalling the convolution-type integral operator $\mathcal{T}_{k,\pi}$ defined in \eqref{eq:def Tk}.

As we show in the remainder of this section, the problem becomes tractable when $\mu$ and $\nu$ are close, and thus the transport is infinitesimal in nature. In this regime, it will turn out that it is sufficient to solve \eqref{eq:geodesic1}, and that \eqref{eq:geodesic2} can be dispensed with. Stein transport approximately solves \eqref{eq:geodesic1}; we summarise this finding as follows:
\begin{result}[Optimal transport interpretation]
After partitioning the time interval $[0,1]$ using the discretisation $0 = t_0 < t_1 < \ldots < t_{N_\mathrm{steps}} = 1$, Stein-transport implements infinitesimally optimal Stein transport maps (or geodesic flow) between $\pi_{t_i}$ and $\pi_{t_{i+1}}$, successively for $i = 0,\ldots,N_{\mathrm{steps}}-1$. Here, the marginals $\pi_{t_i}$ are fixed by the interpolation \eqref{eq:homotopy}. 
\end{result}

Before making this statement  precise in Theorem \ref{thm:infitesimal Stein transport} below, we present an informal derivation of the update equations \eqref{eq:Stein ot}, starting from the geodesic equations \eqref{eq:geodesics}:

\paragraph{Alternative derivation of Stein-transport.} For the Stein optimal transport between $\pi_{t_i}$ and $\pi_{t_{i+1}}$, we may assume that $\phi_t$ is constant on the small time interval $[t_i,t_{i+1}]$. We therefore concentrate on \eqref{eq:geodesic1} rather than \eqref{eq:geodesic2} (which is an update equation for $\phi_t$) and impose the evolution equation for $\pi_t$ as dictated by the interpolation \eqref{eq:homotopy}. The time derivative satisfies
\begin{equation}
\label{eq:pi derivative}
\partial_t \pi_t = -\pi_t \left(h - \int_{\mathbb{R}^d} h \, \mathrm{d}\pi_t \right), 
\end{equation}     
see the proof of Proposition \ref{prop:Stein eq} in Appendix \ref{appsec:proofs}. Equating \eqref{eq:geodesic1} and \eqref{eq:pi derivative}, we obtain the Stein-Poisson equation \eqref{eq:Stein Poisson}, up to the regularising term $\lambda 
\phi \pi$. To arrive at \eqref{eq:Stein ot}, notice that we can write 
\begin{equation*}
-\nabla \cdot (\pi_t \mathcal{T}_{k,\pi_t}\nabla \phi_t) = \int_{\mathbb{R}^d} \xi^{k,\nabla \log \pi_t} (\cdot,y) \phi(y) \pi(\mathrm{d}y), 
\end{equation*}
performing integration by parts and using the definition \eqref{eq:xi} of the KSD-kernel $\xi^{k,\nabla \log \pi}$. Formally replacing $\pi$ by the empirical measure $\frac{1}{N}\sum_{i=1}^N \delta_{X_t^i}$, and adding the regularisation term involving $\lambda$ leads to \eqref{eq:Stein ot}, cf. Remark \ref{rem:SP emp measures}.

To make this informal discussion precise we present Proposition \ref{thm:infitesimal Stein transport} below, which establishes that Stein optimal transport is indeed governed by \eqref{eq:geodesic1} when $\mu$ and $\nu$ are close. The `closeness' between probability distributions is captured rigorously by considering (short) curves and differentiating at zero (corresponding to linearisation or first-order Taylor expansions).

\begin{proposition}
[Infinitesimal Stein transport and the Stein-Poisson equation]
\label{thm:infitesimal Stein transport}
Let $k:\mathbb{R}^d \times \mathbb{R}^d \rightarrow \mathbb{R}$ be bounded, with bounded first order derivatives.
Fix $\varepsilon >0 $ and let $(\pi_t)_{t \in (-\varepsilon, \varepsilon)} \subset \mathcal{P}(\mathbb{R}^d)$ be a curve of probability measures such that $t \mapsto d_k(\pi_0,\pi_t)$ is Lipschitz-continuous. Assume that the Stein optimal transport maps between $\pi_0$ and $\pi_t$ exist and are unique, for all $t \in (-\varepsilon,\varepsilon)$, and denote those maps by $F_t$. Assume moreover that there exist solutions $\phi_t \in L^2(\pi_t)$ to the Stein-Poisson equations
\begin{equation}
\label{eq:SP transport}
\nabla \cdot (\pi_t \mathcal{T}_{k,\pi_t}\nabla \phi_t) = - \partial_t \pi_t,
\end{equation}
where the time-derivatives $\partial_t \pi_t$ are assumed to exist in the distributional sense, and that $(-\varepsilon,\varepsilon) \ni t \mapsto \mathcal{T}_{k,\pi_t} \nabla \phi_t \in \mathcal{H}_k^d$ is continuous.

Then we have
\begin{equation*}
\frac{\mathrm{d}^+}{\mathrm{d}t} F_t(x) := \lim_{t \rightarrow 0^+} \frac{F_t(x) - x}{t} = \mathcal{T}_{k,\pi_0} \nabla \phi_0.
\end{equation*}
\end{proposition}
\begin{proof}
See Appendix \ref{app:transport}. The proof is inspired by the proof of Proposition 8.4.6 in \citet{ambrosio2008gradient}, which is an analogous result for the $W_p$-Wasserstein distance.
\end{proof}

\section{Other perspectives and related work}
\label{sec:works}

\subsection{Regression and transport}

\label{sec:KMED}

Score and flow matching \citep{yang2023diffusion,song2021score,lipman2023flow,liu2023flow} have recently led to breakthroughs in generative modeling, and, based on conditional expectations, both approaches can be framed as regression type approximations of vector fields in dynamical models, not unlike \eqref{eq:KRR}. Stochastic optimal control is also related to least-squares formulations via backward stochastic differential equations (BSDEs), see, for example, \cite[Chapter 6] {pham2009continuous}, \cite{chessari2023numerical} or \cite[Section 2]{richter2023continuous}.  

Interacting particle systems similar to Stein transport have been constructed in recent works \citep{maurais2023adaptive,maurais2024sampling,wang2024measure}, named \emph{kernelised Fisher-Rao (KFR) flow} or \emph{kernel mean embedding (KME) dynamics}. From a regression perspective, those can be understood in terms of the minimisation problem \citep[Section 4]{wang2024measure}
\begin{subequations}
\begin{align}
\label{eq:Stein embedded}
v^* &\in \argmin_{v \in (L^2(\pi))^d} \left(\Vert \mathcal{T}_{k,\pi} (S_\pi v - h_0) \Vert^2_{\mathcal{H}_k} + \lambda \Vert v \Vert^2_{(L^2(\pi))^d} \right)
\\
\label{eq:MMD matching}
& = \argmin_{v \in (L^2(\pi))^d} \left( \mathrm{MMD}_{k}(-\nabla \cdot (\pi v),-\pi h_0) + \lambda \Vert v \Vert^2_{(L^2(\pi))^d} \right), 
\end{align}
\end{subequations}
where $\mathrm{MMD}_{k}(-\nabla \cdot (\pi v),-\pi h_0)$ refers to the maximum mean discrepancy \citep{smola2007hilbert,gretton2012kernel,muandet2017kernel} between the signed measures $-\nabla \cdot(\pi v)$, induced by the flow of the ODE $\mathrm{d}X_t = v(X_t)\,\mathrm{d}t$, and $-\pi h_0 = \partial_t \pi_t$ given by the interpolation \eqref{eq:homotopy}. The construction in \eqref{eq:MMD matching} aims to match flows based on their kernel mean embeddings, while \eqref{eq:Stein embedded} is instructive in light of its comparison to \eqref{eq:KRR mean field}: the Stein equation \eqref{eq:Stein eq} is embedded into $\mathcal{H}_k$ via $\mathcal{T}_{k,\pi}$ and the roles of $\Vert \cdot \Vert_{\mathcal{H}_k}$ and $\Vert \cdot \Vert_{L^2{\pi}}$ have been reversed. In practical terms, KFR/KME-flow replaces the vector fields $\nabla_{X^i} k(\cdot,X^i) + k(\cdot,X^i) \nabla \log \pi_t(X^i)$ by $\nabla_{X^i} k(\cdot,X^i)$: it does not leverage or require the scores $\nabla \log \pi_t$, and is therefore applicable in scenarios where the prior $\pi_0$ is only available through a sample-based approximation (as is typical in data assimilation \citep{law2015data}, for instance). On the other hand, in situations where the score can be evaluated, it is expected that methods incorporating this information into the inference procedure perform better in terms of accuracy and scalability.   

SVGD has been connected to problems of regression type through the formulation
\begin{equation}
\label{eq:reg Wasserstein}
    v^* \in \argmin_{v \in \mathcal{H}^d_k} \left( \int_{\mathbb{R}^d} \left( v - \nabla \log \left( \frac{\mathrm{d}\pi}{\mathrm{d}\pi_1}\right)  \right)^2 \mathrm{d}\pi + \lambda \Vert v \Vert_{\mathcal{H}^d_k}^2 \right),
\end{equation}
where the optimal velocity field $v^* = (\mathcal{T}_{k,\pi} + \lambda I)^{-1} \mathcal{T}_{k,\pi} \left( \nabla \log  \frac{\mathrm{d}\pi}{\mathrm{d}\pi_1} \right)$ is a close relative of the SVGD velocity field $v_{\mathrm{SVGD}} = \mathcal{T}_{k,\pi} \left( \nabla \log  \frac{\mathrm{d}\pi}{\mathrm{d}\pi_1} \right)$ and recovers the Wasserstein-2 gradient $\nabla \log \frac{\mathrm{d}\pi}{\mathrm{d}\pi_1}$ in the limit as $\lambda \rightarrow 0$
\citep{maoutsa2020interacting, he2022regularized, zhu2024approximation}. From an algorithmic perspective, interacting particle systems based on \eqref{eq:reg Wasserstein} have the same fixed points as standard SVGD, and are  therefore unlikely to overcome the finite-particle issues associated to SVGD (see Section \ref{sec:collapse} and Remark \ref{rem:SVGD collapse}).  

\subsection{Gradient flows}
\label{sec:gradient flows}

In Section \ref{sec:geometry}, we have shown that Stein transport performs infinitesimally optimal Stein transport maps between $\pi_t$ and $\pi_{t + \Delta t}$: this is a statement at the particle level, or, in other words, about couplings between $\pi_t$ and $\pi_{t + \Delta t}$. Moreover, the construction does not depend on the specific form of the interpolation \eqref{eq:homotopy}, as explained in Remark \ref{rem:other interpolations}. However, at the level of densities, the interpolation \eqref{eq:homotopy} can be understood from a geometrical perspective as follows: Define the time-reparameterised curve $\rho_t := \pi_{-\log(1-t)}$ as in \cite{domingo2023explicit,chopin2023connection}, satisfying $\rho_0 = \pi_0$ and $\lim_{t \rightarrow \infty} \rho_t = \pi_1$. A direct calculation shows that 
\begin{equation}
\label{eq:FR flow}
\partial_t \rho_t = - \rho_t \left( \log \left( \frac{\rho_t}{\pi_1}\right) - \int_{\mathbb{R}^d} \log \left( \frac{\rho_t}{\pi_1}\right)\mathrm{d}\rho_t \right),
\end{equation}
and this equation can be understood as a $\mathrm{KL}$-gradient flow in the Fisher-Rao geometry \citep{lu2019accelerating,lu2023birth,zhu2024approximation,chen2023gradient,chen2023sampling}, or as a mirror gradient flow \citep{domingo2023explicit,chopin2023connection}. However, our impression is that the more salient point is that \eqref{eq:FR flow} can be understood as a natural or Newton-type gradient flow \cite[Chapter 12]{amari2016information},
\begin{equation}
\label{eq:Newton flow}
\partial_t \rho_t = (\mathrm{Hess} \, \mathrm{KL}(\rho_t|\pi_1))^{-1} \nabla \mathrm{KL}(\rho_t|\pi_1),
\end{equation}
where the Hessian and the gradient are interpreted in the `Euclidean' way, that is, in the geometry generated by `vertical' geodesics of the form $\rho_t = (1-t) \rho_0 + t \rho_1$, for $t \in [0,1]$. The formulation \eqref{eq:Newton flow} is essentially a corollary of the mirror flow perspective (see, for example, \cite[Appendix B]{kerimkulov2023fisher} or \cite[Theorem 1]{raskutti2015information}), but for the convenience of the reader, we provide a self-contained explanation in Appendix \ref{app:natural}. 

Since the Hessian in \eqref{eq:Newton flow} acts as a preconditioner, it is reasonable to expect that methods based on the interpolation \eqref{eq:homotopy} do not suffer significantly from ill-conditioned targets (featuring, for example, elongated or distorted modes with an intricate geometry), and indeed convergence in \eqref{eq:Newton flow} is independent of the log-Sobolev constant of the target $\pi_1$ \citep{lu2019accelerating,lu2023birth,domingo2023explicit,chen2023gradient}. To corroborate this intuition, we numerically compare Stein transport to SVGD in Section \ref{sec:joker} on a target that exhibits the above mentioned characteristics.    
Finally, it is interesting to remark that Newton-type gradient flows of the form \eqref{eq:Newton flow} have been constructed with $\mathrm{Hess}\, \mathrm{KL}(\rho_t|\pi_1)$  and $\nabla \mathrm{KL}(\rho_t|\pi_1)$ derived instead from the Wasserstein geometry \citep{wang2020information} or the Stein geometry \citep{detommaso2018stein}, but those schemes, in contrast to Stein transport, require access to $\mathrm{Hess} \log \pi_1$, the Hessian of the log-target.
\vspace{-0.5cm}   

\section{Implementation and numerical experiments}
\label{sec:numerics}
\vspace{-0.1cm}
\subsection{Adjusted Stein transport}
\label{sec:adjusted}

In this section, we detail a modification to the implementation of Algorithm \ref{alg:stein} that greatly stabilises the method and substantially improves numerical performance. The accuracy of a single Stein transport step (one Euler step on the continuous-time system \eqref{eq:Stein ot}) depends crucially on the accuracy of the approximation $\pi_{t_i} \approx \frac{1}{N} \sum_{j=1}^N \delta_{X^j_{t_i}}$. In the context of Proposition \ref{prop:projection}, this translates into properties of the projection $P_{\bm{X}}$ that should be as close to the identity on $\mathcal{H}_k^d$ as possible.
Additionally, if two particles are very close to each other, $X^j_{t_i} \approx X^l_{t_i}$ for some $j \neq l$, then inversion of the linear system in \eqref{eq:Stein ot} becomes highly unstable, as the Gram matrix $\bm{\xi}^{k, \nabla \log \pi_{t_i}}$ associated to the KSD-kernel $\xi^{k,\nabla \log \pi_{t_i}}$ becomes nearly degenerate. Errors of this type tend to accumulate as the algorithm progresses.  
To mitigate these problems, we find it beneficial to apply a few SVGD steps targeting $\pi_{t_i}$ before carrying out the Stein transport $\pi_{t_i} \mapsto \pi_{t_{i+1}}$. This SVGD-adjustment improves the approximation $\pi_{t_i} \approx \frac{1}{N} \sum_{j=1}^N \delta_{X^j_{t_i}}$, and the repulsive term $\nabla_{X_t^j} k(X_t^i,X_t^j)$ in \eqref{eq:svgd} separates the particles, thereby improving the conditioning of the linear system in \eqref{eq:Stein ot}. We would like to stress that in the suggested algorithm, SVGD is not used to target the actual posterior $\pi_1$; the method is rather very much akin to the use of MCMC samplers within a sequential Monte Carlo framework \citep[Chapter 17]{chopin2020introduction}. For clarity, we summarise SVGD-adjusted Stein transport in Algorithm \ref{alg:adjusted}.  Notice that the incorporation of SVGD steps is fairly adhoc, and a more systematic study of adjustment schemes (possibly relating to the optimal design of experiments \citep{huan2024optimal}) could be an interesting direction for future work.
\begin{algorithm} \caption{Adjusted Stein transport\label{alg:adjusted}}
     \begin{algorithmic}
\Require{Prior samples $X_0^1,\ldots,X_0^N \in \mathbb{R}^d$, prior score function $\nabla \log \pi_0: \mathbb{R}^d \rightarrow \mathbb{R}^d$, negative log-likelihood ${h:\mathbb{R}^d \rightarrow \mathbb{R}}$, with gradient $\nabla h: \mathbb{R}^d \rightarrow \mathbb{R}^d$, positive definite kernel $k$, regularisation parameter $\lambda > 0$, time discretisation $0 = t_0 < t_1 < \ldots < t_{N_\mathrm{steps}} = 1$ with time step $\Delta t$, number $N_{\mathrm{adjust}}$ of interspersed SVGD steps per transport step, SVGD time step $\Delta t_{\mathrm{adjust}}$}
    \vspace{0.2cm}
    \For{$n= 0, \ldots,  N_{\mathrm{steps}} - 1$}
    \vspace{0.1cm}
    \For{$n= 0, \ldots,  N_{\mathrm{SVGD}} - 1$}
    \tikzmark{top1}
    \vspace{0.1cm}
    \State
    $X^i_{n} \gets X^i_{n} +  \frac{\Delta t_{\mathrm{adjust}} }{N} \sum_{j=1}^N \left(k(X_n^i,X_n^j) \nabla \log \pi_{t_n} (X_n^j) + \nabla_{X_n^j} k(X_n^i,X_n^j)\right).$ \tikzmark{right1}
    \EndFor 
    \algcustomendfor
    \tikzmark{bottom1}
      \State Compute the scores  $P_n^j = -t_n \nabla h (X_n^j) + \nabla \log \pi_0(X_n^j).$
      \tikzmark{top}
      \vspace{0.1cm}
      \State Compute
      \vspace{-0.2cm}
      $$
      \begin{array}{rl}
      \bm{\xi}_{ij}  = & P_n^i \cdot \nabla_{X_n^j} k(X_n^i,X_n^j) + P_n^j \cdot \nabla_{X_n^i} k(X_n^i,X_n^j)
      \\
      & + \nabla_{X_n^i} \cdot \nabla_{X_n^j} k(X_n^i, X_n^j) + P_n^i \cdot k(X_n^i,X_n^j) P_n^j.\tikzmark{right}
      \end{array}$$
      \State Compute $ \bm{h}_i = h(X^i_n) - \frac{1}{N} \sum_{i=1}^N h(X_n^i)$.
      \vspace{0.1cm}
      \State Solve $(\bm{\xi} + \lambda I_{N \times N}) \phi = \bm{h}$ for $\phi \in\mathbb{R}^N$.
      \vspace{0.1cm}
      \State $X^i_{n+1} \gets X^i_{n} +  \frac{\Delta t }{N} \sum_{j=1}^N \phi^j \left(k(X_n^i,X_n^j) P_n^j (X_n^j) + \nabla_{X_n^j} k(X_n^i,X_n^j)\right).$ \tikzmark{bottom}
    \EndFor
    \algcustomendfor
    \vspace{0.1cm}
    \State 
    \Return{posterior samples $X_{N_{\mathrm{steps}}}^1,\ldots X_{N_\mathrm{steps}}^N$.}
  \end{algorithmic}
  \AddNote{top}{bottom}{right}{Stein transport step from Algorithm \ref{alg:stein}.}
\AddNote{top1}{bottom1}{right1}{SVGD-adjustment targeting $\pi_{t_n}$.}
\end{algorithm}

\clearpage

\subsection{Numerical experiments}
\label{sec:num experiments}

In the following, we compare Stein transport (Algorithm \ref{alg:stein}), its SVGD-adjusted variant (Algorithm \ref{alg:adjusted}) and SVGD in a number of test cases. For all methods, we use the square-exponential kernel 
\begin{equation}
\label{eq:sq exp kernel}
k(x,y) = \exp \left( -\frac{\Vert x - y \Vert^2}{2 \sigma^2} \right),
\end{equation}
where the bandwidth $\sigma^2$ is chosen adaptively according to the median heuristic \citep{liu2016stein}, ${\sigma^2 = \mathrm{med}^2/(2 \log N)}$, with $\mathrm{med}$ being the median of the pairwise Euclidean distance between the current particle positions. SVGD is implemented using the Adagrad optimiser as suggested by \citet{liu2016stein}.

\subsubsection{Implicit preconditioning: the Joker distribution}
\label{sec:joker}

We follow \citet[Section 5.1]{detommaso2018stein} and consider sampling from a two-dimensional Bayesian posterior, derived from the forward operator
\begin{equation*}
\mathcal{F}(x) = \log \left( (1-x_1)^2 + 100 (x_2 - x_1^2)^2\right), \qquad (x_1,x_2) = x,
\end{equation*}
a scalar logarithmic Rosenbrock function.
We impose the prior $x \sim \mathcal{N}(0,I_{2 \times 2})$, and collect a single observation $y_{\mathrm{obs}} = \mathcal{F}(x_{\mathrm{true}}) + \xi$, with $\xi \sim \mathcal{N}(0,\sigma^2)$ and $x_{\mathrm{true}}$ drawn from the prior, inducing the (unnormalised) likelihood $\exp(-\tfrac{1}{2 \sigma^2} \Vert\mathcal{F}(x) - y_{\mathrm{obs}}\Vert^2)$. In other words, the posterior is given by 
\begin{equation*}
\pi(x) \propto \exp\left(-\tfrac{1}{2 \sigma^2} \Vert \mathcal{F}(x) - y_{\mathrm{obs}}\Vert^2 - \tfrac{1}{2} \Vert x\Vert^2\right),
\end{equation*}
and we plot its density (for $\sigma = 0.3$) in Figure \ref{fig:Joker_target}. We observe that the contours of the target are fairly sharp and narrow, and indeed \citet{detommaso2018stein} use this example to exhibit the benefits of appropriately preconditioned versions of SVGD (cf. the discussion in Section \ref{sec:gradient flows}). 
\begin{figure}[htb]
  \centering
  \begin{subfigure}[b]{0.48\textwidth}
    \centering
    \includegraphics[width=\textwidth]{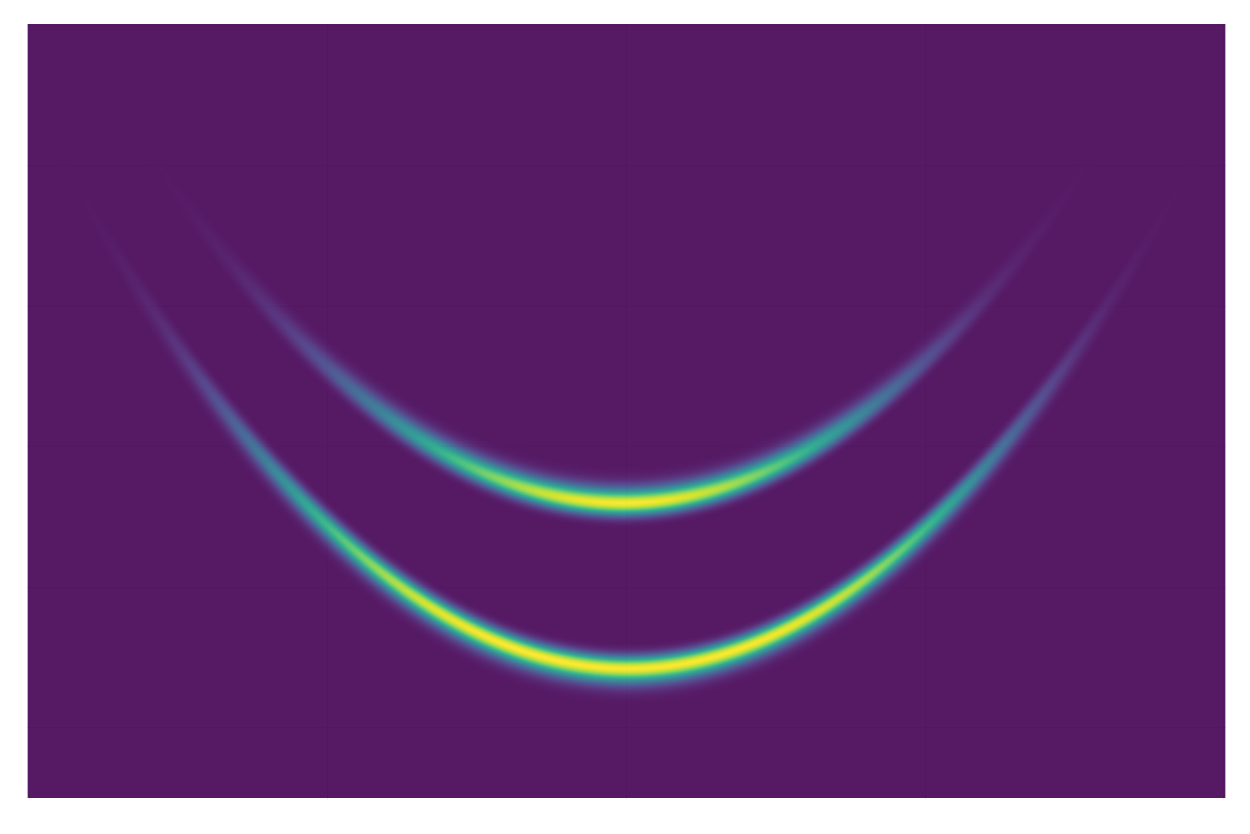}
    \caption{Target.}
    \label{fig:Joker_target}
  \end{subfigure}
  \hfill
  \begin{subfigure}[b]{0.48\textwidth}
    \centering
    \includegraphics[width=\textwidth]{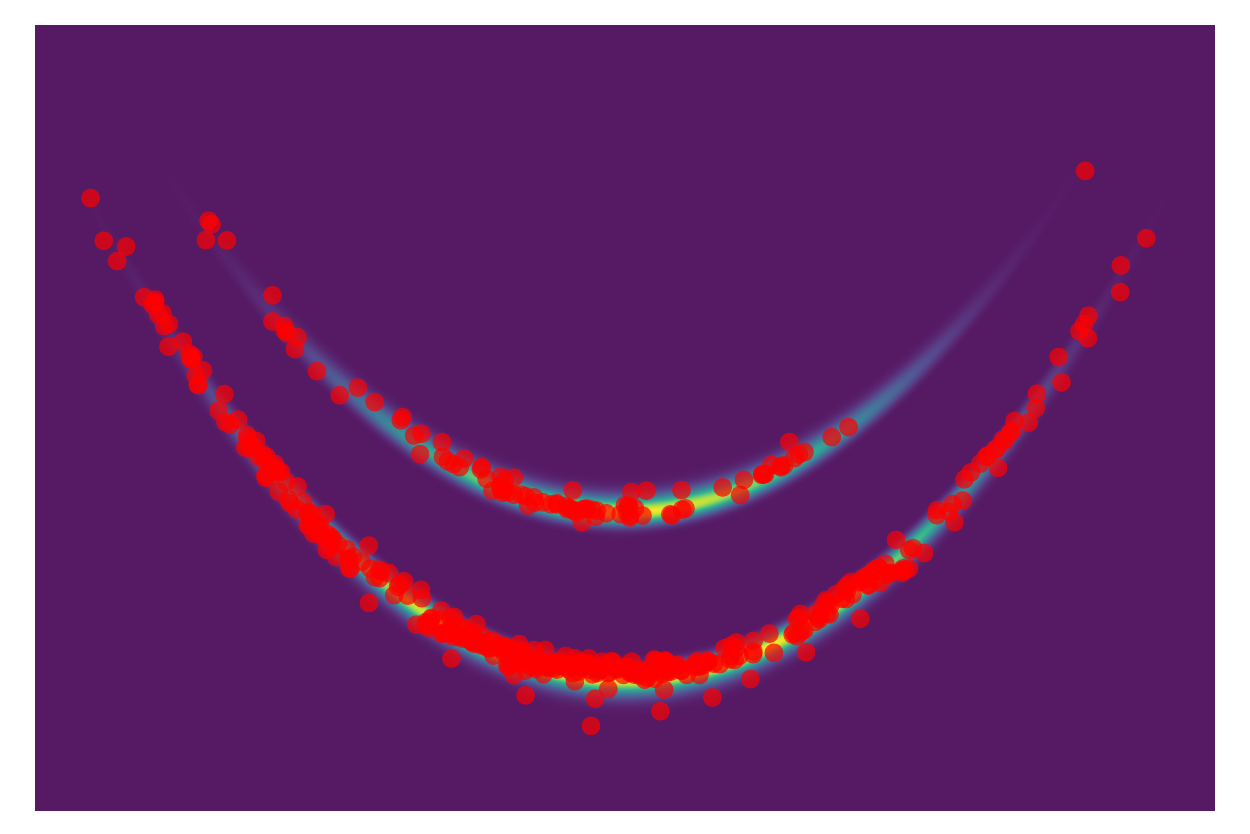}
    \caption{Stein transport.}
    \label{fig:Joker_transport}
  \end{subfigure}
  \vspace{0.2cm}
  \begin{subfigure}[b]{0.48\textwidth}
    \centering
    \includegraphics[width=\textwidth]{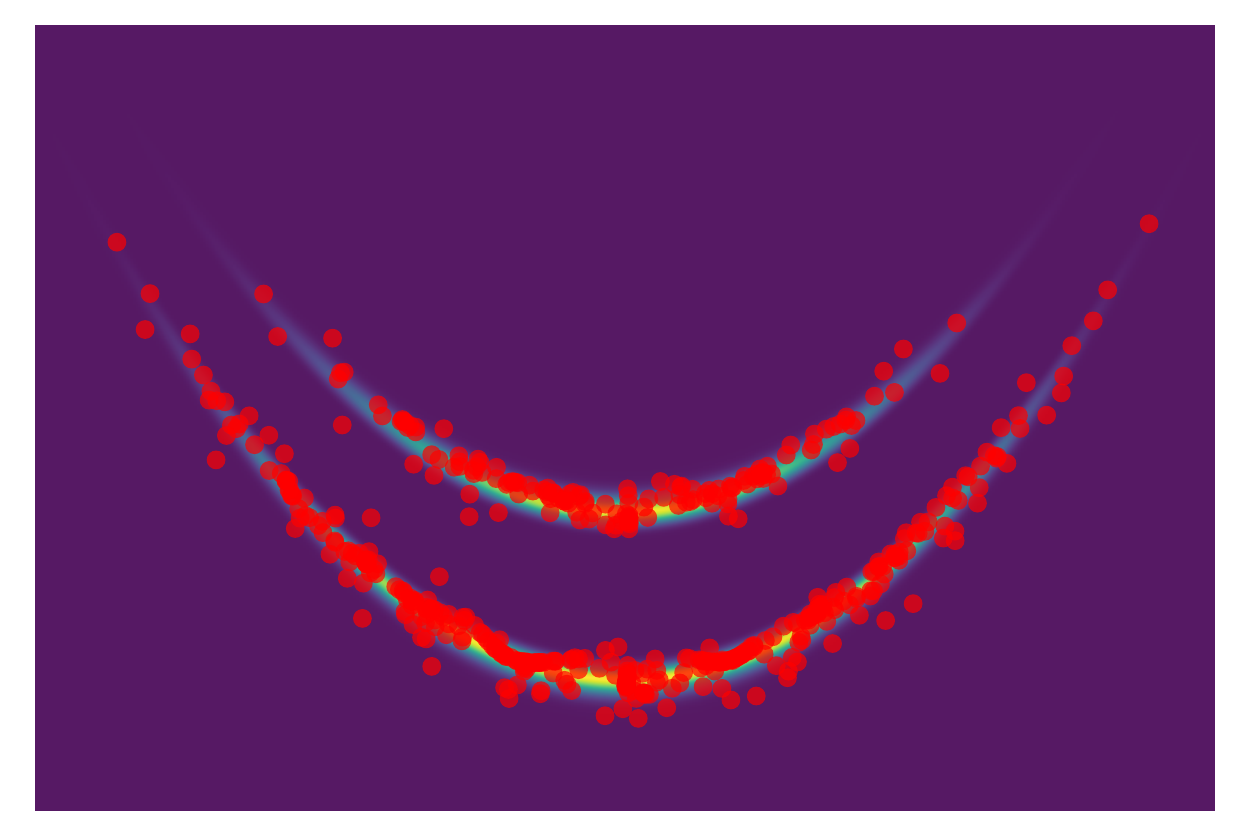}
    \caption{SVGD.}
    \label{fig:Joker_svgd}
  \end{subfigure}
  \hfill
  \begin{subfigure}[b]{0.48\textwidth}
    \centering
    \includegraphics[width=\textwidth]{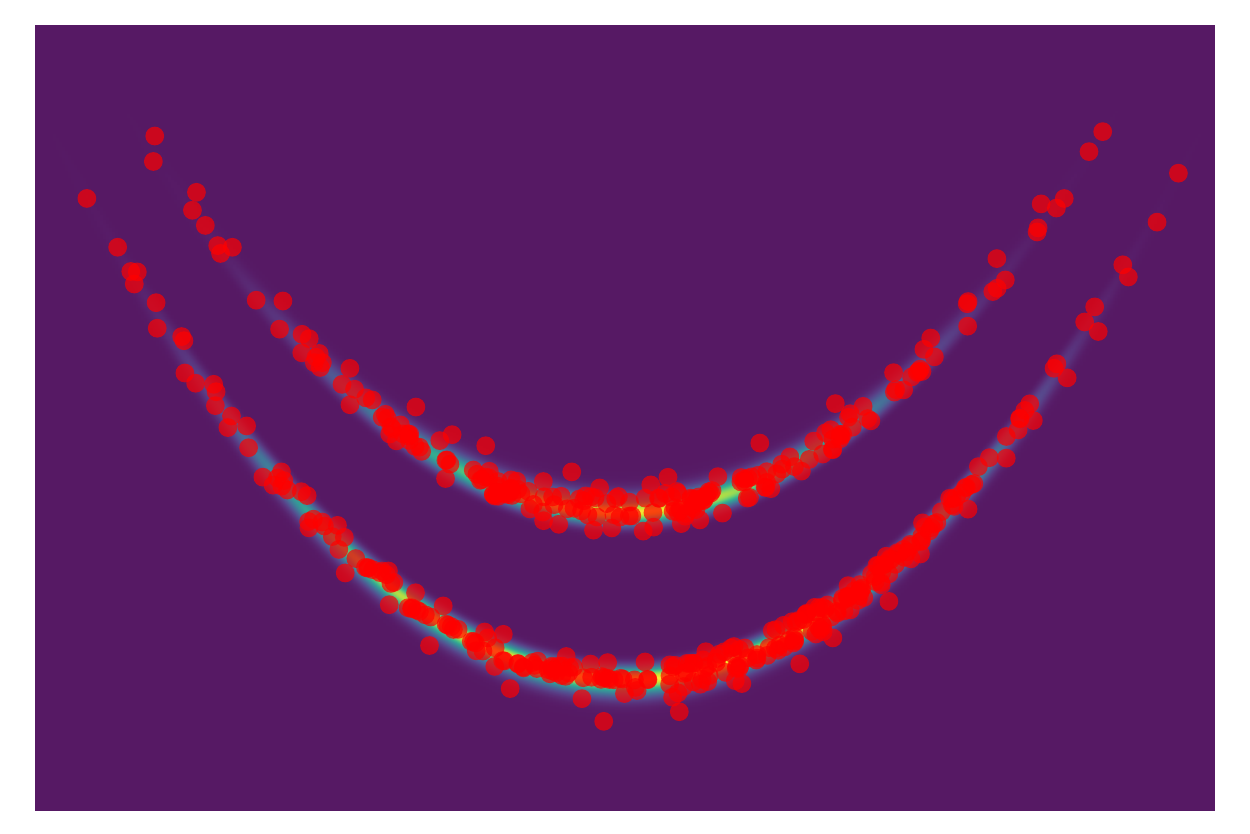}
    \caption{Adjusted Stein transport.}
    \label{fig:Joker_mix}
  \end{subfigure}
  \caption{Posterior approximations for the Joker distribution.}
  \label{fig:Joker}
\end{figure}
We run Stein transport, SVGD, and adjusted Stein transport, initialising  $N = 500$ particles from the prior. For Stein transport and the adjusted variant, we partition the time interval $[0,1]$ into $N_{\mathrm{steps}} = 50$ equidistant steps and use the regularisation parameter $\lambda = 10^{-2}$ (see Algorithm \ref{alg:stein}). Adjusted Stein transport (Algorithm \ref{alg:adjusted}) is implemented with one SVGD step per transport step, with step size $\Delta_{\mathrm{adjust}} = 0.02$ (this is the same step size as for the transport step), and the same regularisation $\lambda = 10^{-2}$. SVGD is run for $250$ steps with step size $\Delta t = 0.01$, and using Adagrad updates as in \cite{liu2016stein}. 
All methods rely on the square-exponential kernel \eqref{eq:sq exp kernel}, with dynamically selected bandwidth according to the median heuristic \citep{liu2016stein}. 

Figure \ref{fig:Joker} shows the posterior approximations obtained by the different methods.
In Figure \ref{fig:KSD_joker}, we show the kernelised Stein discrepancy (KSD) towards the target $\pi$, as a function of the iteration count (or rather, as a function of the number of evaluations of $\nabla h$). As suggested by \citet{gorham2017measuring}, KSD is implemented using the inverse multiquadric kernel $k_{\mathrm{IMQ}}(x,y) = (1 + \Vert x- y \Vert^2)^{-1/2}$. Notice that adjusted Stein transport takes twice as many gradient evaluations compared to the unadjusted counterpart, due to the interspersed SVGD steps.
Figure \ref{fig:Joker_transport} shows that (unadjusted) Stein transport is only able to provide a rather inaccurate approximation of the posterior. As discussed in Section \ref{sec:adjusted}, this observation can be attributed to instabilities and accumulation of errors. Indeed, Figure \ref{fig:KSD_joker} shows that unadjusted Stein transport fails to improve the KSD-score after roughly half of the steps, an indication that the particle representation of the intermediate distributions has become unreliable (or the Gram matrix $\bm{\xi}$ has degenerated) and thus Stein transport has become ineffective. Comparing Figures \ref{fig:Joker_svgd} and \ref{fig:Joker_mix}, we observe that adjusted Stein transport is able to fit the sharp contour lines of the posterior more accurately than SVGD, resulting in a lower final KSD-score (see Figure \ref{fig:KSD_joker}). We attribute this finding to the fact that (adjusted) Stein transport follows a Newton-type gradient flow (see Section \ref{sec:gradient flows}), and thus our results are in line with those of \citet{detommaso2018stein} for the Stein variational Newton method. Notice that, in contrast to Stein variational Newton, adjusted Stein transport does not require evaluations of the Hessian of the log-target.
\noindent
\begin{figure}[ht]
  \begin{minipage}[t]{0.49\textwidth}
    \centering
    \includegraphics[width=\textwidth]{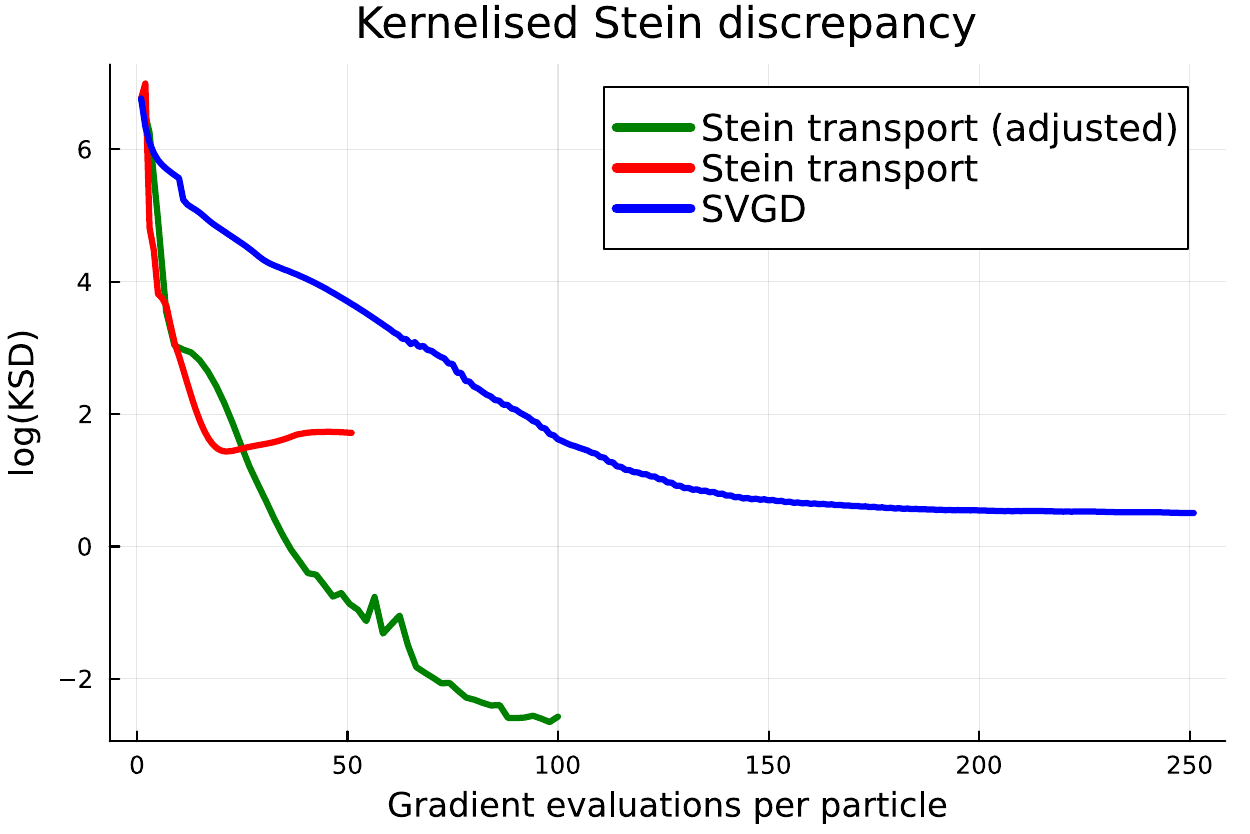} 
    \captionof{figure}{KSD evolution for the Joker distribution from Section \ref{sec:joker}.}
    \label{fig:KSD_joker}
  \end{minipage}\hfill
  \begin{minipage}[t]{0.49\textwidth}
    \centering
    \includegraphics[width=\textwidth]{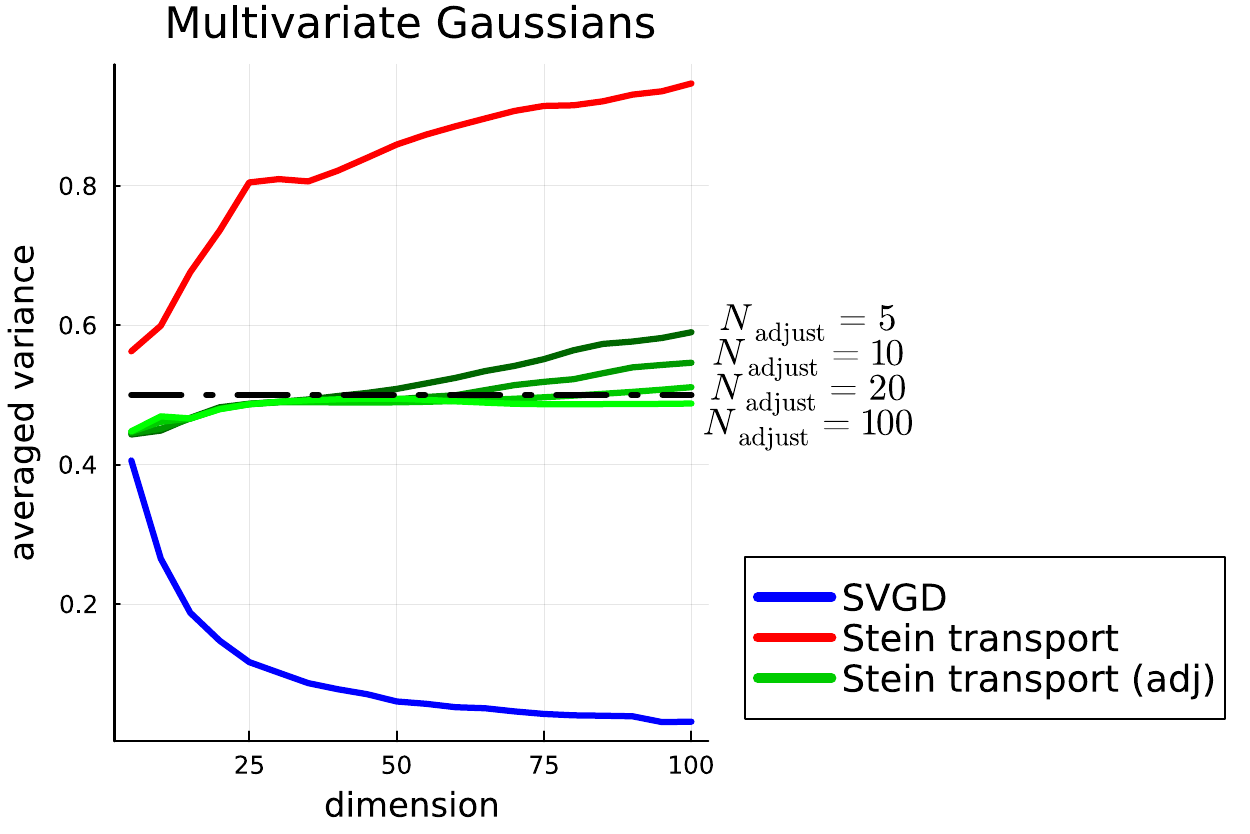} 
    \captionof{figure}{Averaged variance $\tfrac{1}{d} \mathrm{Tr} \,\widehat{\mathrm{Cov}}$ of the approximate posterior, as a function of the dimension. The black line indicates the true value $\tfrac{1}{2}$.}
    \label{fig:mult_Gauss}
  \end{minipage}
\end{figure}
\subsubsection{(Non-)Collapse in high dimensions}
\label{sec:collapse}
In this section, we perform a numerical investigation of the variance under- and overestimation phenomenon of SVGD and Stein transport, respectively (see Remark \ref{rem:SVGD collapse} and the discussion around Proposition \ref{prop:projection}). We also refer to \citet{ba2021understanding,gong2021sliced,liu2022Stein,zhuo2018message} for background.

\textbf{Multivariate Gaussians.} In a first experiment, we consider sampling from a multivariate Gaussian target, a standard example that illustrates the variance collapse of SVGD in high dimensions \citep{ba2021understanding,gong2021sliced,liu2022Stein,zhuo2018message}. We initialise the particles from the Gaussian prior $\mathcal{N}(\bm{1},I_{d \times d})$, where $\bm{1} = (1, \ldots, 1) \in \mathbb{R}^d$ denotes the $d$-dimensional `all ones' vector. With the negative log-likelihood $h(x) = \tfrac{1}{2} \Vert x + \bm{1} \Vert^2$ (that is, we impose a Gaussian observation model with unit covariance and observation $y_{\mathrm{obs}} = - \bm{1}$), the target posterior is $\mathcal{N}(0, \tfrac{1}{2} I_{d \times d})$. We initialise $200$ particles from the prior, and run SVGD (using the Adagrad optimiser with step size $\Delta t = 0.1$ for $200$ steps) as well as Stein transport (for $100$ steps and regularisation $\lambda = 10^{-2}$). The adjusted variant is run with $5$, $10$, $20$ or $100$ interspersed SVGD steps, each of which with step size $\Delta t_{\mathrm{adjust}} = 0.1$ and using Adagrad.\footnote{Note that using Adagrad here instead of plain SVGD is a slight departure from Algorithm \ref{alg:adjusted}.}

In Figure \ref{fig:mult_Gauss}, we plot $\tfrac{1}{d} \mathrm{Tr} \,\widehat{\mathrm{Cov}}$, the trace of the estimated covariance matrix, rescaled by $\tfrac{1}{d}$, as a function of $d$. The black line indicates the true value, $\tfrac{1}{d} \mathrm{Tr} \,\mathrm{Cov} \mathcal{N}(0, \tfrac{1}{2} I_{d \times d}) = \tfrac{1}{2}$. In line with prior works and Remark \ref{rem:SVGD collapse}, SVGD severely underestimates the posterior variance in high-dimensional scenarios. In contrast (and in line with Proposition \ref{prop:projection}), unadjusted Stein transport severely overestimates the posterior variance. Adjusted Stein transport provides a fairly accurate value for the posterior variance, especially if the number of interspersed SVGD-adjustment steps is large enough ($\approx 20$). We have included the plot for $100$ adjustment steps in order to show that the performance saturates;  finding an appropriate balance between adjustment and transport moves does not seem to be an issue.

\textbf{Low-rank Gaussian mixture.} To further showcase the ability of adjusted Stein transport to avoid posterior collapse in high dimensional settings, we consider the task of sampling from a Gaussian mixture with low-rank structure, following \citet{liu2022Stein}. More specifically, the target is given by $\pi(x) = \tfrac{1}{4}\sum_{i=1}^4 \mathcal{N}(x;\mu_i,I_{d \times d})$, where the means are defined as
\begin{equation*}
\mu_j = \left( \sqrt{5} \cos (2j\pi/4 + \pi/4), \sqrt{5} \cos (2j\pi/4 + \pi/4), 0,\ldots,0 \right)^\top \in \mathbb{R}^d,
\end{equation*}
for $j = 1,\ldots,4$.
The low-rank structure manifests itself in the first two coordinates: the means are placed on a circle, in an equidistant way (see Figure \ref{fig:low rank mix}). The marginal in the remaining $d-2$ coordinates is standard Gaussian, so that when initialising the particles from the prior $\mathcal{N}(0,I_{d \times d})$, only the first two coordinates need to be shifted.

We set $d = 50$, and implement SVGD and adjusted Stein transport with $N = 200$ particles. To be specific, SVGD is run for $150$ Adagrad steps (with step size $\Delta t = 0.01$) and adjusted Stein transport is run for $100$ steps (that is, the step size is $\Delta t= 0.01$), interspersed with $N_{\mathrm{adjust}} = 20$ SVGD Adagrad adjustment steps of size $\Delta t_{\mathrm{adjust}} = 0.01$. As we can see in Figure \ref{fig:low_rank_svgd} and already observed by \citet{liu2022Stein}, SVGD severely underestimates the spread of the Gaussian mixture components. Adjusted Stein transport, on the other hand, reaches a satisfactory posterior approximation (note that an ensemble size of $200$ particles is relatively small in $50$ dimensions). In this example, the adjustment steps proved to be crucial; we found it difficult to obtain satisfactory results with unadjusted Stein transport due to instabilities and accumulation of errors.

\begin{figure}[htbp]
  \centering
  \begin{subfigure}[b]{0.5\textwidth}
    \centering
    \includegraphics[width=\textwidth]{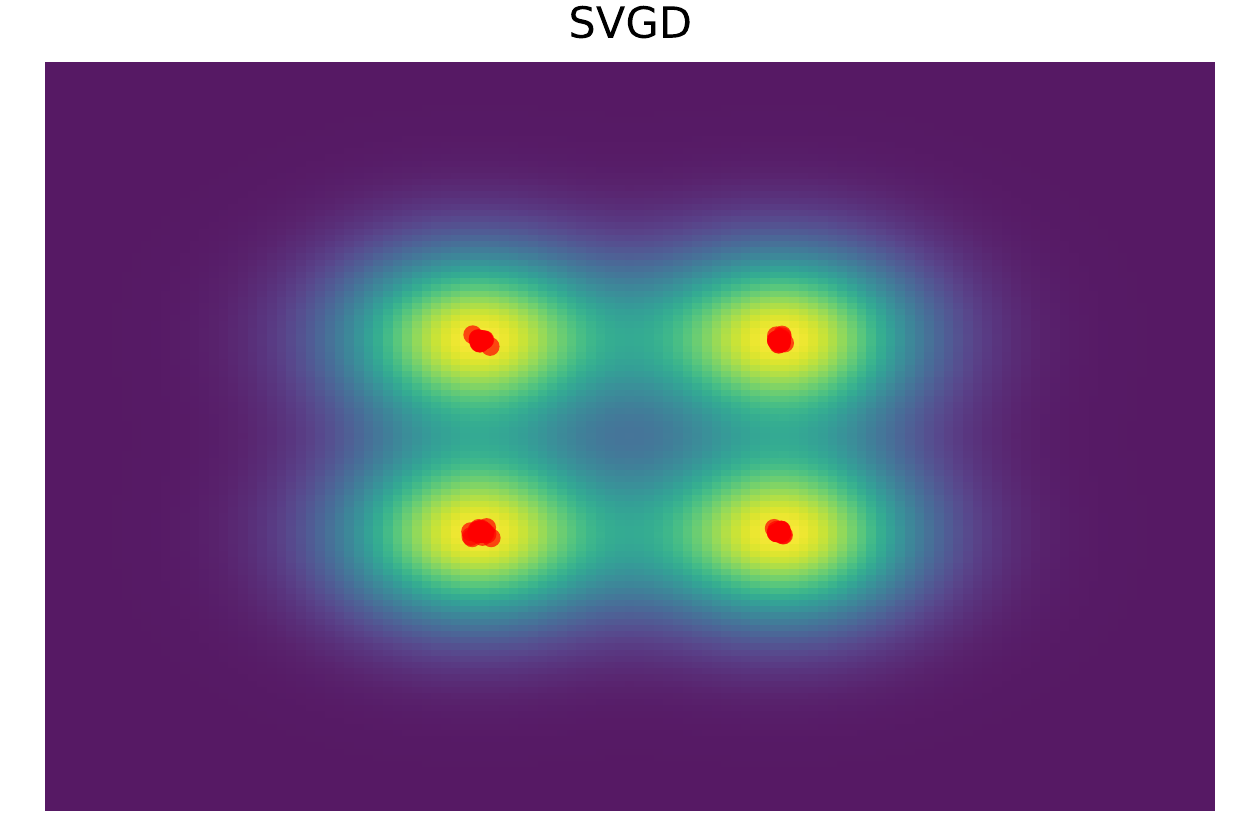} 
    \caption{SVGD.}
    \label{fig:low_rank_svgd}
  \end{subfigure}\hspace*{\fill}%
  \begin{subfigure}[b]{0.5\textwidth}
    \centering
    \includegraphics[width=\textwidth]{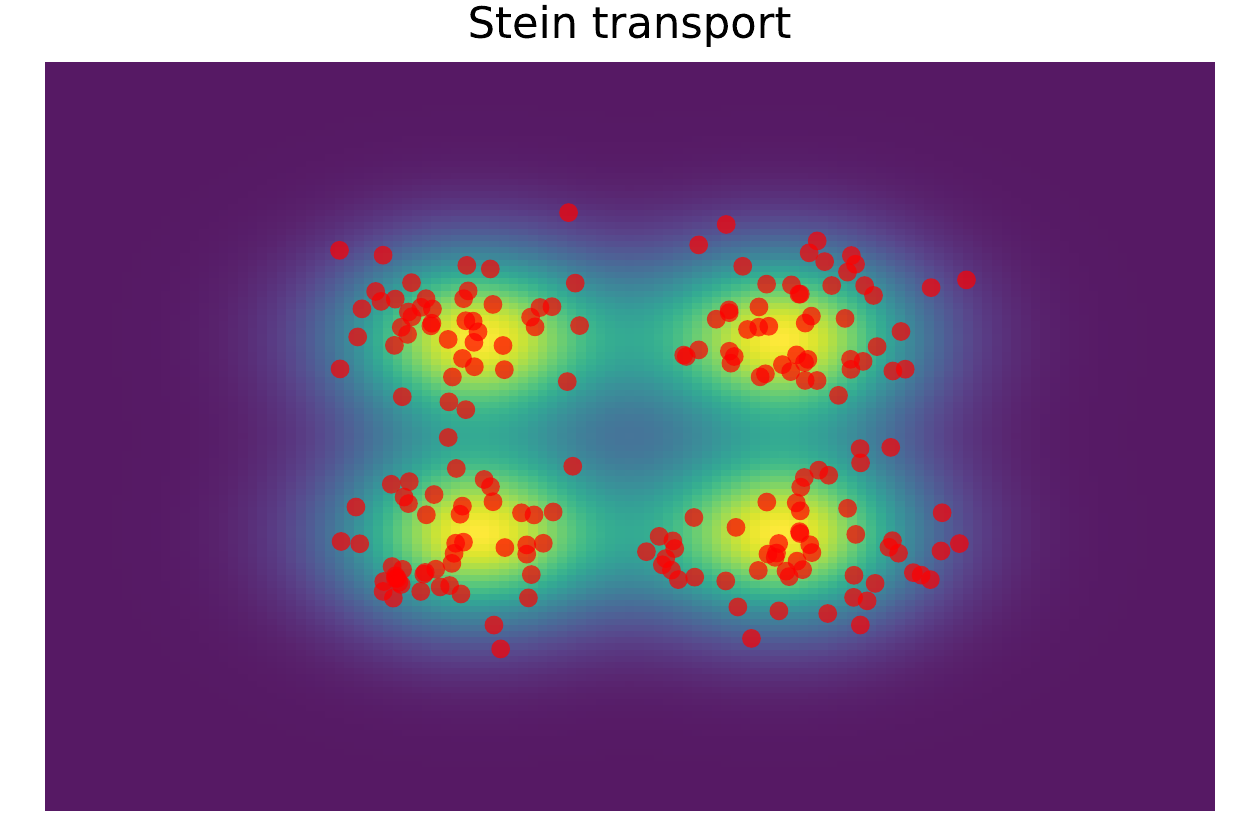} 
    \caption{Adjusted Stein transport.}
  \end{subfigure}
  \caption{Gaussian mixture with low-rank structure ($d = 50$, ensemble size $200$ particles). We show the marginals in the first two coordinates of the approximations obtained by SVGD and adjusted Stein transport.}
  \label{fig:low rank mix}
\end{figure}

\subsubsection{Convergence in unit time: Bayesian logistic regression}
\label{sec:logistic reg}

In our last experiment, we compare SVGD and (adjusted) Stein transport in the context of a Bayesian logistic regression task. We use the $60$-dimensional Splice dataset \citep{ratsch2001soft}, assume a standard Gaussian prior, and use $N = 500$ particles. SVGD is run using the Adagrad optimiser with standard parameters and time step $\Delta t = 0.01$ (significantly increasing the step size leads to instabilities). Adjusted Stein transport is run using $50$ steps (that is, with time step $\Delta t = 0.02$) and one SVGD-adjustment step (with time step $\Delta t = 0.01$) per transport step. Figure \ref{fig:BLR} shows the time evolution of the KSD (using the inverse multiquadric kernel $k_{\mathrm{IMQ}}(x,y) = (1 + \Vert x- y \Vert^2_2)^{-1/2}$ as in Section \ref{sec:joker}) and the test accuracy along the dynamics of the particle systems. We observe that adjusted Stein transport reaches low KSD-scores and high test accuracies with significantly less gradient evaluations per particle. We take this observation as an indication that the construction principle behind Stein transport (converging to the posterior at time $t = 1$ vs $t \rightarrow \infty$ for SVGD) can indeed significantly reduce the computational cost.

\begin{figure}[htbp]
  \centering
  \begin{subfigure}[b]{0.5\textwidth}
    \centering
    \includegraphics[width=\textwidth]{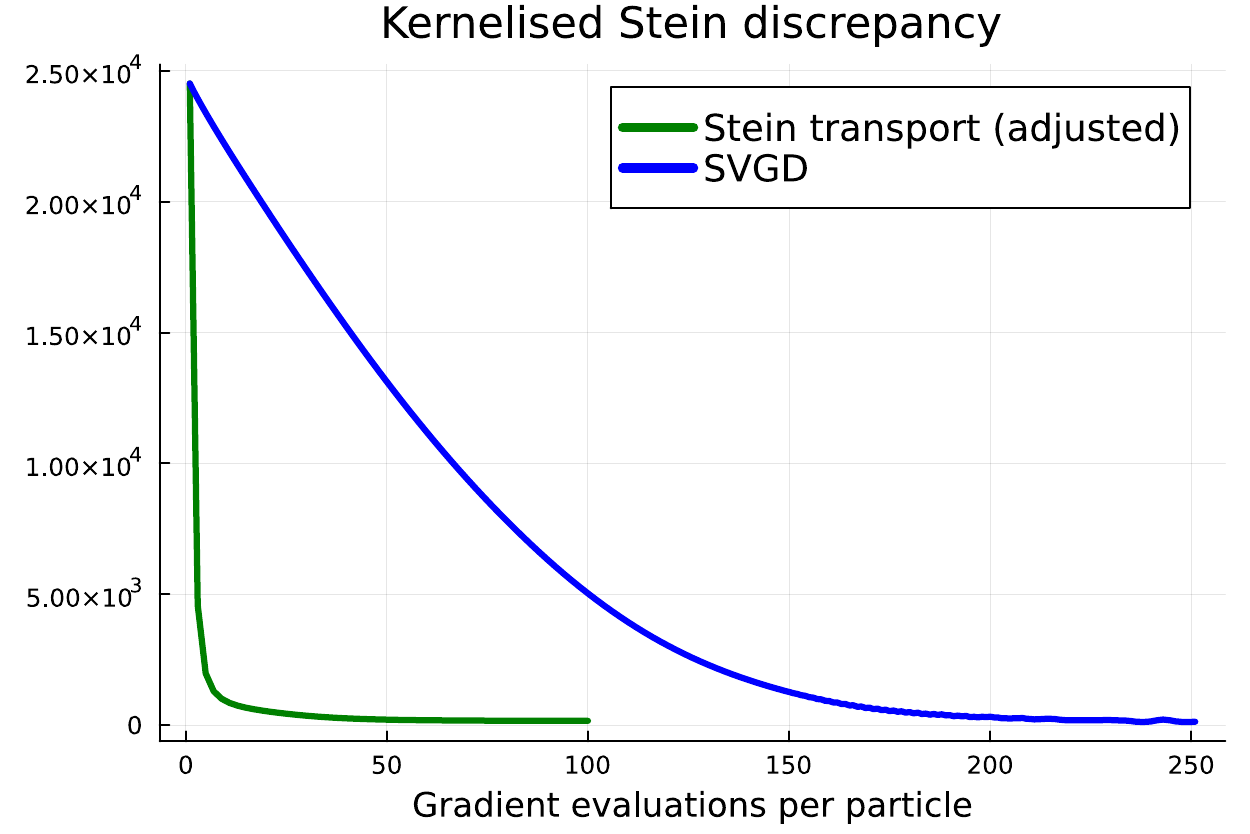} 
    \label{fig:subfig1}
  \end{subfigure}\hspace*{\fill}%
  \begin{subfigure}[b]{0.5\textwidth}
    \centering
    \includegraphics[width=\textwidth]{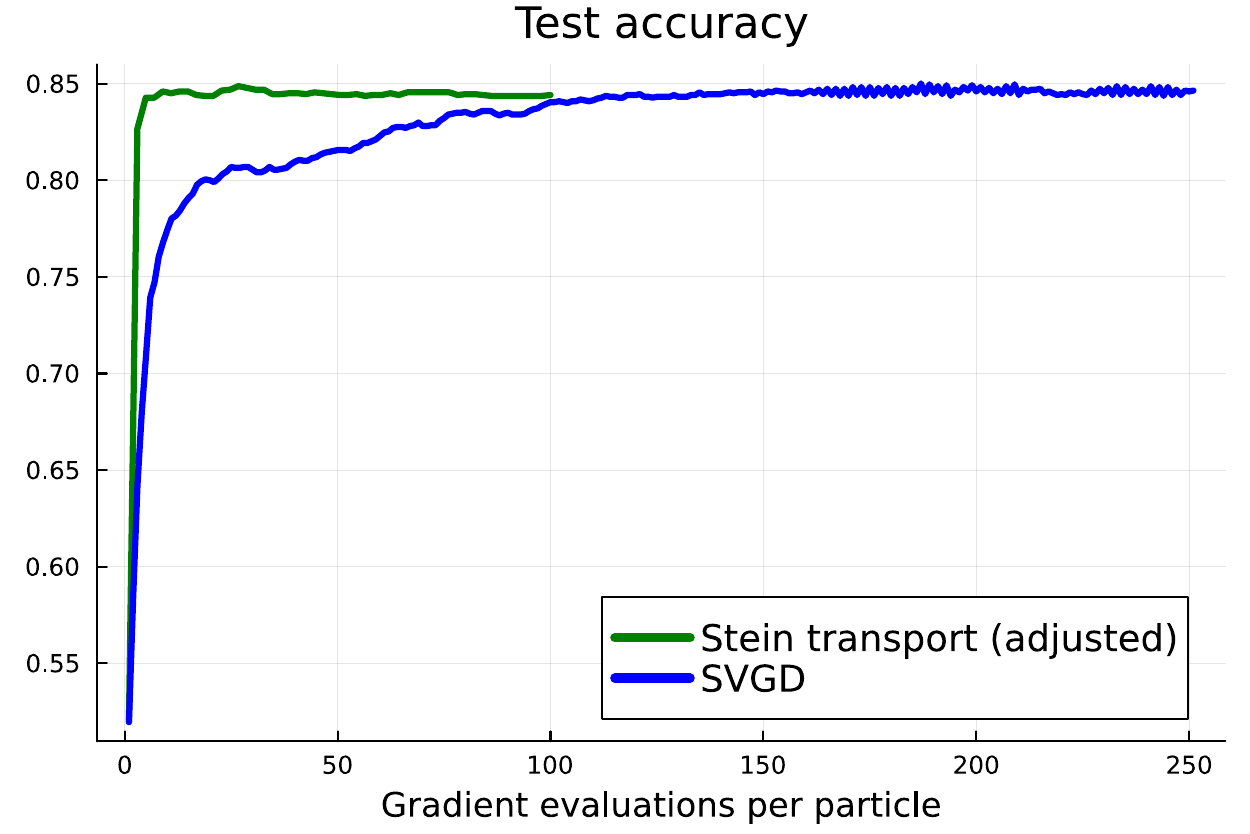} 
  \end{subfigure}
  \caption{Bayesian logistic regression for the Splice data set ($d = 60$): KSD and test accuracy along the time evolution of the particle system.}
  \label{fig:BLR}
\end{figure}

\section{Conclusion, limitations and outlook}
\label{sec:outlook}

We have developed Stein transport based on a hybrid framework that combines regression with dynamics, resulting in a scheme reminiscent of SVGD and sharing its geometrical foundations. The adjusted variant in particular demonstrates promising numerical results that align closely with the underlying theory. Firstly, Stein transport is designed to reach its posterior approximation at time 
$t=1$; in our experiments, it indeed achieved similar or superior accuracy compared to SVGD while requiring significantly less computational effort (see in particular Section \ref{sec:logistic reg}). Secondly, since Stein transport is constructed to follow a preconditioned gradient flow (see equation \eqref{eq:Newton flow}), it produces a more accurate posterior approximation in scenarios where the target distribution exhibits complex geometry. Thirdly, unlike SVGD, which often suffers from particle collapse in high-dimensional settings, Stein transport maintains robustness against such issues, arguably because it is constructed from a regression perspective (see Section \ref{sec:collapse}). This construction ensures that the finite-particle vector field remains close to a projection derived from a mean-field limit (see Proposition \ref{prop:projection}).
In future work, it would be promising to further explore the theoretical insights into regularisation and finite-particle effects (see Section \ref{sec:mean field dynamics}), and to deepen the connections between statistical estimation and transport methods more broadly. Specifically, it would be valuable to systematically develop weighted schemes and incorporate principled adjustment mechanisms. 

\paragraph{Acknowledgements.}
This work was partly supported by Deutsche Forschungsgemeinschaft (DFG) through the
grant CRC 1114 ‘Scaling Cascades in Complex Systems’ (project A02, project number 235221301, second funding phase). Many thanks to Deniz Akyildiz for very valuable comments on a preliminary version of this manuscript!

\bibliography{refs.bib}

\begin{thebibliography}{114}
\providecommand{\natexlab}[1]{#1}
\providecommand{\url}[1]{\texttt{#1}}
\expandafter\ifx\csname urlstyle\endcsname\relax
  \providecommand{\doi}[1]{doi: #1}\else
  \providecommand{\doi}{doi: \begingroup \urlstyle{rm}\Url}\fi

\bibitem[Amari(2016)]{amari2016information}
S.-i. Amari.
\newblock \emph{Information geometry and its applications}, volume 194.
\newblock Springer, 2016.

\bibitem[Ambrosio et~al.(2008)Ambrosio, Gigli, and
  Savar{\'e}]{ambrosio2008gradient}
L.~Ambrosio, N.~Gigli, and G.~Savar{\'e}.
\newblock \emph{Gradient flows: in metric spaces and in the space of
  probability measures}.
\newblock Springer Science \& Business Media, 2008.

\bibitem[Anastasiou et~al.(2023)Anastasiou, Barp, Briol, Ebner, Gaunt,
  Ghaderinezhad, Gorham, Gretton, Ley, Liu, et~al.]{anastasiou2023stein}
A.~Anastasiou, A.~Barp, F.-X. Briol, B.~Ebner, R.~E. Gaunt, F.~Ghaderinezhad,
  J.~Gorham, A.~Gretton, C.~Ley, Q.~Liu, et~al.
\newblock Stein’s method meets computational statistics: A review of some
  recent developments.
\newblock \emph{Statistical Science}, 38\penalty0 (1):\penalty0 120--139, 2023.

\bibitem[Arbel et~al.(2019)Arbel, Korba, Salim, and Gretton]{arbel2019maximum}
M.~Arbel, A.~Korba, A.~Salim, and A.~Gretton.
\newblock Maximum mean discrepancy gradient flow.
\newblock In \emph{Advances in Neural Information Processing Systems
  (NeurIPS)}, 2019.

\bibitem[Ay et~al.(2017)Ay, Jost, V{\^a}n~L{\^e}, and
  Schwachh{\"o}fer]{ay2017information}
N.~Ay, J.~Jost, H.~V{\^a}n~L{\^e}, and L.~Schwachh{\"o}fer.
\newblock \emph{Information geometry}, volume~64.
\newblock Springer, 2017.

\bibitem[Ba et~al.(2021)Ba, Erdogdu, Ghassemi, Sun, Suzuki, Wu, and
  Zhang]{ba2021understanding}
J.~Ba, M.~A. Erdogdu, M.~Ghassemi, S.~Sun, T.~Suzuki, D.~Wu, and T.~Zhang.
\newblock Understanding the variance collapse of {SVGD} in high dimensions.
\newblock In \emph{International Conference on Learning Representations
  (ICML)}, 2021.

\bibitem[Benamou and Brenier(2000)]{benamou2000computational}
J.-D. Benamou and Y.~Brenier.
\newblock A computational fluid mechanics solution to the {M}onge-{K}antorovich
  mass transfer problem.
\newblock \emph{Numerische Mathematik}, 84\penalty0 (3):\penalty0 375--393,
  2000.

\bibitem[Blanchard and M{\"u}cke(2018)]{blanchard2018optimal}
G.~Blanchard and N.~M{\"u}cke.
\newblock Optimal rates for regularization of statistical inverse learning
  problems.
\newblock \emph{Foundations of Computational Mathematics}, 18\penalty0
  (4):\penalty0 971--1013, 2018.

\bibitem[Blei et~al.(2017)Blei, Kucukelbir, and McAuliffe]{blei2017variational}
D.~M. Blei, A.~Kucukelbir, and J.~D. McAuliffe.
\newblock Variational inference: A review for statisticians.
\newblock \emph{Journal of the American statistical Association}, 112\penalty0
  (518):\penalty0 859--877, 2017.

\bibitem[Brooks et~al.(2011)Brooks, Gelman, Jones, and
  Meng]{brooks2011handbook}
S.~Brooks, A.~Gelman, G.~Jones, and X.-L. Meng.
\newblock \emph{Handbook of {M}arkov chain {M}onte {C}arlo}.
\newblock CRC press, 2011.

\bibitem[Chen et~al.(2021)Chen, Hosseini, Owhadi, and Stuart]{chen2021solving}
Y.~Chen, B.~Hosseini, H.~Owhadi, and A.~M. Stuart.
\newblock Solving and learning nonlinear {PDEs} with {G}aussian processes.
\newblock \emph{Journal of Computational Physics}, 447:\penalty0 110668, 2021.

\bibitem[Chen et~al.(2023{\natexlab{a}})Chen, Huang, Huang, Reich, and
  Stuart]{chen2023gradient}
Y.~Chen, D.~Z. Huang, J.~Huang, S.~Reich, and A.~M. Stuart.
\newblock Gradient flows for sampling: mean-field models, {G}aussian
  approximations and affine invariance.
\newblock \emph{arXiv preprint arXiv:2302.11024}, 2023{\natexlab{a}}.

\bibitem[Chen et~al.(2023{\natexlab{b}})Chen, Huang, Huang, Reich, and
  Stuart]{chen2023sampling}
Y.~Chen, D.~Z. Huang, J.~Huang, S.~Reich, and A.~M. Stuart.
\newblock Sampling via gradient flows in the space of probability measures.
\newblock \emph{arXiv preprint arXiv:2310.03597}, 2023{\natexlab{b}}.

\bibitem[Chessari et~al.(2023)Chessari, Kawai, Shinozaki, and
  Yamada]{chessari2023numerical}
J.~Chessari, R.~Kawai, Y.~Shinozaki, and T.~Yamada.
\newblock Numerical methods for backward stochastic differential equations: A
  survey.
\newblock \emph{Probability Surveys}, 20:\penalty0 486--567, 2023.

\bibitem[Chewi et~al.(2020)Chewi, Le~Gouic, Lu, Maunu, and
  Rigollet]{chewi2020svgd}
S.~Chewi, T.~Le~Gouic, C.~Lu, T.~Maunu, and P.~Rigollet.
\newblock {SVGD} as a kernelized {W}asserstein gradient flow of the chi-squared
  divergence.
\newblock \emph{Advances in Neural Information Processing Systems (NeurIPS)},
  2020.

\bibitem[Chewi et~al.(2024)Chewi, Niles-Weed, and
  Rigollet]{chewi2024statistical}
S.~Chewi, J.~Niles-Weed, and P.~Rigollet.
\newblock Statistical optimal transport.
\newblock \emph{arXiv preprint arXiv:2407.18163}, 2024.

\bibitem[Chizat et~al.(2018)Chizat, Peyr{\'e}, Schmitzer, and
  Vialard]{chizat2018unbalanced}
L.~Chizat, G.~Peyr{\'e}, B.~Schmitzer, and F.-X. Vialard.
\newblock Unbalanced optimal transport: Dynamic and {K}antorovich formulations.
\newblock \emph{Journal of Functional Analysis}, 274\penalty0 (11):\penalty0
  3090--3123, 2018.

\bibitem[Chopin and Papaspiliopoulos(2020)]{chopin2020introduction}
N.~Chopin and O.~Papaspiliopoulos.
\newblock \emph{An introduction to sequential Monte Carlo}.
\newblock Springer, 2020.

\bibitem[Chopin et~al.(2023)Chopin, Crucinio, and Korba]{chopin2023connection}
N.~Chopin, F.~R. Crucinio, and A.~Korba.
\newblock A connection between tempering and entropic mirror descent.
\newblock \emph{arXiv preprint arXiv:2310.11914}, 2023.

\bibitem[Chwialkowski et~al.(2016)Chwialkowski, Strathmann, and
  Gretton]{chwialkowski2016kernel}
K.~Chwialkowski, H.~Strathmann, and A.~Gretton.
\newblock A kernel test of goodness of fit.
\newblock In \emph{International conference on machine learning}, pages
  2606--2615. PMLR, 2016.

\bibitem[Coghi et~al.(2023)Coghi, Nilssen, N{\"u}sken, and
  Reich]{coghi2023rough}
M.~Coghi, T.~Nilssen, N.~N{\"u}sken, and S.~Reich.
\newblock Rough {McKean}--{V}lasov dynamics for robust ensemble {K}alman
  filtering.
\newblock \emph{The Annals of Applied Probability}, 33\penalty0 (6B):\penalty0
  5693--5752, 2023.

\bibitem[Daum and Huang(2008)]{daum2008particle}
F.~Daum and J.~Huang.
\newblock Particle flow for nonlinear filters with log-homotopy.
\newblock In \emph{Signal and Data Processing of Small Targets 2008}, volume
  6969, page 696918. International Society for Optics and Photonics, 2008.

\bibitem[Daum and Huang(2009)]{daum2009nonlinear}
F.~Daum and J.~Huang.
\newblock Nonlinear filters with particle flow induced by log-homotopy.
\newblock In \emph{Signal Processing, Sensor Fusion, and Target Recognition
  XVIII}, volume 7336, page 733603. International Society for Optics and
  Photonics, 2009.

\bibitem[Daum et~al.(2010)Daum, Huang, and Noushin]{daum2010exact}
F.~Daum, J.~Huang, and A.~Noushin.
\newblock Exact particle flow for nonlinear filters.
\newblock In \emph{Signal processing, sensor fusion, and target recognition
  XIX}, volume 7697, page 769704. International society for optics and
  photonics, 2010.

\bibitem[De~Vito et~al.(2005)De~Vito, Caponnetto, and Rosasco]{de2005model}
E.~De~Vito, A.~Caponnetto, and L.~Rosasco.
\newblock Model selection for regularized least-squares algorithm in learning
  theory.
\newblock \emph{Foundations of Computational Mathematics}, 5:\penalty0 59--85,
  2005.

\bibitem[Del~Moral et~al.(2006)Del~Moral, Doucet, and Jasra]{del2006sequential}
P.~Del~Moral, A.~Doucet, and A.~Jasra.
\newblock Sequential monte carlo samplers.
\newblock \emph{Journal of the Royal Statistical Society Series B: Statistical
  Methodology}, 68\penalty0 (3):\penalty0 411--436, 2006.

\bibitem[Detommaso et~al.(2018)Detommaso, Cui, Marzouk, Spantini, and
  Scheichl]{detommaso2018stein}
G.~Detommaso, T.~Cui, Y.~Marzouk, A.~Spantini, and R.~Scheichl.
\newblock A {S}tein variational {N}ewton method.
\newblock In \emph{Advances in Neural Information Processing Systems
  (NeurIPS)}, 2018.

\bibitem[Domingo{-}Enrich and Pooladian(2023)]{domingo2023explicit}
C.~Domingo{-}Enrich and A.~Pooladian.
\newblock An explicit expansion of the {K}ullback-{L}eibler divergence along
  its {F}isher-{R}ao gradient flow.
\newblock \emph{Trans. Mach. Learn. Res.}, 2023, 2023.

\bibitem[Dunbar et~al.(2022)Dunbar, Duncan, Stuart, and
  Wolfram]{dunbar2022ensemble}
O.~R. Dunbar, A.~B. Duncan, A.~M. Stuart, and M.-T. Wolfram.
\newblock Ensemble inference methods for models with noisy and expensive
  likelihoods.
\newblock \emph{SIAM Journal on Applied Dynamical Systems}, 21\penalty0
  (2):\penalty0 1539--1572, 2022.

\bibitem[Duncan et~al.(2023)Duncan, N{\"u}sken, and
  Szpruch]{duncan2019geometry}
A.~Duncan, N.~N{\"u}sken, and L.~Szpruch.
\newblock On the geometry of {S}tein variational gradient descent.
\newblock \emph{Journal of Machine Learning Research}, 24:\penalty0 1--39,
  2023.

\bibitem[El~Moselhy and Marzouk(2012)]{el2012bayesian}
T.~A. El~Moselhy and Y.~M. Marzouk.
\newblock {B}ayesian inference with optimal maps.
\newblock \emph{Journal of Computational Physics}, 231\penalty0 (23):\penalty0
  7815--7850, 2012.

\bibitem[Engl et~al.(1996)Engl, Hanke, and Neubauer]{engl1996regularization}
H.~W. Engl, M.~Hanke, and A.~Neubauer.
\newblock \emph{Regularization of inverse problems}, volume 375.
\newblock Springer Science \& Business Media, 1996.

\bibitem[Fasshauer(2007)]{fasshauer2007meshfree}
G.~E. Fasshauer.
\newblock \emph{Meshfree approximation methods with MATLAB}, volume~6.
\newblock World Scientific, 2007.

\bibitem[Fisher et~al.(2021)Fisher, Nolan, Graham, Prangle, and
  Oates]{fisher2021measure}
M.~Fisher, T.~Nolan, M.~Graham, D.~Prangle, and C.~Oates.
\newblock Measure transport with kernel {S}tein discrepancy.
\newblock In \emph{International Conference on Artificial Intelligence and
  Statistics}, pages 1054--1062. PMLR, 2021.

\bibitem[Gallego and Insua(2018)]{gallego1812stochastic}
V.~Gallego and D.~Insua.
\newblock Stochastic gradient {MCMC} with repulsive forces.
\newblock \emph{arXiv:1812.00071}, 2018.

\bibitem[Gelman and Meng(1998)]{gelman1998simulating}
A.~Gelman and X.-L. Meng.
\newblock Simulating normalizing constants: From importance sampling to bridge
  sampling to path sampling.
\newblock \emph{Statistical science}, pages 163--185, 1998.

\bibitem[Gladin et~al.(2024)Gladin, Dvurechensky, Mielke, and
  Zhu]{gladin2024interaction}
E.~Gladin, P.~Dvurechensky, A.~Mielke, and J.-J. Zhu.
\newblock Interaction-force transport gradient flows.
\newblock \emph{arXiv:2405.17075}, 2024.

\bibitem[Gong et~al.(2021)Gong, Li, and
  Hern{\'{a}}ndez{-}Lobato]{gong2021sliced}
W.~Gong, Y.~Li, and J.~M. Hern{\'{a}}ndez{-}Lobato.
\newblock Sliced kernelized {S}tein discrepancy.
\newblock In \emph{9th International Conference on Learning Representations
  (ICLR)}, 2021.

\bibitem[Gorham and Mackey(2017)]{gorham2017measuring}
J.~Gorham and L.~Mackey.
\newblock Measuring sample quality with kernels.
\newblock In \emph{International Conference on Machine Learning (ICML)}, pages
  1292--1301. PMLR, 2017.

\bibitem[Gretton et~al.(2012)Gretton, Borgwardt, Rasch, Sch{\"o}lkopf, and
  Smola]{gretton2012kernel}
A.~Gretton, K.~M. Borgwardt, M.~J. Rasch, B.~Sch{\"o}lkopf, and A.~Smola.
\newblock A kernel two-sample test.
\newblock \emph{The Journal of Machine Learning Research}, 13\penalty0
  (1):\penalty0 723--773, 2012.

\bibitem[Hable and Christmann(2011)]{hable2011qualitative}
R.~Hable and A.~Christmann.
\newblock On qualitative robustness of support vector machines.
\newblock \emph{Journal of Multivariate Analysis}, 102\penalty0 (6):\penalty0
  993--1007, 2011.

\bibitem[He et~al.(2024)He, Balasubramanian, Sriperumbudur, and
  Lu]{he2022regularized}
Y.~He, K.~Balasubramanian, B.~K. Sriperumbudur, and J.~Lu.
\newblock Regularized {S}tein variational gradient flow.
\newblock \emph{Foundations of Computational Mathematics}, pages 1--59, 2024.

\bibitem[Heng et~al.(2021)Heng, Doucet, and Pokern]{heng2021gibbs}
J.~Heng, A.~Doucet, and Y.~Pokern.
\newblock Gibbs flow for approximate transport with applications to {B}ayesian
  computation.
\newblock \emph{Journal of the Royal Statistical Society: Series B (Statistical
  Methodology)}, 83\penalty0 (1):\penalty0 156--187, 2021.

\bibitem[Huan et~al.(2024)Huan, Jagalur, and Marzouk]{huan2024optimal}
X.~Huan, J.~Jagalur, and Y.~Marzouk.
\newblock Optimal experimental design: Formulations and computations.
\newblock \emph{arXiv:2407.16212}, 2024.

\bibitem[Jordan et~al.(1999)Jordan, Ghahramani, Jaakkola, and
  Saul]{jordan1999introduction}
M.~I. Jordan, Z.~Ghahramani, T.~S. Jaakkola, and L.~K. Saul.
\newblock An introduction to variational methods for graphical models.
\newblock \emph{Machine learning}, 37\penalty0 (2):\penalty0 183--233, 1999.

\bibitem[Jost and Jost(2008)]{jost2008riemannian}
J.~Jost and J.~Jost.
\newblock \emph{Riemannian geometry and geometric analysis}, volume 42005.
\newblock Springer, 2008.

\bibitem[Kanagawa et~al.(2018)Kanagawa, Hennig, Sejdinovic, and
  Sriperumbudur]{kanagawa2018gaussian}
M.~Kanagawa, P.~Hennig, D.~Sejdinovic, and B.~K. Sriperumbudur.
\newblock Gaussian processes and kernel methods: A review on connections and
  equivalences.
\newblock \emph{arXiv:1807.02582}, 2018.

\bibitem[Kerimkulov et~al.(2023)Kerimkulov, Leahy, Siska, Szpruch, and
  Zhang]{kerimkulov2023fisher}
B.~Kerimkulov, J.-M. Leahy, D.~Siska, L.~Szpruch, and Y.~Zhang.
\newblock A {F}isher-{R}ao gradient flow for entropy-regularised {M}arkov
  decision processes in {P}olish spaces.
\newblock \emph{arXiv:2310.02951}, 2023.

\bibitem[Kirsch(2021)]{kirsch2021introduction}
A.~Kirsch.
\newblock \emph{An introduction to the mathematical theory of inverse
  problems}, volume 120.
\newblock Springer Nature, 2021.

\bibitem[Komorowski et~al.(2012)Komorowski, Landim, and
  Olla]{komorowski2012fluctuations}
T.~Komorowski, C.~Landim, and S.~Olla.
\newblock \emph{Fluctuations in Markov processes: time symmetry and martingale
  approximation}, volume 345.
\newblock Springer Science \& Business Media, 2012.

\bibitem[Korba et~al.(2020)Korba, Salim, Arbel, Luise, and
  Gretton]{korba2020non}
A.~Korba, A.~Salim, M.~Arbel, G.~Luise, and A.~Gretton.
\newblock A non-asymptotic analysis for {S}tein variational gradient descent.
\newblock \emph{Advances in Neural Information Processing Systems (NeurIPS)},
  2020.

\bibitem[Korba et~al.(2021)Korba, Aubin-Frankowski, Majewski, and
  Ablin]{korba2021kernel}
A.~Korba, P.-C. Aubin-Frankowski, S.~Majewski, and P.~Ablin.
\newblock Kernel {S}tein discrepancy descent.
\newblock \emph{International Conference on Machine Learning (ICML)}, 2021.

\bibitem[Laugesen et~al.(2015)Laugesen, Mehta, Meyn, and
  Raginsky]{laugesen2015poisson}
R.~S. Laugesen, P.~G. Mehta, S.~P. Meyn, and M.~Raginsky.
\newblock Poisson's equation in nonlinear filtering.
\newblock \emph{SIAM Journal on Control and Optimization}, 53\penalty0
  (1):\penalty0 501--525, 2015.

\bibitem[Law et~al.(2015)Law, Stuart, and Zygalakis]{law2015data}
K.~Law, A.~Stuart, and K.~Zygalakis.
\newblock Data assimilation.
\newblock \emph{Cham, Switzerland: Springer}, 214:\penalty0 52, 2015.

\bibitem[Lee and Lee(2012)]{lee2012smooth}
J.~M. Lee and J.~M. Lee.
\newblock \emph{Smooth manifolds}.
\newblock Springer, 2012.

\bibitem[Lelievre and Stoltz(2016)]{lelievre2016partial}
T.~Lelievre and G.~Stoltz.
\newblock Partial differential equations and stochastic methods in molecular
  dynamics.
\newblock \emph{Acta Numerica}, 25:\penalty0 681--880, 2016.

\bibitem[Liero et~al.(2018)Liero, Mielke, and Savar{\'e}]{liero2018optimal}
M.~Liero, A.~Mielke, and G.~Savar{\'e}.
\newblock Optimal entropy-transport problems and a new
  {H}ellinger--{K}antorovich distance between positive measures.
\newblock \emph{Inventiones mathematicae}, 211\penalty0 (3):\penalty0
  969--1117, 2018.

\bibitem[Lipman et~al.(2023)Lipman, Chen, Ben{-}Hamu, Nickel, and
  Le]{lipman2023flow}
Y.~Lipman, R.~T.~Q. Chen, H.~Ben{-}Hamu, M.~Nickel, and M.~Le.
\newblock Flow matching for generative modeling.
\newblock In \emph{The Eleventh International Conference on Learning
  Representations, {ICLR}}, 2023.

\bibitem[Liu et~al.(2019)Liu, Zhuo, Cheng, Zhang, and
  Zhu]{liu2019understanding}
C.~Liu, J.~Zhuo, P.~Cheng, R.~Zhang, and J.~Zhu.
\newblock Understanding and accelerating particle-based variational inference.
\newblock In \emph{International Conference on Machine Learning}, pages
  4082--4092. PMLR, 2019.

\bibitem[Liu(2017)]{liu2017stein}
Q.~Liu.
\newblock Stein variational gradient descent as gradient flow.
\newblock In \emph{Advances in Neural Information Processing Systems
  (NeurIPS)}, 2017.

\bibitem[Liu and Wang(2016)]{liu2016stein}
Q.~Liu and D.~Wang.
\newblock {S}tein variational gradient descent: A general purpose {B}ayesian
  inference algorithm.
\newblock In \emph{Advances in Neural Information Processing Systems
  (NeurIPS)}, 2016.

\bibitem[Liu et~al.(2016)Liu, Lee, and Jordan]{liu2016kernelized}
Q.~Liu, J.~Lee, and M.~Jordan.
\newblock A kernelized {S}tein discrepancy for goodness-of-fit tests.
\newblock In \emph{International conference on machine learning (ICML)}, pages
  276--284. PMLR, 2016.

\bibitem[Liu et~al.(2022)Liu, Zhu, Ton, Wynne, and Duncan]{liu2022Stein}
X.~Liu, H.~Zhu, J.~Ton, G.~Wynne, and A.~B. Duncan.
\newblock Grassmann {S}tein variational gradient descent.
\newblock In \emph{International Conference on Artificial Intelligence and
  Statistics, {AISTATS}}. {PMLR}, 2022.

\bibitem[Liu et~al.(2023)Liu, Gong, and Liu]{liu2023flow}
X.~Liu, C.~Gong, and Q.~Liu.
\newblock Flow straight and fast: Learning to generate and transfer data with
  rectified flow.
\newblock In \emph{The Eleventh International Conference on Learning
  Representations (ICLR)}, 2023.

\bibitem[Lu et~al.(2019{\natexlab{a}})Lu, Lu, and Nolen]{lu2019scaling}
J.~Lu, Y.~Lu, and J.~Nolen.
\newblock Scaling limit of the {S}tein variational gradient descent: The mean
  field regime.
\newblock \emph{SIAM Journal on Mathematical Analysis}, 51\penalty0
  (2):\penalty0 648--671, 2019{\natexlab{a}}.

\bibitem[Lu et~al.(2019{\natexlab{b}})Lu, Lu, and Nolen]{lu2019accelerating}
Y.~Lu, J.~Lu, and J.~Nolen.
\newblock Accelerating {L}angevin sampling with birth-death.
\newblock \emph{arXiv:1905.09863}, 2019{\natexlab{b}}.

\bibitem[Lu et~al.(2023)Lu, Slep{\v{c}}ev, and Wang]{lu2023birth}
Y.~Lu, D.~Slep{\v{c}}ev, and L.~Wang.
\newblock Birth--death dynamics for sampling: global convergence,
  approximations and their asymptotics.
\newblock \emph{Nonlinearity}, 36\penalty0 (11):\penalty0 5731, 2023.

\bibitem[Maoutsa et~al.(2020)Maoutsa, Reich, and Opper]{maoutsa2020interacting}
D.~Maoutsa, S.~Reich, and M.~Opper.
\newblock Interacting particle solutions of {F}okker--{P}lanck equations
  through gradient--log--density estimation.
\newblock \emph{Entropy}, 22\penalty0 (8):\penalty0 802, 2020.

\bibitem[Maurais and Marzouk(2023)]{maurais2023adaptive}
A.~Maurais and Y.~Marzouk.
\newblock Adaptive algorithms for continuous-time transport: Homotopy-driven
  sampling and a new interacting particle system.
\newblock In \emph{NeurIPS 2023 Workshop Optimal Transport and Machine
  Learning}, 2023.

\bibitem[Maurais and Marzouk(2024)]{maurais2024sampling}
A.~Maurais and Y.~Marzouk.
\newblock Sampling in unit time with kernel {F}isher-{R}ao flow.
\newblock \emph{arXiv:2401.03892}, 2024.

\bibitem[Mielke(2023)]{mielke2023introduction}
A.~Mielke.
\newblock An introduction to the analysis of gradients systems.
\newblock \emph{arXiv preprint arXiv:2306.05026}, 2023.

\bibitem[Muandet et~al.(2017)Muandet, Fukumizu, Sriperumbudur, Sch{\"o}lkopf,
  et~al.]{muandet2017kernel}
K.~Muandet, K.~Fukumizu, B.~Sriperumbudur, B.~Sch{\"o}lkopf, et~al.
\newblock Kernel mean embedding of distributions: A review and beyond.
\newblock \emph{Foundations and Trends{\textregistered} in Machine Learning},
  10\penalty0 (1-2):\penalty0 1--141, 2017.

\bibitem[N{\"u}sken and Renger(2023)]{nusken2021stein}
N.~N{\"u}sken and D.~Renger.
\newblock Stein variational gradient descent: Many-particle and long-time
  asymptotics.
\newblock \emph{Foundations of Data Science}, 5\penalty0 (3):\penalty0
  286--320, 2023.

\bibitem[Oates et~al.(2016)Oates, Papamarkou, and
  Girolami]{oates2016controlled}
C.~J. Oates, T.~Papamarkou, and M.~Girolami.
\newblock The controlled thermodynamic integral for {B}ayesian model evidence
  evaluation.
\newblock \emph{Journal of the American Statistical Association}, 111\penalty0
  (514):\penalty0 634--645, 2016.

\bibitem[Oates et~al.(2017)Oates, Girolami, and Chopin]{oates2014control}
C.~J. Oates, M.~Girolami, and N.~Chopin.
\newblock Control functionals for {M}onte {C}arlo integration.
\newblock \emph{Journal of the Royal Statistical Society Series B: Statistical
  Methodology}, 79\penalty0 (3):\penalty0 695--718, 2017.

\bibitem[Pardoux and Veretennikov(2001)]{pardoux2001poisson}
{\'E}.~Pardoux and Y.~Veretennikov.
\newblock On the {P}oisson equation and diffusion approximation. {I}.
\newblock \emph{The Annals of Probability}, 29\penalty0 (3):\penalty0
  1061--1085, 2001.

\bibitem[Pathiraja et~al.(2021)Pathiraja, Reich, and
  Stannat]{pathiraja2021mckean}
S.~Pathiraja, S.~Reich, and W.~Stannat.
\newblock {McKean}--{V}lasov {SDEs} in nonlinear filtering.
\newblock \emph{SIAM Journal on Control and Optimization}, 59\penalty0
  (6):\penalty0 4188--4215, 2021.

\bibitem[Pham(2009)]{pham2009continuous}
H.~Pham.
\newblock \emph{Continuous-time stochastic control and optimization with
  financial applications}, volume~61.
\newblock Springer Science \& Business Media, 2009.

\bibitem[Radhakrishnan and Meyn(2019)]{radhakrishnan2019gain}
A.~Radhakrishnan and S.~Meyn.
\newblock Gain function tracking in the feedback particle filter.
\newblock In \emph{2019 American Control Conference (ACC)}, pages 5352--5359.
  IEEE, 2019.

\bibitem[Rahimi et~al.(2007)Rahimi, Recht, et~al.]{rahimi2007random}
A.~Rahimi, B.~Recht, et~al.
\newblock Random features for large-scale kernel machines.
\newblock In \emph{NIPS}, volume~3, page~5. Citeseer, 2007.

\bibitem[Raskutti and Mukherjee(2015)]{raskutti2015information}
G.~Raskutti and S.~Mukherjee.
\newblock The information geometry of mirror descent.
\newblock \emph{IEEE Transactions on Information Theory}, 61\penalty0
  (3):\penalty0 1451--1457, 2015.

\bibitem[R{\"a}tsch et~al.(2001)R{\"a}tsch, Onoda, and
  M{\"u}ller]{ratsch2001soft}
G.~R{\"a}tsch, T.~Onoda, and K.-R. M{\"u}ller.
\newblock Soft margins for adaboost.
\newblock \emph{Machine learning}, 42:\penalty0 287--320, 2001.

\bibitem[Reich(2011)]{reich2011dynamical}
S.~Reich.
\newblock A dynamical systems framework for intermittent data assimilation.
\newblock \emph{BIT Numerical Mathematics}, 51\penalty0 (1):\penalty0 235--249,
  2011.

\bibitem[Reich(2012)]{reich2012gaussian}
S.~Reich.
\newblock A {G}aussian-mixture ensemble transform filter.
\newblock \emph{Quarterly Journal of the Royal Meteorological Society},
  138\penalty0 (662):\penalty0 222--233, 2012.

\bibitem[Reich(2022)]{reich2022data}
S.~Reich.
\newblock Data assimilation: A dynamic homotopy-based coupling approach.
\newblock In \emph{Stochastic Transport in Upper Ocean Dynamics Annual
  Workshop}, pages 261--280. Springer Nature Switzerland Cham, 2022.

\bibitem[Richter et~al.(2023)Richter, Sallandt, and
  N{\"u}sken]{richter2023continuous}
L.~Richter, L.~Sallandt, and N.~N{\"u}sken.
\newblock From continuous-time formulations to discretization schemes: tensor
  trains and robust regression for {BSDEs} and parabolic {PDEs}.
\newblock \emph{arXiv:2307.15496}, 2023.

\bibitem[Robert and Casella(2013)]{robert2013monte}
C.~Robert and G.~Casella.
\newblock \emph{Monte Carlo statistical methods}.
\newblock Springer Science \& Business Media, 2013.

\bibitem[Rudi et~al.(2017)Rudi, Carratino, and Rosasco]{rudi2017falkon}
A.~Rudi, L.~Carratino, and L.~Rosasco.
\newblock Falkon: An optimal large scale kernel method.
\newblock In \emph{Advances in Neural Information Processing Systems
  (NeurIPS)}, 2017.

\bibitem[Santambrogio(2015)]{santambrogio2015optimal}
F.~Santambrogio.
\newblock Optimal transport for applied mathematicians.
\newblock \emph{Birk{\"a}user, NY}, 55\penalty0 (58-63):\penalty0 94, 2015.

\bibitem[Smale and Zhou(2007)]{smale2007learning}
S.~Smale and D.-X. Zhou.
\newblock Learning theory estimates via integral operators and their
  approximations.
\newblock \emph{Constructive approximation}, 26\penalty0 (2):\penalty0
  153--172, 2007.

\bibitem[Smith(2013)]{smith2013sequential}
A.~Smith.
\newblock \emph{Sequential Monte Carlo methods in practice}.
\newblock Springer Science \& Business Media, 2013.

\bibitem[Smola et~al.(2007)Smola, Gretton, Song, and
  Sch{\"o}lkopf]{smola2007hilbert}
A.~Smola, A.~Gretton, L.~Song, and B.~Sch{\"o}lkopf.
\newblock A {H}ilbert space embedding for distributions.
\newblock In \emph{International conference on algorithmic learning theory},
  pages 13--31. Springer, 2007.

\bibitem[Smola and Sch{\"o}lkopf(1998)]{smola1998learning}
A.~J. Smola and B.~Sch{\"o}lkopf.
\newblock \emph{Learning with kernels}, volume~4.
\newblock Citeseer, 1998.

\bibitem[Song et~al.(2021)Song, Sohl{-}Dickstein, Kingma, Kumar, Ermon, and
  Poole]{song2021score}
Y.~Song, J.~Sohl{-}Dickstein, D.~P. Kingma, A.~Kumar, S.~Ermon, and B.~Poole.
\newblock Score-based generative modeling through stochastic differential
  equations.
\newblock In \emph{9th International Conference on Learning Representations,
  {ICLR}}, 2021.

\bibitem[Sriperumbudur et~al.(2011)Sriperumbudur, Fukumizu, and
  Lanckriet]{sriperumbudur2011universality}
B.~K. Sriperumbudur, K.~Fukumizu, and G.~R. Lanckriet.
\newblock Universality, characteristic kernels and {RKHS} embedding of
  measures.
\newblock \emph{Journal of Machine Learning Research}, 12\penalty0 (7), 2011.

\bibitem[Steinwart and Christmann(2008)]{steinwart2008support}
I.~Steinwart and A.~Christmann.
\newblock \emph{Support vector machines}.
\newblock Springer Science \& Business Media, 2008.

\bibitem[Syed et~al.(2021)Syed, Romaniello, Campbell, and
  Bouchard-C{\^o}t{\'e}]{syed2021parallel}
S.~Syed, V.~Romaniello, T.~Campbell, and A.~Bouchard-C{\^o}t{\'e}.
\newblock Parallel tempering on optimized paths.
\newblock In \emph{International Conference on Machine Learning (ICML)}, pages
  10033--10042. PMLR, 2021.

\bibitem[Taghvaei and Mehta(2016)]{taghvaei2016gain}
A.~Taghvaei and P.~G. Mehta.
\newblock Gain function approximation in the feedback particle filter.
\newblock In \emph{2016 IEEE 55th Conference on Decision and Control (CDC)},
  pages 5446--5452. IEEE, 2016.

\bibitem[Taghvaei et~al.(2020)Taghvaei, Mehta, and Meyn]{taghvaei2020diffusion}
A.~Taghvaei, P.~G. Mehta, and S.~P. Meyn.
\newblock Diffusion map-based algorithm for gain function approximation in the
  feedback particle filter.
\newblock \emph{SIAM/ASA Journal on Uncertainty Quantification}, 8\penalty0
  (3):\penalty0 1090--1117, 2020.

\bibitem[Tian et~al.(2024)Tian, Panda, and Lin]{tian2024liouville}
Y.~Tian, N.~Panda, and Y.~T. Lin.
\newblock Liouville flow importance sampler.
\newblock \emph{arXiv:2405.06672}, 2024.

\bibitem[Trillos et~al.(2023)Trillos, Hosseini, and
  Sanz-Alonso]{trillos2023optimization}
N.~G. Trillos, B.~Hosseini, and D.~Sanz-Alonso.
\newblock From optimization to sampling through gradient flows.
\newblock \emph{Notices of the American Mathematical Society}, 70\penalty0 (6),
  2023.

\bibitem[Vargas et~al.(2024)Vargas, Padhy, Blessing, and
  N\"usken]{vargas2024transport}
F.~Vargas, S.~Padhy, D.~Blessing, and N.~N\"usken.
\newblock Transport meets variational inference: Controlled {M}onte {C}arlo
  diffusions.
\newblock In \emph{The Twelfth International Conference on Learning
  Representations: ICLR 2024}, 2024.

\bibitem[Villani(2003)]{villani2003topics}
C.~Villani.
\newblock \emph{Topics in optimal transportation}.
\newblock American Mathematical Soc., 2003.

\bibitem[Villani(2008)]{villani2008optimal}
C.~Villani.
\newblock \emph{Optimal transport: old and new}, volume 338.
\newblock Springer Science \& Business Media, 2008.

\bibitem[Wahba and Wang(2019)]{wahba}
G.~Wahba and Y.~Wang.
\newblock \emph{Representer Theorem}, pages 1--11.
\newblock American Cancer Society, 2019.
\newblock ISBN 9781118445112.
\newblock \doi{https://doi.org/10.1002/9781118445112.stat08200}.

\bibitem[Wang and N{\"u}sken(2024)]{wang2024measure}
L.~Wang and N.~N{\"u}sken.
\newblock Measure transport with kernel mean embeddings.
\newblock \emph{arXiv:2401.12967}, 2024.

\bibitem[Wang and Li(2020)]{wang2020information}
Y.~Wang and W.~Li.
\newblock Information {N}ewton's flow: second-order optimization method in
  probability space.
\newblock \emph{arXiv:2001.04341}, 2020.

\bibitem[Weidmann(2012)]{weidmann2012linear}
J.~Weidmann.
\newblock \emph{Linear operators in Hilbert spaces}, volume~68.
\newblock Springer Science \& Business Media, 2012.

\bibitem[Wendland(2004)]{wendland2004scattered}
H.~Wendland.
\newblock \emph{Scattered data approximation}, volume~17.
\newblock Cambridge university press, 2004.

\bibitem[Yan et~al.(2023)Yan, Wang, and Rigollet]{yan2023learning}
Y.~Yan, K.~Wang, and P.~Rigollet.
\newblock Learning {G}aussian mixtures using the {W}asserstein-{F}isher-{R}ao
  gradient flow.
\newblock \emph{arXiv:2301.01766}, 2023.

\bibitem[Yang et~al.(2023)Yang, Zhang, Song, Hong, Xu, Zhao, Zhang, Cui, and
  Yang]{yang2023diffusion}
L.~Yang, Z.~Zhang, Y.~Song, S.~Hong, R.~Xu, Y.~Zhao, W.~Zhang, B.~Cui, and
  M.-H. Yang.
\newblock Diffusion models: A comprehensive survey of methods and applications.
\newblock \emph{ACM Computing Surveys}, 56\penalty0 (4):\penalty0 1--39, 2023.

\bibitem[Zhou(2008)]{zhou2008derivative}
D.-X. Zhou.
\newblock Derivative reproducing properties for kernel methods in learning
  theory.
\newblock \emph{Journal of computational and Applied Mathematics}, 220\penalty0
  (1-2):\penalty0 456--463, 2008.

\bibitem[Zhu and Mielke(2024)]{zhu2024approximation}
J.-J. Zhu and A.~Mielke.
\newblock Approximation, kernelization, and entropy-dissipation of gradient
  flows: from {W}asserstein to {F}isher-{R}ao.
\newblock \emph{preprint}, 2024.
\newblock URL \url{https://jj-zhu.github.io/file/ZhuMielke24AppKerEntFR.pdf}.

\bibitem[Zhuo et~al.(2018)Zhuo, Liu, Shi, Zhu, Chen, and
  Zhang]{zhuo2018message}
J.~Zhuo, C.~Liu, J.~Shi, J.~Zhu, N.~Chen, and B.~Zhang.
\newblock Message passing {S}tein variational gradient descent.
\newblock In \emph{International Conference on Machine Learning (ICML)}, pages
  6018--6027. PMLR, 2018.

\end{thebibliography}
\bibliographystyle{abbrvnat}  

\appendix

\section{Proofs for Sections \ref{sec:homotopies} and \ref{sec:KRR}}
\label{appsec:proofs}

\subsection{Proof of Proposition \ref{prop:KRR}}
\label{appsec:KRR_proof}

We present two proofs of Proposition \ref{prop:KRR}, mainly because the notation and set up will be valuable for the proofs in Section \ref{app:sec4 proofs}.

\subsubsection{Proof via the representer theorem}

The first proof of Proposition \ref{prop:KRR} relies on the following version of the representer theorem, here recalled for convenience.

\begin{thm}[Representer theorem, \cite{wahba}]
\label{thm:representer}
Let $(H,\langle \cdot, \cdot \rangle_H)$ be a Hilbert space over $\mathbb{R}$ and denote its continuous dual by $(H',\langle \cdot, \cdot \rangle_{H'})$. Let $U' = \left\{ u_1', \ldots, u'_N\right\} \subset H'$ be a collection of continuous linear functionals on $H$.
Denote the set of associated Riesz representers by $U = \{u_1, \ldots, u_N \} \subset H$, that is, we have that
\begin{equation*}
u_j'(v) = \langle u_j, v \rangle_H,    
\end{equation*}
for all $v \in H$ and $j = 1,\ldots,N$.
Furthermore, let $\{y_1, \ldots, y_N\} \subset \mathbb{R}$ be a collection of real numbers, $\lambda > 0$ a regularisation parameter, and consider the regression problem
\begin{equation}
\label{eq:abstract regression}
v^* \in \argmin_{v \in H} \left( \frac{1}{N} \sum_{j=1}^N \left( u'_j(v) - y_j \right)^2 + \lambda \Vert v \Vert_{H}^2 \right).
\end{equation}
Then, \eqref{eq:abstract regression} admits a unique solution $v^*$. Moreover, $v^*$ belongs to the linear span of $U$, that is, 
\begin{equation*}
v^* = \sum_{i=1}^N \phi_i u_i,
\end{equation*}
for appropriate coefficients $\phi_i \in \mathbb{R}$. The coefficient vector $(\phi_i)_{i=1}^N = \phi \in \mathbb{R}^N$ can be obtained as the unique solution to the linear system
\begin{equation*}
(\tfrac{1}{N}\bm{\xi} + \lambda I_{N \times N}) \phi = \bm{y},
\end{equation*}
where $\bm{y} = (y_1,\ldots,y_N)^\top \in \mathbb{R}^N$, and the matrix $\bm{\xi} \in \mathbb{R}^{N \times N}$ is given by $\bm{\xi}_{ij} = \langle u_i, u_j \rangle_H$.
\end{thm}

\begin{proof}[Proof of Proposition \ref{prop:KRR}]
For $j=1,\ldots,N$, let us define the functionals $u'_j:\mathcal{H}^d_k \rightarrow \mathbb{R}$ via
\begin{equation*}
u'_j(v) := (S_\pi v)(X^j), \qquad \qquad v \in \mathcal{H}^d_k.
\end{equation*}
Clearly, the functionals $u'_j$ are linear, since $S_\pi$ is a linear operator. Moreover, since
$k \in C^{1,1}(\mathbb{R}^d \times \mathbb{R}^d; \mathbb{R})$, we have that the divergence operators 
\begin{equation*}
\begin{array}{cccc}
\nabla \cdot \Big\vert_{X^j}: & \mathcal{H}^d_k & \rightarrow & \mathbb{R},  \\
  & v & \mapsto & ( \nabla \cdot v)(X^j) = \sum_{l=1}^d \frac{\partial v^l(X^j)}{\partial x^l}
\end{array}
\end{equation*}
are continuous, see \citet[Corollary 4.36]{steinwart2008support}.
Consequently, the functionals $u_j'$ are continuous. We now claim that the corresponding Riesz representers are given by
\begin{equation}
\label{eq:Riesz representers}
u_j = k(\cdot,X^j) \nabla \log \pi(X^j) +  \nabla_{X^j} k(\cdot,X^j) \in \mathcal{H}^d_k.
\end{equation}
Indeed, by the reproducing property, as well as the derivative reproducing property \citep[Theorem 1]{zhou2008derivative} we have that
\begin{equation}
\label{eq:uiuj}
\langle u_j, v \rangle_{\mathcal{H}^d_k} = \nabla \log \pi(X^j) \cdot v(X^j) + (\nabla \cdot v)(X^j) = (S_\pi v)(X^j) = u_j'(v),
\end{equation}
for all $v \in \mathcal{H}^d_k$ and $j = 1,\ldots,N$, as required. A direct computation using the reproducing and derivative reproducing properties now shows that
\begin{equation*}
\langle u_i, u_j \rangle_{\mathcal{H}_k^d} = (\bm{\xi}^{k, \nabla \log \pi})_{ij}, \qquad \qquad i,j = 1, \ldots, N,
\end{equation*}
where the right-hand side is defined according to \eqref{eq:xi} and \eqref{eq:matrices}. The claim now follows from Theorem \ref{thm:representer} together with \eqref{eq:Riesz representers} and \eqref{eq:uiuj}.
\end{proof}

\subsubsection{Proof via the Tikhonov regression formula}
\label{app:Tikhonov}

Here, we show that Problem \ref{prob:KRR} is a specific instance of Tikhonov-regularised least-squares problems. For fixed particle positions $X_1,\ldots, X_N \in \mathbb{R}^d$, we start by equipping $\mathbb{R}^N$ with the inner product
\begin{equation}
\label{eq:N inner product}
\langle x,y \rangle_N := \frac{1}{N} \sum_{i=1}^N x_i y_i, \qquad \qquad x,y \in \mathbb{R}^N, 
\end{equation}
and define the linear operator $S_{\pi,N} : (\mathcal{H}_k^d, \langle \cdot, \cdot  \rangle_{\mathcal{H}_k^d}) \rightarrow (\mathbb{R}^N, \langle \cdot, \cdot \rangle_N)$ via
\begin{equation}
\label{eq:sample operator}
(S_{\pi,N} v)_i = (S_\pi v)(X^i), \qquad i=1,\ldots, N, \qquad v \in \mathcal{H}_k^d.
\end{equation}
The objective in \eqref{eq:KRR} can now be rewritten in the form
\begin{equation}
\label{eq:abstract Tikhonov}
v^* \in \argmin_{v \in \mathcal{H}_k^d} \left( \Vert S_{\pi,N} v - \bm{h}_0 \Vert_N^2 + \lambda \Vert v \Vert^2_{\mathcal{H}_k^2}  \right),
\end{equation}
using the norm $\Vert x \Vert^2 := \langle x,x\rangle_N$ in $\mathbb{R}^N.$

We now note that the objective in \eqref{eq:abstract Tikhonov} coincides with the \emph{Tikhonov functional} in \citet[equation (2.12)]{kirsch2021introduction} for the choices $X = (\mathcal{H}_k^d, \langle \cdot, \cdot  \rangle_{\mathcal{H}_k^d})$, $Y = (\mathbb{R}^N, \langle \cdot, \cdot \rangle_N)$, $K = S_{\rho,N}$, $\alpha = \lambda$ and $y = \bm{h}_0$. By \citet[Theorem 2.11]{kirsch2021introduction},   $v_{N,\lambda}$ can be written in the form
\begin{subequations}
\label{eq:Tikhonov formula}
\begin{align}
\label{eq:vNeps}
v_{N,\lambda} & = (\lambda I_{\mathcal{H}_k^d} + S^*_{\pi,N}S_{\pi,N})^{-1} S^*_{\pi,N}\bm{h}_0
\\
\label{eq:Tikhonov2}
& = S^*_{\pi,N} (\lambda I_{N \times N} + S_{\pi,N}S^*_{\pi,N})^{-1}\bm{h}_0,
\end{align}
\end{subequations}
where $I_{\mathcal{H}_k^d}$ denotes the identity operator on $\mathcal{H}_k^d$.

To apply the Tikhonov formulas \eqref{eq:Tikhonov formula}, we seek an expression for the adjoint operator ${S^*_{\pi,N}:(\mathbb{R}^N,\langle \cdot, \cdot \rangle_N) \rightarrow (\mathcal{H}_k^d,\langle \cdot, \cdot \rangle_{\mathcal{H}_k^d})}$,
characterised by
\begin{equation*}
\langle c, S_{\pi,N} v \rangle_{N} = \langle S^*_{\pi,N} c, v \rangle_{\mathcal{H}_k^d}, \qquad \text{for all} \, v \in \mathcal{H}_k^d, \,\, c \in \mathbb{R}^N.
\end{equation*}
A direct calculation using the reproducing and derivative reproducing properties \cite[Theorem 1]{zhou2008derivative} of $k$ shows that
\begin{equation}
\label{eq:SNstar}
S^*_{\pi,N} c = \frac{1}{N} \sum_{j=1}^N \left( \nabla _{X^j} k(\cdot, X^j) + k(\cdot, X^j) \nabla \log \pi (X^j)\right) c_j, \qquad c \in \mathbb{R}^N. 
\end{equation}
From this, we directly obtain that
\begin{equation}
\label{eq:SNSNstar}
 (S_{\pi,N} S^*_{\pi,N} c)_i =  \frac{1}{N} \sum_{j=1}^N \bm{\xi}_{ij} c_j.  
\end{equation}
Combining \eqref{eq:SNSNstar} with \eqref{eq:SNSNstar}, we see that \eqref{eq:Tikhonov2} coincides with the representation for $v^*$ given in Proposition \ref{prob:KRR}.
\begin{remark}[Weighted kernel ridge regression]
\label{rem:weighted KRR}
For the construction in Section \ref{sec:weighted dynamics}, it is crucial to slightly generalise the formulation in Problem \ref{prob:KRR}, 
\begin{equation}
\label{eq:weighted KRR}
    v^* \in \argmin_{v \in \mathcal{H}^d_k} \left( \sum_{j=1}^N w^j \left( (S_\pi v)(X^j) - h_0(X^j) \right)^2  + \lambda \Vert v \Vert_{\mathcal{H}^d_k}^2 \right),
\end{equation}
replacing $1/N$ by the particle weights $w^j$; the motivation is to allow approximations of the form $\pi \approx \sum_{i=1}^N w^i \delta_{X^i}$. The Tikhonov regression approach can be adapted without difficulties: Instead of \eqref{eq:N inner product}, we define the weight-dependent inner product
\begin{equation*}
\langle x,y \rangle_w :=  \sum_{i=1}^N w_i x_i y_i, \qquad \qquad x,y \in \mathbb{R}^N.
\end{equation*}
Proceeding analogously, we arrive at modifications of \eqref{eq:SNstar} and \eqref{eq:SNSNstar}, with $1/N$ replaced by $w^j$. The unique solution to \eqref{eq:weighted KRR} is therefore given by
\begin{equation}
\label{eq:v weights}
v^* =\sum_{j=1}^N w^j \phi^j\left( k(\cdot,X^j)\nabla \log \pi(X^j) + \nabla_{X^j} k(\cdot,X^j)\right),
\end{equation}
with $(\phi^j)_{j=1}^N$ determined from 
\begin{equation}
\label{eq:phi weights}
\sum_{j=1}^N(\bm{\xi}^{k,\nabla \log \pi})_{ij} w^j \phi^j +  \lambda \phi^i = h(X^i) - \sum_{j=1}^N w^j h(X^j), \qquad i =1, \ldots, N.
\end{equation}
Clearly, \eqref{eq:v weights} and \eqref{eq:phi weights} generalise \eqref{eq:v star} and \eqref{eq:linear system}.
\end{remark}

\subsection{Miscellaneous proofs}

\label{sec:misc}

\begin{proof}[Proof of Proposition \ref{prop:Stein eq}]
We have
\begin{equation}
\label{eq:dt rho}
\partial_t \pi_t = - \frac{h e^{-ht} \pi_0}{Z_t} - \frac{e^{-ht}\pi_0}{Z_t^2} \partial_t Z_t = -h \pi_t - \pi_t \frac{\partial_t Z_t}{Z_t}.
\end{equation}
Moreover,
\begin{equation}
\label{eq:dt Z}
\frac{1}{Z_t}\partial_t Z_t = 
\frac{1}{Z_t}\partial_t \left( \int_{\mathbb{R}^d}
e^{-ht} \,\mathrm{d}\pi_0
\right) = -\frac{1}{Z_t} \int_{\mathbb{R}^d} h e^{-ht} \, \mathrm{d}\pi_0 = -\int_{\mathbb{R}^d} h \, \mathrm{d}\pi_t,
\end{equation}
where the exchange of differentiation and integration is permissible since $he^{-h}$ is bounded, by the assumption that $h$ is bounded from below. Combining \eqref{eq:dt rho} and \eqref{eq:dt Z}, we see that the interpolation \eqref{eq:homotopy} satisfies
\begin{equation}
\label{eq:tempering dynamics}
\partial_t \pi_t = -\pi_t \left(h - \int_{\mathbb{R}^d} h \, \mathrm{d}\pi_t \right). 
\end{equation}
We now argue that the law associated to the ODE \eqref{eq:ODE} is governed by the same equation (hence proving the claim by the well-posedness assumption). Indeed, $\text{Law}(X_t)$
satisfies the continuity equation $\partial_t \pi_t + \nabla \cdot(\pi_t v_t) = 0$, see \cite[Section 4.1.2]{santambrogio2015optimal}. We now see that 
\begin{equation*}
\nabla \cdot(\pi_t v_t) = \rho_t S_{\pi_t} v_t = \pi_t \left(h - \int_{\mathbb{R}^d} h \, \mathrm{d}\pi_t \right),    
\end{equation*}
by \eqref{eq:Stein eq}, completing the proof.
\end{proof}

\begin{proof}[Proof of Lemma \ref{lem:score}]
First, we have that 
\begin{equation}
\label{eq:dtnablalogrho}
\frac{\mathrm{d}}{\mathrm{d}t} \left(\nabla \log \rho_t (X_t)\right) = \left(\nabla \partial_t \log \rho_t\right)(X_t) + \left(\mathrm{Hess} \log \rho_t(X_t) \right) \frac{\mathrm{d}X_t}{\mathrm{d}t}.  
\end{equation}
From $\partial_t \rho_t + \nabla \cdot (\rho_t v_t) = 0$, see \citet{santambrogio2015optimal}, it follows that 
\begin{equation}
\label{eq:dtlogrho}
\partial_t \log \rho_t = \frac{\partial_t \rho_t}{\rho_t} = - \frac{\nabla \cdot (\rho_t v_t)}{\rho_t} = - \nabla \cdot v_t - v_t \cdot \nabla \log \rho_t.
\end{equation}
Plugging \eqref{eq:dtlogrho} into \eqref{eq:dtnablalogrho} and noticing that 
\begin{equation*}
\nabla \left( v_t \cdot \nabla \log \rho_t \right)  = (\nabla v_t) (\nabla \log \rho_t)(X_t) + \left(\mathrm{Hess} \log \rho_t(X_t) \right) v_t(X_t)
\end{equation*}
leads to the claimed identity. 
\end{proof}

\section{Proofs for Section \ref{sec:ot}}
\label{app:sec4 proofs}

The objective of this section is to prove Propositions \ref{prop:KRR mean field} and \ref{prop:projection} as well as Theorems \ref{thm:connections} and \ref{thm:consistency}. We begin with the following simple estimate on the KSD-kernel $\xi^{k,\nabla \log \pi}$:  
\begin{lemma}[Estimate on $\xi^{k,\nabla \log \pi}$]
\label{lem:estimates}
Under Assumption \ref{ass:basic}, there exists a constant $C>0$ such that
\begin{equation*}
|\xi^{k,\nabla \log \pi}(x,y)| \le C(|\nabla \log \pi(x)| + |\nabla \log \pi(y)| + |\nabla \log \pi(x)||\nabla \log \pi(y)|),
\end{equation*}
for all $x,y \in \mathbb{R}^d$.
\end{lemma}
\begin{proof}
This follows directly from the Cauchy-Schwarz and triangle inequalities, as well as from the boundedness of $k$ and its derivatives.
\end{proof}
The proofs of Theorems \ref{thm:connections} and \ref{thm:consistency} rely Tikhonov formulae of the form \eqref{eq:Tikhonov formula}. The following lemma collects basic properties of the relevant operators.

\begin{lemma}
\label{lem:operators}
Assume that $\Vert \nabla 
\log \pi \Vert_{(L^2(\pi))^d} < \infty$ and that $k$ is bounded, with bounded first-order derivatives. Then the following hold:
\begin{enumerate}
\item
\label{it:S bounded}
The Stein operator $S_\pi$ is bounded from $\mathcal{H}_k^d$ to $L^2(\pi)$, that is, there exists a constant $C>0$ such that
\begin{equation*}
\Vert S_\pi v \Vert_{L^2(\pi)} \le C \Vert v \Vert_{\mathcal{H}_k^d},  
\end{equation*}
for all $v \in \mathcal{H}_k^d$.
\item
\label{it:T k rho nabla}
There exists a constant $C>0$ such that
\begin{equation}
\label{eq:T k rho nabla bound}
\Vert \mathcal{T}_{k,\pi}\nabla \phi \Vert_{\mathcal{H}_k^d} \le C \Vert \phi \Vert_{L^2(\pi)},
\end{equation}
    for all $\phi \in C_c^{\infty}(\mathbb{R}^d)$. Therefore, there is a unique extension of $\mathcal{T}_{k,\pi}\nabla$ to a bounded linear operator from $L^2(\pi)$ to $\mathcal{H}_k^d$ that we denote by the same symbol.
    \item
    \label{it:S rho adjoint}
    The adjoint of $S_\pi : \mathcal{H}_k^d \rightarrow L^2(\pi)$ is given by $-\mathcal{T}_{k,\pi}\nabla : L^2(\pi) \rightarrow \mathcal{H}_k^d$, that is,
\begin{equation*}
\langle S_\pi v, \phi \rangle_{L^2(\pi)} =  - \langle v, \mathcal{T}_{k,\pi} \nabla \phi \rangle_{\mathcal{H}_k^d}, 
\end{equation*}
    for all $v \in \mathcal{H}_k^d$ and $\phi \in L^2(\pi)$.
\end{enumerate}
\end{lemma}
\begin{proof}
\ref{it:S bounded}.) There exists a constant $C>0$ such that
\begin{equation}
\label{eq:S bounded}
\Vert S_\pi v \Vert_{L^2(\pi)} \le \Vert \nabla \log \pi \Vert_{(L^2(\rho))^d}    \Vert v \Vert_{(L^2(\rho))^d} + \Vert \nabla \cdot v \Vert_{L^2(\pi)} \le C \Vert v \Vert_{\mathcal{H}_k^d},
\end{equation}
for all $v \in \mathcal{H}_k^d$. The first inequality in \eqref{eq:S bounded} is implied by the triangle and Cauchy-Schwarz inequalities, while the second inequality follows from the regularity and boundedness assumptions on $k$, see \citet[Theorem 4.26 and Corollary 4.36]{steinwart2008support}.

\ref{it:T k rho nabla}.)
For $\phi \in C_c^\infty(\mathbb{R}^d)$, we have that
\begin{subequations}
\begin{align}
\label{eq:T xi1}
\Vert \mathcal{T}_{k,\pi} \nabla \phi \Vert_{\mathcal{H}_k^d}^2 &= \left\langle \int_{\mathbb{R}^d} k(\cdot,y) \nabla \phi(y) \pi(\mathrm{d}y),\int_{\mathbb{R}^d} k(\cdot,z) \nabla \phi(z) \pi(\mathrm{d}z) \right\rangle_{\mathcal{H}_k^d} 
\\
\label{eq:T xi2}
& = \int_{\mathbb{R}^d} \int_{\mathbb{R}^d} \nabla \phi(y) \cdot k(y,z) \nabla \phi(z) \pi(\mathrm{d}y) \pi(\mathrm{d}z) = \int_{\mathbb{R}^d} \int_{\mathbb{R}^d} \xi^{k,\nabla \log \pi}(y,z) \phi(y) \phi(z) \pi(\mathrm{d}y) \pi(\mathrm{d}z).
\end{align}
\end{subequations}
From \eqref{eq:T xi1} to \eqref{eq:T xi2} we have used the fact that Bochner integration commutes with bounded linear operators \citep[equation (A.32)]{steinwart2008support} to change the order of integration and inner products and apply the reproducing property. The second identity in \eqref{eq:T xi2} follows from integration by parts. We now obtain the bound \eqref{eq:T k rho nabla bound} from Lemma \ref{lem:estimates} and the fact that $\Vert \nabla \log \pi \Vert_{(L^2(\pi))^d} < \infty$ by assumption.

\item
\ref{it:S rho adjoint}.) We have
\begin{equation}
\label{eq:S rho adjoint}
-\langle \mathcal{T}_{k,\pi} \nabla \phi, v \rangle_{\mathcal{H}_k^d} = - \int_{\mathbb{R}^d}  \langle k(\cdot,y) \nabla \phi(y),v\rangle_{\mathcal{H}_k^d} \rho(\mathrm{d}y) = -\int_{\mathbb{R}^d} \nabla \phi \cdot v \, \mathrm{d}\pi = \langle \phi, S_\pi v \rangle_{L^2(\pi)},
\end{equation}
as required, for all $v \in \mathcal{H}_k^d$ and $\phi \in L^2(\pi)$. As in the proof of the second statement, we have made use of the fact that Bochner integration and bounded linear operators commute.
\end{proof}
Proposition \ref{prop:KRR mean field} can now be obtained from a Tikhonov-regularised least-squares formulation (cf. the proof of Proposition \ref{prop:KRR} in Appendix \ref{app:Tikhonov}):
\begin{proof}[Proof of Proposition \ref{prop:KRR mean field}]
As in the proof of Proposition \ref{prop:KRR}, we can reformulate \eqref{eq:KRR mean field} as
\begin{equation*}
v^*_\infty \in \argmin_{v \in \mathcal{H}_k^d} \left( \Vert S_\pi v - h_{0,\infty} \Vert^2_{L^2(\pi)} + \lambda \Vert v \Vert^2_{\mathcal{H}_k^d} \right).
\end{equation*}
Building on Lemma \ref{lem:operators} and \citet[Theorem 2.11]{kirsch2021introduction}, there exists a unique minimiser, given by
\begin{subequations}
\label{eq:Tikhonov mean field}
\begin{align}
\label{eq:Tikhonov mean field 1}
v^*_{\infty} & = S_\pi^*(\lambda I_{L^2(\pi)}  + S_\pi S_\pi^*)^{-1}h_{0,\infty}
\\
& = (\lambda I_{\mathcal{H}_k^d}  + S_\pi^* S_\pi)^{-1} S_\pi^* h_{0,\infty}.
\end{align}
\end{subequations}
The claim now follows from $S_\pi^* = - \mathcal{T}_{k,\pi} \nabla$, see Lemma \ref{lem:operators}, and from the fact that the equation $$(S_\pi S_\pi^* + \lambda I_{L^2(\pi)}) \phi = h_{0,\infty}$$ can be written in the form \eqref{eq:Stein Poisson}, multiplying both sides by $\pi$. 
\end{proof}

\begin{proof}[Proof of Proposition \ref{prop:projection}]
Using the Stein equation $S_\pi v_{\infty} = h - \int_{\mathbb{R}^d} h \, \mathrm{d}\pi$, the fact that $\tfrac{1}{N}\sum_{i=1}^N = \int_{\mathbb{R}^d} h \, \mathrm{d}\pi$, and \eqref{eq:Tikhonov2}, we can write
\begin{equation*}
v_{N,0} = S^*_{\pi,N} ( S_{\pi,N}S^*_{\pi,N})^{-1} S_{\pi,N} v_{\infty}.
\end{equation*}
Notice that the operator $S_{\pi,N}S^*_{\pi,N}:\mathbb{R}^N \rightarrow \mathbb{R}^N$ is invertible since $\xi \in \mathbb{R}^{N \times N}$ is invertible, see equation    \eqref{eq:SNSNstar}. The operator $P_{\bm{X}} := S^*_{\pi,N} ( S_{\pi,N}S^*_{\pi,N})^{-1} S_{\pi,N}$ is an orthogonal projection onto the subspace of $\mathcal{H}_k^d$ defined in \eqref{eq:proj subspace}. Indeed, it is immediate that $P_{\bm{X}}$ is self-adjoint (and positive definite), and that $P_{\bm{X}}^2 = P_{\bm{X}}$. We also have $\mathrm{Ran} P_{\bm{X}} = \mathrm{Ran} S_{\pi,N}^*$ which coincides with the subspace \eqref{eq:proj subspace}, see equation \eqref{eq:SNstar}.
\end{proof}

Before proceeding to the proof of Theorem \ref{thm:connections}, we need the following technical lemma, which is a slight extension of \citet[Theorem 7.12]{villani2003topics} to the context of Hilbert spaces and Bochner integrals.
\begin{lemma}
\label{lem:villani}
Let $H$ be a separable Hilbert space with corresponding norm $\Vert \cdot \Vert_H$. Let $\phi: \mathbb{R}^d \rightarrow H$ be Borel-measurable. Assume that there exist constants $C>0$ and $p>0$ such that
\begin{equation}
\label{eq:villani bound}
\Vert \phi(x) \Vert_H \le C(1 + |x|^p),
\end{equation}
for all $x \in \mathbb{R}^d$. Let $\mu_k \subset \mathcal{P}(\mathbb{R}^d)$ be a sequence of probability measures with finite $p^{\mathrm{th}}$ moments that converges in $W^p$ to some $\mu \in \mathcal{P}(\mathbb{R}^d)$. Then 
\begin{equation*}
\int_{\mathbb{R}^d} \phi\,\mathrm{d}\mu_k \rightarrow \int_{\mathbb{R}^d} \phi \, \mathrm{d}\mu \qquad
\end{equation*}
as Bochner integrals in $H$.
\end{lemma}
\begin{proof}
By \citet[Theorem 7.12]{villani2003topics}, convergence in $W^p$ implies
\begin{equation}
\label{eq:tightness}
\lim_{R \rightarrow \infty} \limsup_{k \rightarrow \infty} \int_{|x| \ge R} |x|^p \, \mathrm{d}\mu_k(x) =0.
\end{equation}
For arbitrary $R>0$, we have
\begin{subequations}
\begin{align}
\label{eq:bounded in ball}
\left\Vert \int_{\mathbb{R}^d} \phi \, \mathrm{d}\mu_k - \int_{\mathbb{R}^d} \phi \, \mathrm{d}\mu\right\Vert_H 
& \le \left\Vert \int_{|x| < R} \phi \, \mathrm{d}\mu_k - \int_{|x| < R} \phi \, \mathrm{d}\mu  \right\Vert_H  
\\
& + C \left( \int_{|x| \ge R} (1+ |x|^p) \, \mathrm{d}\mu_k + \int_{|x| \ge R} (1 + |x|^p ) \, \mathrm{d}\mu\right).
\end{align}
\end{subequations}

By \citet[Theorem 5.1]{hable2011qualitative}, the right-hand side of \eqref{eq:bounded in ball} vanishes in the limit as $k \rightarrow \infty$. Therefore, we see that 
\begin{equation}
\nonumber
\limsup_{k \rightarrow \infty} \left\Vert \int_{\mathbb{R}^d} \phi \, \mathrm{d}\mu_k - \int_{\mathbb{R}^d} \phi \, \mathrm{d}\mu\right\Vert_H \le C \limsup_{k \rightarrow \infty} \left( \int_{|x| \ge R} (1+ |x|^p) \, \mathrm{d}\mu_k + \int_{|x| \ge R} (1 + |x|^p ) \, \mathrm{d}\mu\right).
\end{equation}
The result now follows from \eqref{eq:tightness} by taking the limit as $R \rightarrow \infty$.
\end{proof}

\begin{proof}[Proof of Theorem \ref{thm:connections}] We follow a similar calculation by \citet[Section 3]{smale2007learning} and write
\begin{subequations}
\begin{align*}
v_{N,\lambda} - v_{\infty,\lambda} & = (\lambda I_{\mathcal{H}_k^d} + S^*_{\pi,N}S_{\pi,N})^{-1} \left( S_{\pi,N}^* h_{0,N} 
- (\lambda I_{\mathcal{H}_k^d} + S^*_{\pi,N}S_{\pi,N}) v_{\infty,\lambda}
\right)
\\
& 
=(\lambda I_{\mathcal{H}_k^d} + S^*_{\pi,N}S_{\pi,N})^{-1} \left( S_{\pi,N}^* (h_{0,N} - S_{\pi,N} v_{\infty,\lambda}) 
- \lambda v_{\infty,\lambda}
\right)
\\
& = (\lambda I_{\mathcal{H}_k^d} + S^*_{\pi,N}S_{\pi,N})^{-1} \left( S_{\pi,N}^* (h_{0,N} - S_{\pi,N} v_{\infty,\lambda}) 
- (S_\pi^* h_{0,\infty}- S_\pi^* S_\pi v_{\infty,\lambda})
\right)
\\
& = (\lambda I_{\mathcal{H}_k^d} + S^*_{\pi,N}S_{\pi,N})^{-1} \left(( S_\pi^* S_\pi - S_{\pi,N}^*S_{\pi,N}) v_{\infty,\lambda} 
+ S^*_{\pi,N} h_{0,N} - S_\pi^* h_{0,\infty}
\right),
\end{align*}
\end{subequations}
using the sample Stein operator $S_{\pi,N}$ defined in \eqref{eq:sample operator}, as well as the Tikhonov formulae \eqref{eq:Tikhonov formula} and \eqref{eq:Tikhonov mean field}. From this, we see that 
\begin{equation}
\label{eq:v estimate}
\Vert v_{N,\lambda} - v_{\infty,\lambda} \Vert_{\mathcal{H}_k^d} \le \frac{1}{\lambda} \left(
\Vert (S_\pi^* S_\pi - S_{\pi,N}^*S_{\pi,N}) v_{\infty,\lambda} \Vert_{\mathcal{H}_k^d} + \Vert S^*_{\pi,N} h_{0,N} - S_\pi^* h_{0,\infty}\Vert_{\mathcal{H}_k^d}
\right).
\end{equation}
To estimate the first term on the right-hand side of \eqref{eq:v estimate}, we introduce the notation $\pi^{(N)}:= \frac{1}{N}\sum_{i=1}^N \delta_{X^i}$ and compute 
\begin{subequations}
\begin{align*}
S^*_{\pi,N} S_{\pi,N} v_{\infty,\lambda} & = \frac{1}{N}\sum_{i=1}^N \left[ \left(\nabla_{X^i} k(\cdot,X^i) + k(\cdot,X^i) \nabla \log \pi(X^i) \right) (S_\pi v_{\infty,\lambda})(X^i)\right]
\\
& = \int_{\mathbb{R}^d}
\left(\nabla_{y} k(\cdot,y) + k(\cdot,y) \nabla \log \pi(y) \right) (S_\pi v_{\infty,\lambda})(y) \pi^{(N)}(\mathrm{d}y),
\end{align*}
\end{subequations}
using \eqref{eq:sample operator} and \eqref{eq:SNstar}. Similarly,
\begin{equation*}
S^*_{\pi} S_{\pi} v_{\infty,\lambda} 
 = \int_{\mathbb{R}^d}
\left(\nabla_{y} k(\cdot,y) + k(\cdot,y) \nabla \log \pi(y) \right) (S_\pi v_{\infty,\lambda})(y) \pi(\mathrm{d}y).
\end{equation*}
Consequently,
\begin{equation*} \left\Vert (S_\pi^* S_\pi - S_{\pi,N}^*S_{\pi,N}) v_{\infty,\lambda} \right\Vert^2_{\mathcal{H}_k^d} =
\left\Vert \int_{\mathbb{R}^d}
\left(\nabla_{y} k(\cdot,y) + k(\cdot,y) \nabla \log \pi(y) \right) (S_\pi v_{\infty,\lambda})(y) (\pi - \pi^{(N)})(\mathrm{d}y) \right\Vert^2_{\mathcal{H}_k^d}.
\end{equation*}
To show convergence of the left-hand side, by Lemma \ref{lem:villani} it is sufficient to show that the map ${\phi: \mathbb{R}^d \ni y \mapsto \left(\nabla_{y} k(\cdot,y) + k(\cdot,y) \nabla \log \pi(y) \right) (S_\pi v_{\infty,\lambda})(y) \in \mathcal{H}_k^d}$ satisfies the growth bound \eqref{eq:villani bound}. Indeed, we have  the following bound,
\begin{subequations}
\label{eq:kS bound}
\begin{align}
\left\Vert\left(\nabla_{y} k(\cdot,y) + k(\cdot,y) \nabla \log \pi(y) \right) (S_\pi v_{\infty,\lambda})(y) \right \Vert_{\mathcal{H}_k^d} \le | S_\pi v_{\infty,\lambda}(y)| \Vert \nabla_{y} k(\cdot,y) + k(\cdot,y) \nabla \log \pi(y) \Vert_{\mathcal{H}_k^d}  \\ = | S_\pi v_{\infty,\lambda}(y)| \sqrt{\xi(y,y)}
\le \widetilde{C}_1(\lambda) (1 + |\nabla \log \pi(y)|^2) \le \widetilde{C}_2(\lambda) (1+ |y|^p), 
\end{align}
\end{subequations}
for some constants $\widetilde{C}_1,\widetilde{C}_2 > 0$ that might depend on $\lambda$. Here, we have used the fact that $v_{\infty,\lambda}$ and its first derivatives are bounded (since $v_{\infty,\lambda} \in \mathcal{H}_k^d$ and $k$ and its first derivatives are bounded), as well as the growth assumption on $\nabla \log \pi$ and the estimate from Lemma \ref{lem:estimates}.

We conclude the proof by providing a similar estimate to show convergence of the second term on the right-hand side of \eqref{eq:v estimate}.
We have
\begin{subequations}
\begin{align}
\label{eq:Sh norm}
 & \Vert S^*_{\pi,N} h_{0,N} - S_\pi^* h_{0,\infty}\Vert_{\mathcal{H}_k^d} = \left\Vert S^*_{\pi,N} \left( h - \tfrac{1}{N} \sum_{i=1}^N h(X^i)\right) - S_\pi^* \left( h - \int_{\mathbb{R}^d} h \, \mathrm{d}\pi\right) \right\Vert_{\mathcal{H}_k^d} 
\\
& 
\label{eq:Sh terms}
\le \Vert (S^*_{\pi,N} - S_\pi^*) h \Vert_{\mathcal{H}_k^d} + \left\Vert S^*_{\pi,N} \left(\tfrac{1}{N} \sum_{i=1}^N h(X^i) - \int_{\mathbb{R}^d} h \, \mathrm{d}\pi \right)\right\Vert_{\mathcal{H}_k^d} + \left\Vert (S^*_{\pi,N} - S^*_\pi) \int_{\mathbb{R}^d} h \, \mathrm{d}\pi\right\Vert_{\mathcal{H}_k^d}
 \\
 & \le  \left\Vert \int_{\mathbb{R}^d}
\left(\nabla_{y} k(\cdot,y) + k(\cdot,y) \nabla \log \pi(y) \right) h(y) (\pi - \pi^{(N)})(\mathrm{d}y) \right\Vert_{\mathcal{H}_k^d}   
\\
& + \left\Vert \left(\int_{\mathbb{R}^d}
\left(\nabla_{y} k(\cdot,y) + k(\cdot,y) \nabla \log \pi(y) \right) (\pi - \pi^{(N)})(\mathrm{d}y)\right) \left( \int_{\mathbb{R}^d} h \, \mathrm{d}\pi - \frac{1}{N} \sum_{i=1}^N h(X^i) \right)\right\Vert_{\mathcal{H}_k^d}
\end{align}
\end{subequations}
For the first term in \eqref{eq:Sh terms}, notice that
\begin{equation*}
\Vert (S^*_{\pi,N} - S_\pi^*) h \Vert_{\mathcal{H}_k^d} =    \left\Vert \int_{\mathbb{R}^d}
\left(\nabla_{y} k(\cdot,y) + k(\cdot,y) \nabla \log \pi(y) \right) h(y) (\pi - \pi^{(N)})(\mathrm{d}y) \right \Vert_{\mathcal{H}_k^d},
\end{equation*}
with 
\begin{equation*}
\left\Vert\left(\nabla_{y} k(\cdot,y) + k(\cdot,y) \nabla \log \pi(y) \right) h(y)\right\Vert_{\mathcal{H}_k^d} \le \sqrt{\xi(y,y)} |h(y)| \le \widetilde{C} (1+|y|^p), 
\end{equation*}
as in \eqref{eq:kS bound}, with a possibly different constant $\widetilde{C}$. The second and the third term in \eqref{eq:Sh terms} can be bounded in the same way, and we omit the details. Putting everything together, convergence to zero of \eqref{eq:Sh norm} follows from Lemma \ref{lem:villani}.
\end{proof}
As a preparation for the proof of Theorem \ref{thm:consistency}, we need the following lemma, characterising the kernel of $S_\pi S_\pi^*$. 
\begin{lemma}
\label{lem:constants}
Let Assumptions \ref{ass:basic} and  \ref{ass:universality} be satisfied. Then $\ker S_\pi S_\pi^*$ consists of constants.
\end{lemma}
\begin{proof}
Let us first introduce the linear subspace $L_0^2(\pi) \subset L^2(\pi)$,
\begin{equation}
\label{eq:L20}
L_0^2(\pi) = \left\{ \phi \in L^2(\pi) : \qquad \int_{\mathbb{R}^d} \phi \, \mathrm{d}\pi = 0\right\},
\end{equation}
equipped with the restriction of the $L^2(\pi)$-inner product.
Since $S_\pi^* = - \mathcal{T}_{k,\pi}\nabla$ vanishes on constants, we may view $S_\pi S_\pi^*$ as a bounded linear operator on $L_0^2(\pi)$ and show that $\ker S_\pi S_\pi^* = \{0\}$. To that end, notice that $\ker S_\pi S_\pi^* = \ker S_\pi^* = (\mathrm{Ran} \,S_\pi)^\perp$, where the orthogonal complement is taken with respect to the $L^2(\pi)$ inner product, see \citet[Theorem 4.13, b)]{weidmann2012linear}. It is thus sufficient to show that if for some $\phi \in C_c^\infty(\mathbb{R}^d) \cap L_0^2(\pi)$, we have
\begin{equation}
\label{eq:phi orthogonal}
\langle \phi, S_\pi v \rangle_{L^2(\pi)} = 0, \qquad \text{for all} \, \, v \in \mathcal{H}_k^d,
\end{equation}
then necessarily $\phi = 0$. For this, first notice that
\begin{equation*}
\langle \phi, S_\pi v \rangle_{L^2(\pi)} = - \int_{\mathbb{R}^d} \nabla \phi \cdot v \, \mathrm{d}\pi.  
\end{equation*}
Assume now that $\phi \in C_c^\infty(\mathbb{R}^d) \cap L_0^2(\pi)$ is fixed such that \eqref{eq:phi orthogonal} holds. By density of $\mathcal{H}_k^d$ in $(L^2(\pi))^d$, we can choose a sequence $(v_n) \subset \mathcal{H}_k^d$ such that $v_n \rightarrow \nabla \phi$ in $(L^2(\pi))^d$. We then obtain
\begin{equation*}
0 = \langle \phi, S_\pi v_n \rangle_{L^2(\pi)} = - \int_{\mathbb{R}^d} \nabla \phi \cdot v_n \,\mathrm{d}\pi \rightarrow - \int_{\mathbb{R}^d} |\nabla \phi |^2 \, \mathrm{d}\pi 
\end{equation*}
by continuity, which clearly implies $\phi = 0$.
\end{proof}
We also need the following compactness result:
\begin{lemma}
\label{lem:compact}
Let Assumptions \ref{ass:basic} and  \ref{ass:universality} be satisfied. Then $S_\pi S_\pi^*: L^2(\pi) \rightarrow L^2(\pi)$ is compact.  
\end{lemma}
\begin{proof}
Using Lemma \ref{lem:operators}, first notice that  
\begin{subequations}
\label{eq:int rep SpiSpistar}
\begin{align}
S_\pi S_\pi^* \phi & = - S_\pi \mathcal{T}_{k,\pi}\nabla \phi = - S_\pi \int_{\mathbb{R}^d} k(\cdot,y) \nabla \phi(y) \pi(\mathrm{d}y)
\nonumber
\\
& = S_\pi \int_{\mathbb{R}^d} \left( \nabla_y k(\cdot,y) + k(\cdot,y) \nabla \log \pi(y)  \right) \phi(y) \pi(\mathrm{d}y) = \int_{\mathbb{R}^d} \xi^{k,\nabla \log \pi}(\cdot,y) \phi(y) \pi(\mathrm{d} y).
\nonumber
\tag{\ref*{eq:int rep SpiSpistar}}
\end{align}
\end{subequations}
From Lemma \ref{lem:estimates} we now have
\begin{equation*}
 \int_{\mathbb{R}^d} \xi^{k,\nabla \log \pi}(x,x) \pi(\mathrm{d}x) < \infty.   
\end{equation*}
The statement therefore follows from \citet[Theorem 4.27]{steinwart2008support}.
\end{proof}

\begin{proof}[Proof of Theorem \ref{thm:consistency}]
We begin with the estimate
\begin{subequations}
\begin{align}
\Vert S_\pi v_{N,\lambda} - h_{0,\infty} \Vert_{L^2(\pi)} & \le \Vert S_\pi(v_{N,\lambda} - v_{\infty,\lambda}) \Vert_{L^2(
\pi)} + \Vert S_\pi v_{\infty,\lambda} - h_{0,\infty} \Vert_{L^2(\pi)} 
\\
\label{eq:thm2 estimate}
& \le \Vert S_\pi \Vert_{\mathcal{H}_k^d \rightarrow L^2(\pi)} \Vert v_{N,\lambda} - v_{\infty,\lambda} \Vert_{\mathcal{H}_k^d} + \Vert S_\pi v_{\infty,\lambda} - h_{0,\infty} \Vert_{L^2(\pi)}.
\end{align}
\end{subequations}
Notice that the operator norm $\Vert S_\pi \Vert_{\mathcal{H}_k^d \rightarrow L^2(\pi)}$ is finite by Lemma \ref{lem:operators}, and that for fixed $\lambda >0$, the term $\Vert v_{N,\lambda} - v_{\infty,\lambda} \Vert_{\mathcal{H}_k^d}$ converges to zero as $N \rightarrow \infty$, as a consequence of Theorem \ref{thm:consistency}. 

Note that the second term in \eqref{eq:thm2 estimate} does not depend on $N$. We show that it converges to zero as $\lambda \rightarrow 0$. Recall the linear subspace $L_0^2(\pi) \subset L^2(\pi)$ defined in \eqref{eq:L20}. Clearly, $h_{0,\infty} \in L_0^2(\pi)$, and by Lemmas \ref{lem:constants} and \ref{lem:compact},  $S_\pi S_\pi^*$ acts as a self-adjoint, strictly positive definite, and compact operator on $L_0^2(\pi)$. Consequently, there exists an orthonormal basis $(e_i)_{i=1}^\infty$ in $L^2_0(\pi)$ such that $S_\pi S_\pi^* e_i = \mu_i e_i$, with strictly positive eigenvalues $\mu_i$. Writing $h_{0,\infty} = \sum_{i=1}^\infty h_i e_i$, we see from \eqref{eq:Tikhonov mean field 1} that
\begin{equation}
\label{eq:Sv - h}
\Vert S_\pi v_{\infty,\lambda} - h_{0,\infty} \Vert_{L^2(\pi)}^2 = \sum_{i=1}^\infty \left( \frac{\mu_i}{\lambda+\mu_i} - 1 \right)^2 h_i^2.
\end{equation}
Since the term in parenthesis is bounded by one, the dominated convergence theorem allows us to take the limit as $\lambda \rightarrow 0$ termwise. As these limits are zero, we conclude that $\Vert S_\pi v_{\infty,\varepsilon} - h_{0,\infty} \Vert_{L^2(\pi)}^2 \rightarrow 0$. Note that the fact that $\mu_i >0$ is crucial, since otherwise the corresponding term would not converge to zero. 

We now conclude as follows: For given $\varepsilon >0$, we can choose $\lambda > 0$ such that $\Vert S_\pi v_{\infty,\lambda} - h_{0,\infty} \Vert_{L^2(\pi)} < \tfrac{\varepsilon}{2}$. For this fixed value of $\lambda$, we have from Theorem \ref{thm:connections} that $\Vert S_\pi \Vert_{\mathcal{H}_k^d \rightarrow L^2(\pi)} \Vert v_{N,\lambda} - v_{\infty,\lambda} \Vert_{\mathcal{H}_k^d}$ converges to zero as $N \rightarrow \infty$. We can thus take $N_0 \in \mathcal{N}$ such that this term is bounded from above by $\tfrac{\varepsilon}{2}$, for all $N \ge N_0$. 
\end{proof}

\begin{proof}[Proof of Proposition \ref{prop:reg -> 0}]
The Stein-Poisson equation \eqref{eq:Stein Poisson} implies 
\begin{equation*}
\lambda  \phi^{(\lambda)} =  h - \int_{\mathbb{R}^d} h \, \mathrm{d}\pi + \frac{1}{\pi}\nabla \cdot (\pi \mathcal{T}_{k,\pi} \nabla \phi^{(\lambda)}) = h_{0,\infty} - S_{\pi}v_{\infty,\lambda},
\end{equation*}
using the fact that $v_{\infty,\lambda} = - \mathcal{T}_{k,\pi} \nabla \phi^{(\lambda)}$. The first claim now follows from \eqref{eq:Sv - h}, together with reasoning following it in the proof of Theorem \ref{thm:consistency}. For the second claim, we follow an argument from \citet[Appendix A]{de2005model}: Using $\mu_i/(\lambda + \mu_i) - 1 = -1/(1+ \mu_i/\lambda)$, we estimate \eqref{eq:Sv - h},
\begin{equation}
\label{eq:lambda^alpha calculation}
\Vert S_\pi v_{\infty,\lambda} - h_{0,\infty} \Vert_{L^2(\pi)}^2 = \sum_{i=1}^\infty \left( \frac{1}{1+ \frac{\mu_i}{\lambda}} \right)^2 h_i^2
 \le  \sum_{i=1}^\infty \left( \frac{1}{ \left(\frac{\mu_i}{\lambda} \right)^\alpha} \right)^2 h_i^2 = \lambda^{2 \alpha} \sum_{i=1}^\infty \mu_i^{-2 \alpha} h_i^2,
 \end{equation}
 making use of the inequality $x^\alpha \le x + 1$, which holds for all $\alpha \in (0,1]$ and $x \in [0,\infty)$. From the integral representation of $S_\pi S_\pi^*$ in \eqref{eq:int rep SpiSpistar}, we see that the RKHS associated to $\xi^{k,\nabla \log \pi}$ is isomorphic to the range $(S_\pi S_\pi^*)^{1/2} L^2(\pi)$, see also \citet[Theorem 4.51]{steinwart2008support}. By the definition of the fractional powers $\mathcal{H}^\alpha_{k,\nabla \log \pi}$ \citep[Definition 4.11]{muandet2017kernel}, the right-most expression is finite if $h_{0,\infty} \in  \mathcal{H}^\alpha_{k,\nabla \log \pi}$, and so the claim follows.
\end{proof}

\section{Proof of Proposition \ref{thm:infitesimal Stein transport}}
\label{app:transport}

\begin{proof}[Proof of Proposition \ref{thm:infitesimal Stein transport}]

We first summarise the main steps of the proof. For fixed $t \in (-\varepsilon,\varepsilon)$, let us denote by $(v^t_{\tau})_{\tau \in [0,1]}$ the family of vector fields that realise the Stein optimal transport between $\pi_0$ and $\pi_t$, in the sense of \eqref{eq:Stein optimality}, noticing that $t$ is replaced by $\tau$ in that equation. The Stein optimal transport maps then take the form 
\begin{equation*}
F_t(x) = x + \int_0^1 v_\tau^t(X_\tau^{t,x}) \, \mathrm{d}\tau,  
\end{equation*}
where $X_\tau^t$ solves the ODE
\begin{equation*}
\frac{\mathrm{d}X_\tau^{t,x}}{\mathrm{d}\tau} = v_\tau^t(X_\tau^{t,x}), \qquad X_\tau^{t,x} = x.    
\end{equation*}
We then proceed as follows: Firstly, we show that there exists a sequence $t_n \rightarrow 0$ with $t_n > 0$ such that 
\begin{equation}
\label{eq:tn limit}
\frac{F_{t_n}(x) - x}{t_n} = \frac{1}{t_n} \int_0^{t_n} v_\tau^{t_n} (X^{t_n,x}_\tau) \, \mathrm{d}\tau \xrightarrow{n \rightarrow \infty} v^*(x),
\end{equation}
for some $v^* \in \mathcal{H}_k^d$ and all $x \in \mathbb{R}^d$. Secondly, we show that 
\begin{equation}
\label{eq:continuity transport proof}
\partial_t \pi_t\vert_{t = 0} + \nabla \cdot(\pi_0 v^*) = 0,
\end{equation}
in the weak sense: $v^*$ is compatible with the dynamics of the curve $(\pi_t)_{t \in (-\varepsilon,\varepsilon)}$. Next, we show that $v^*$ is minimal among vector fields that satisfy \eqref{eq:continuity transport proof}, in the sense of the $\Vert \cdot \Vert_{\mathcal{H}_k^d}$-norm. Lastly, we show that these characterisations of $v^* $ allow us to conclude convergence along any sequence $t_n$ in \eqref{eq:tn limit}, to the same limit, and that this limit is characterised by the Stein-Poisson equation \eqref{eq:SP transport}. 

\emph{Step 1.} We first argue that 
\begin{equation}
\label{eq:arclength}
\Vert v^t_\tau \Vert_{\mathcal{H}_k^d} = d_k(\pi_0,\pi_t),
\end{equation}
for all $t \in (-\varepsilon,\varepsilon)$ and $\tau \in [0,1]$. In other words, minimisers in \eqref{eq:dk} are automatically parameterised by arc-length, as $\Vert v^t_\tau \Vert_{\mathcal{H}_k^d}$ is constant in $\tau \in [0,1]$.  This is a standard fact for geodesic distances on Riemannian manifolds of the form \eqref{eq:dk}, see, for instance, \citet[Section 1.4]{jost2008riemannian}. For convenience, let us repeat the argument: 
\begin{enumerate}
\item
Firstly, Jensen's inequality implies the energy-length comparison
\begin{equation}
\label{eq:Jensen}
\int_0^1 \Vert v^t_\tau \Vert^2_{\mathcal{H}_k^d} \, \mathrm{d} \tau \ge \left(\int_0^1 \Vert v^t_\tau \Vert_{\mathcal{H}_k^d} \, \mathrm{d}\tau \right)^2,  
\end{equation}
with equality if and only if $\Vert v^t_\tau \Vert_{\mathcal{H}_k^d}$ is constant in $\tau$. 
\item
Secondly, the right-hand side of \eqref{eq:Jensen} is reparameterisation-invariant in the following sense: If $\gamma:[0,1] \rightarrow [0,1]$ is a reparameterisation (meaning that $\gamma$ is absolutely continuous on $(0,1)$, with $\gamma(0) = 0$, $\gamma(1) = 1$ and $\gamma' > 0$ Lebesgue almost everywhere on $(0,1)$), then the reparameterised vector field $v^t_\tau$ (defined by $\widetilde{v}^t_{\tau}:= \gamma'(\tau) v^t_{\gamma(\tau)}$ on $\tau \in (0)$ and $\widetilde{v}^t_{0} = v^t_0$ as well as $\widetilde{v}^t_{1} = v^t_1$) satisfies
\begin{equation}
\label{eq:repara}
\int_0^1 \Vert v^t_\tau \Vert_{\mathcal{H}_k^d} \, \mathrm{d}\tau = \int_0^1 \Vert \widetilde{v}^t_\tau \Vert_{\mathcal{H}_k^d} \, \mathrm{d}\tau.
\end{equation}
\end{enumerate}
Combining \eqref{eq:Jensen} and \eqref{eq:repara}, it follows that any given $(v^t_{\tau})_{\tau \in [0,1]}$ can be reparameterised in such a way that \eqref{eq:arclength} holds, without affecting the right-hand side of \eqref{eq:Jensen}, but minimising the right-hand side.

From \eqref{eq:arclength}, the fact that $v^0_\tau = 0$, and Lipschitz continuity of $t \mapsto d_k(\pi_0,\pi_t)$, it follows that the set 
\begin{equation}
\label{eq:set}
\left\{\frac{1}{t} \int_0^1 v_\tau^t \, \mathrm{d}\tau: \quad t \in (-\varepsilon,\varepsilon) \setminus \{0\} \right\} 
\end{equation}
is bounded in $\mathcal{H}_k^d$. As a consequence, the Banach-Alaoglu theorem allows us to extract a sequence $\frac{1}{t_n} \int_0^1 v_\tau^{t_n} \, \mathrm{d}\tau$ with $t_n \rightarrow 0$ that converges weakly to some $v^* \in \mathcal{H}_k^d$; that is 
\begin{equation*}
\left\langle \frac{1}{t_n} \int_0^1 v_\tau^{t_n} \,\mathrm{d}\tau, h \right\rangle_{\mathcal{H}_k^d}  \xrightarrow{n \rightarrow \infty} \langle v^*,h \rangle_{\mathcal{H}_k^d},    
\end{equation*}
for all $h \in \mathcal{H}_k^d$. By choosing $h$ appropriately and using the reproducing property (and the fact that bounded linear operators commute with Bochner integration), we see that $\tfrac{1}{t_n} \int_0^1 v^{t_n}_\tau \, \mathrm{d}\tau$ converges to $v^*$ pointwise (for example, choosing $h = (k(x,\cdot),0\ldots,0)$, with arbitrary $x \in \mathbb{R}^d$,  shows that the first components converge).

We now show the convergence in \eqref{eq:tn limit}, for fixed $x \in \mathbb{R}^d$. To that end, notice that
\begin{equation}
\label{eq:F estimate}
\left| \frac{F_{t_n}(x) - x}{t_n} - v^*(x)\right| \le \left| \tfrac{1}{t_n} \int_0^1 \left( v_\tau^{t_n}(X_\tau^{t_n,x}) - v_\tau^{t_n}(x) \right)\mathrm{d}\tau \right| + \left| \tfrac{1}{t_n} \int_0^1 v_\tau^{t_n}(x) \, \mathrm{d}\tau - v^*(x) \right|.    
\end{equation}
According to the previous arguments, 
the second term on the right-hand side converges to zero. The first term can be bounded from above by a constant times
\begin{equation}
\label{eq:1st term}
\tfrac{1}{t_n}\int_0^1\Vert v_\tau^{t_n} \Vert_{\mathcal{H}_k^d} |X_\tau^{t_n,x} - x | \, \mathrm{d}\tau,
\end{equation}
owing to the fact that $v_\tau^{t_n}$ is Lipschitz continuous (since $k$ has bounded derivatives by assumption), and the Lipschitz constant is controlled by the RKHS-norm \cite[Corollary 4.36]{steinwart2008support}. Since $\tfrac{1}{t_n}\Vert v_\tau^{t_n} \Vert_{\mathcal{H}_k^d}$ is bounded (again, because of \eqref{eq:arclength} and the Lipschitz property of $t \mapsto d_k(\pi_0,\pi_t)$) and 
\begin{equation*}
|X_\tau^{t_n,x} - x | \le \int_0^\tau |v_s^{t_n}(X_s^{t,x})| \, \mathrm{d}s \lesssim \int_0^\tau \Vert v_s^{t_n}(X_s^{t,x})\Vert_{\mathcal{H}_k^d} \, \mathrm{d}s = \tau d_k(\pi_0,\pi_{t_n}) \xrightarrow{n \rightarrow \infty} 0,
\end{equation*}
we conclude that \eqref{eq:1st term} converges to zero, and therefore \eqref{eq:tn limit} holds.

\emph{Step 2.} Let us show that any $v^*$ identified in Step 1 satisfies the continuity equation \eqref{eq:continuity transport proof}. For $\psi \in C_c^\infty(\mathbb{R}^d)$, we have 
\begin{equation*}
\frac{1}{t_n} \left( \int_{\mathbb{R}^d} \psi \, \mathrm{d}\pi_{t_n} - \int_{\mathbb{R}^d} \psi \, \mathrm{d}\pi_0\right) = \frac{1}{t_n} \left( \int_{\mathbb{R}^d} \psi \, \mathrm{d}((F_{t_n})_{\#} \pi_0) - \int_{\mathbb{R}^d} \psi \, \mathrm{d}\pi_0\right) = \frac{1}{t_n} \int_{\mathbb{R}^d} \left(\psi(F_{t_n}(x)) - \psi(x)\right) \mathrm{d}\pi_0.
\end{equation*}
The left-hand side of this equality converges to $\int_{\mathbb{R}^d} \psi \partial_t \pi |_{t = 0} \, \mathrm{d}x$, while the left-hand side converges to 
$
\int_{\mathbb{R}^d} \nabla \psi \cdot v^* \, \mathrm{d}\pi_0
$, using \eqref{eq:tn limit}.

\emph{Step 3.} In this step, we establish that the weak limit $v^*$ is minimal among solutions to \eqref{eq:continuity transport proof}, in the sense of $\Vert \cdot \Vert_{\mathcal{H}^d_k}$. To achieve this, we compare optimal Stein transport interpolations encoded by the vector fields $v^t_\tau$ to the (suboptimal) flow of $(\pi_t)_{t \in (-\varepsilon,\varepsilon)}$. For $t \in (-\varepsilon,\varepsilon)$ and $\tau \in [0,1]$, we set 
\begin{equation*}
w^t_\tau := t \mathcal{T}_{k,\pi_{\tau t}} \nabla \phi_{\tau t}, \qquad \text{and} \qquad \rho_\tau := \pi_{\tau t}.     
\end{equation*}
Notice that $\rho_0 = \pi_0$ and $\rho_1 = \pi_t$, as well as
$\partial_\tau \rho_\tau + \nabla \cdot (\rho_{\tau} w^t_\tau) = 0$, and therefore $(w^t_\tau)_{\tau \in [0,1]}$ is a (suboptimal) competitor for $(v_\tau^t)_{\tau \in [0,1]}$ in the transport formulation \eqref{eq:dk} for $d_k(\pi_0,\pi_t)$. We therefore have 
\begin{subequations}
\begin{align*}
\left \Vert \tfrac{1}{t_n} \int_0^1 v_\tau^{t_n} \, \mathrm{d}\tau\right \Vert^2_{\mathcal{H}_k^d} \le \tfrac{1}{t_n^2} \int_0^1 \Vert v_\tau^{t_n} \Vert^2_{\mathcal{H}_k^d} \, \mathrm{d}\tau \le \tfrac{1}{t_n^2} \int_0^1 \Vert w_\tau^{t_n} \Vert^2_{\mathcal{H}_k^d} \, \mathrm{d}\tau & = \int_0^1 \Vert \mathcal{T}_{k,\pi_{\tau t_n}} \nabla \phi_{\tau t_n} \Vert^2_{\mathcal{H}_k^d} \, \mathrm{d}\tau 
\\
& \xrightarrow{n \rightarrow \infty} \Vert \mathcal{T}_{k,\pi_0} \nabla \phi_{0} \Vert^2_{\mathcal{H}_k^d},
\end{align*}
\end{subequations}
implying that $\Vert v^* \Vert_{\mathcal{H}_k^d} \le \Vert \mathcal{T}_{k,\pi_0} \nabla \phi_{0} \Vert_{\mathcal{H}_k^d}$ by the weak lower semicontinuity of $\Vert \cdot \Vert_{\mathcal{H}_k^d}$.

\emph{Step 4.} The fact that $v^*$ satisfies \eqref{eq:continuity transport proof} and $\Vert v^* \Vert_{\mathcal{H}_k^d} \le \Vert \mathcal{T}_{k,\pi_0} \nabla \phi_{0} \Vert_{\mathcal{H}_k^d}$ implies $v^* = \mathcal{T}_{k,\pi_0}\nabla \phi_0$, because solutions to the continuity equation with minimal RKHS-norm take precisely this form \cite[Proposition 5]{duncan2019geometry}. In this step, we make this argument precise. The Helmholtz decomposition in $\mathcal{H}_k^d$ \cite[Proposition 6]{duncan2019geometry} is
\begin{equation}
\label{eq:decomp}
\mathcal{H}_k^d = (L^2_{\mathrm{div}}(\pi_0) \cap \mathcal{H}_k^d) \oplus \overline{\mathcal{T}_{k,\pi_0} \nabla C_c^\infty(\mathbb{R}^d)}^{\mathcal{H}_k^d},
\end{equation}
where the two subspaces are orthogonal in $\mathcal{H}_k^d$, and $L^2_{\mathrm{div}}(\pi_0)$ is the space weighted divergence-free vector fields,
\begin{equation*}
L^2_{\mathrm{div}}(\pi_0) = \left\{ v \in (L^2(\pi_0))^d: \quad \langle v, \nabla \phi \rangle_{(L^2(\pi_0))^d}, \quad \text{for all}\,\, \phi \in C_c^\infty(\mathbb{R}^d\right\}.
\end{equation*}
Because both $v^*$ and $w := \mathcal{T}_{k,\pi_0}\nabla \phi_0$ satisfy the continuity equation \eqref{eq:continuity transport proof}, we have that $w - v^* \in L^2_{\mathrm{div}}(\pi_0)$. 
By the second item in Lemma \ref{lem:operators},
we have $w \in \overline{\mathcal{T}_{k,\pi_0} \nabla C_c^\infty(\mathbb{R}^d)}^{\mathcal{H}_k^d}$, and by the orthogonal decomposition \eqref{eq:decomp}, we see that $\langle w - v^*, w \rangle_{\mathcal{H}_k^d} = 0$. Rearranging, we obtain
\begin{equation}
\label{eq:inequality chain}
\Vert w \Vert^2_{\mathcal{H}_k} = \langle v^*, w \rangle_{\mathcal{H}_k^d} \le \Vert v^* \Vert_{\mathcal{H}_k^d} \Vert w \Vert_{\mathcal{H}_k^d} \le \Vert w \Vert^2_{\mathcal{H}_k^d},
\end{equation}
where we have used Cauchy-Schwarz in the first inequality and the result from Step 3 in the second inequality. The chain of inequalities in \eqref{eq:inequality chain} forces the Cauchy-Schwarz inequality to become an equality, which implies that $w$ and $v^*$ are linearly dependent. Together with $\langle w - v^*, w \rangle_{\mathcal{H}_k^d} = 0$, it follows that $v^* = w$. 

\emph{Step 5}. By the previous step, any convergent sequence in the set \eqref{eq:set} with $t_n \rightarrow 0$ has the same (weak) limit. Therefore, the limit $\lim_{t \rightarrow 0} \tfrac{1}{t}\int_0^1 v_{\tau}^{t_n}\, \mathrm{d}\tau$ exists and equals $v^*$, and the proof is completed by repeating the calculation in \eqref{eq:F estimate}.
\end{proof}

\section{Fisher-Rao and natural gradient flows}
\label{app:natural}

The purpose of this appendix is to briefly explain the interpretation \eqref{eq:Newton flow} of \eqref{eq:FR flow}. The discussion here is purely heuristic; we refer the reader to \cite{ambrosio2008gradient,mielke2023introduction} for rigorous accounts of gradient flow theory. 

As alluded to in Section \ref{sec:gradient flows}, both the gradient and the Hessian operator in \eqref{eq:Newton flow} are interpreted in the `Euclidean' way. This means that the tangent spaces 
\begin{equation*}
T_{\rho} \mathcal{P}(\mathbb{R}^d) = \left\{ \sigma: \quad \int_{\mathbb{R}^d} \sigma \, \mathrm{d}x = 0 \right\},
\end{equation*}
containing infinitesimal changes to $\rho \in \mathcal{P}(\mathbb{R}^d)$ of the form $\rho \mapsto \rho + \sigma$,
are equipped with the inner product
\begin{equation}
\label{eq:euclidean metric}
 \langle \sigma_1,  \sigma_2 \rangle  := \int_{\mathbb{R}^d} \sigma_1 \sigma_2  \, \mathrm{d}x,    
\end{equation}
which is similar to the standard Euclidean inner product $\langle x, y \rangle_{\mathbb{R}^N} = \sum^N_{i=1} x_i y_i$. The $L^2$-derivative of $\rho \mapsto \mathrm{KL}(\rho|\pi)$ at a fixed probability measure $\rho^* \in \mathcal{P}(\mathbb{R}^d)$ can be determined from the relation
\begin{equation}
\label{eq:def gradient}
\partial_t \mathrm{KL}(\rho_t | \pi) \Big \vert_{t = 0} = \int_{\mathbb{R}^d} \nabla \mathrm{KL}(\rho^* | 
\pi)  \partial_t \rho_t |_{t = 0} \, \mathrm{d}x,
\end{equation}
which is supposed to hold for all (differentiable) curves $(\rho_t)_{t \in (-\varepsilon,\varepsilon)} \subset \mathcal{P}(\mathbb{R}^d)$ with $\rho_0 = \rho^*$. 
Here, we interpret $\nabla \mathrm{KL}(\rho^*|\pi)$ as an element of the dual (cotangent) space $T^*_{\rho*} \mathcal{P}(\mathbb{R}^d)$, acting on the tangent vector $\partial_t \rho_t |_{t = 0} \in T_{\rho*} \mathcal{P}(\mathbb{R}^d)$ via the $L^2$-pairing in \eqref{eq:def gradient}.
Direct calculation shows that 
\begin{equation}
\label{eq:gradient computation}
\partial_t \mathrm{KL}(\rho_t | \pi) \Big \vert_{t = 0} = \partial_t \int_{\mathbb{R}^d} \log \left( \frac{\mathrm{d} \rho_t}{\mathrm{d} \pi}  \right) \mathrm{d} \rho_t \Big \vert_{t = 0} = \int_{\mathbb{R}^d}\log  \left( \frac{\mathrm{d} \rho^*}{\mathrm{d} \pi}  \right) \partial_t \rho_t |_{t = 0} \, \mathrm{d}x, 
 \end{equation}
so that \eqref{eq:def gradient} implies 
\begin{equation}
\label{eq:cotangent derivative}
\nabla \mathrm{KL}(\rho^* | \pi) = \log \left( \frac{\mathrm{d} \rho^*}{\mathrm{d}\pi}\right).
\end{equation}
\begin{remark}
The derivative in \eqref{eq:cotangent derivative} plays the same role as the exterior derivative \cite[Chapter 11]{lee2012smooth} in differential geometry. Notice that \eqref{eq:cotangent derivative} is only defined up to an additive constant, since $\nabla \mathrm{KL}(\rho^* | \pi)$ acts via integrals against mean-zero functions in $T_{\rho^*}\mathcal{P}(\mathbb{R}^d)$. We can also compute the \emph{Riemannian gradient} with respect to \eqref{eq:euclidean metric}, essentially interpreting the left-hand side of \eqref{eq:def gradient} as the Riemannian metric: 
\begin{equation}
\label{eq:Riemannian gradient}
\mathrm{grad}\, \mathrm{KL}(\rho^* | \pi) = \log \left( \frac{\mathrm{d} \rho^*}{\mathrm{d}\pi}\right) - \int_{\mathbb{R}^d }\log \left( \frac{\mathrm{d} \rho^*}{\mathrm{d}\pi}\right) \mathrm{d}x.
\end{equation}
Importantly, \eqref{eq:Riemannian gradient} contains an additional constant, ensuring that $\mathrm{grad}\, \mathrm{KL}(\rho^* | \pi) \in T_{\rho^*}\mathcal{P}(\mathbb{R}^d)$.
\end{remark}
In a similar manner, the Hessian (as a quadratic form on $T_{\rho^*} \mathcal{P}(\mathbb{R}^d)$) can be determined from 
\begin{equation}
\label{eq:def Hessian}
\partial_t^2 \mathrm{KL}(\rho_t | \pi) |_{t = 0} = \mathrm{Hess} \,  \mathrm{KL} (\rho^* |\pi)( \partial_t \rho_t |_{t = 0},\partial_t \rho_t |_{t = 0} )
\end{equation}
for all \emph{geodesics} $(\rho_t)_{t \in (-\varepsilon,\varepsilon)} \subset \mathcal{P}(\mathbb{R}^d)$ that satisfy $\rho_0 = \rho^*$. 
Analogously to the Euclidean case, where geodesics are given by linear interpolations of the form $x_t =(1 - t) x_0 + t x_1$, geodesics for \eqref{eq:euclidean metric} are given by $\rho_t = (1-t) \rho_0 + t \rho_1$; in particular, $\partial_t^2 \rho_t = 0$.   
As in \eqref{eq:gradient computation}, we compute 
\begin{subequations}
\begin{align*}
\partial^2_t \mathrm{KL}(\rho_t|\pi)|_{t = 0}  = \int_{\mathbb{R}^d}  \frac{(\partial_t \rho)^2\vert_{t  = 0}}{\rho^*}  \, \mathrm{d}x + \int_{\mathbb{R}^d} \log \left( \frac{\rho_t}{\pi}\right) \underbrace{\partial_t^2 \rho_t}_{=0} \mathrm{d}x\Big|_{t=0}.
\end{align*}
\end{subequations}
Comparing to \eqref{eq:def Hessian}, we see that 
\begin{equation}
\label{eq:Hess KL}
\mathrm{Hess} \,  \mathrm{KL} (\rho^* |\pi) (\sigma,\sigma) = \int_{\mathbb{R}^d} \frac{\sigma^2}{\rho^*} \, \mathrm{d}x.
\end{equation}
\begin{remark}[Fisher-Rao geometry]
The Hessian quadratic form \eqref{eq:Hess KL} coincides with the Fisher-Rao metric tensor \citep{amari2016information,ay2017information}. 
\end{remark}
The quadratic form \eqref{eq:Hess KL} can also be interpreted as a mapping from the tangent space $T_{\rho*} \mathcal{P}(\mathbb{R}^d)$ into its dual (cotangent) space $T^*_{\rho*} \mathcal{P}(\mathbb{R}^d)$,
\begin{equation}
\label{eq:Hessian map}
T_{\rho*} \mathcal{P}(\mathbb{R}^d) \ni \sigma \mapsto \frac{\sigma}{\rho^*} \in T^*_{\rho*} \mathcal{P}(\mathbb{R}^d),
\end{equation}
where we recall that the action of $T^*_{\rho*} \mathcal{P}(\mathbb{R}^d)$ on $T_{\rho*} \mathcal{P}(\mathbb{R}^d) $ is understood via the $L^2$-pairing in \eqref{eq:euclidean metric}.\footnote{Notice that, in contrast to the situation in conventional Riemannian geometry, it is not straightforwardly possible to give meaning to the Hessian (or the Fisher-Rao metric) as an operator mapping $T_{\rho*} \mathcal{P}(\mathbb{R}^d)$ to itself, because $\frac{\sigma}{\rho^*}$ does not necessarily integrate to zero.} To interpret \eqref{eq:Newton flow}, we need to invert \eqref{eq:Hessian map} and apply the result to \eqref{eq:cotangent derivative}. Indeed, it is immediate to verify that the required inverse\footnote{We highlight that \eqref{eq:inverse} inverts \eqref{eq:Hessian map} up to additive constants; functions that differ by such constants are identified in $T_{\rho^*}\mathcal{P}(\mathbb{R}^d)$.} is given by
\begin{equation}
\label{eq:inverse}
T^*_{\rho*} \mathcal{P}(\mathbb{R}^d) \ni f \mapsto \rho^* \left( f - \int_{\mathbb{R}^d} f \, \mathrm{d}\rho^*\right) \in T_{\rho*} \mathcal{P}(\mathbb{R}^d),
\end{equation}
so that replacing $f$ by $\log \left( \frac{\mathrm{d}\rho^*}{\mathrm{d}\pi}\right)$ connects \eqref{eq:FR flow} and \eqref{eq:Newton flow}.
\end{document}